\documentclass[twoside,11pt]{article} %
\usepackage[submit]{optional}

\RequirePackage{amsthm,amsmath,amsfonts,amssymb}  

\usepackage{xspace}

\usepackage{enumitem}
\usepackage{caption}
\usepackage{bm}
\usepackage{tikz-cd}
\usepackage{xcolor}
\usepackage{nicefrac}
\usepackage[most]{tcolorbox}
\usepackage{xparse}

\opt{submit}{

 \usepackage{jmlr2e}
}

\usepackage{url} 

 \usepackage{hyperref}

 \opt{submit}{

\newtheorem{proposition}{Proposition}

\newtheorem{remark}{Remark}

\newtheorem{lemma}{Lemma}

\newtheorem{corollary}{Corollary}

\newtheorem{definition}{Definition}

\newtheorem{conjecture}{Conjecture}

\newtheorem{axiom}{Axiom}

}

\definecolor{mydarkblue}{rgb}{0,0.08,0.45}
\hypersetup{ hidelinks }

\usepackage{mathrsfs}
\usepackage[capitalize]{cleveref}  

\usepackage{autonum} %

\usepackage[scr = boondoxo, scrscaled = 1]{mathalfa}

 \usepackage{thmtools}

\renewcommand{\iff}{\Leftrightarrow}

\newcommand{\eps}{\epsilon}

\newcommand{\sk}{\!\!\sss{\sqrt{\k}}}
\newcommand{\sks}{\!\!\sss{\sqrt{\ks}}}
\newcommand{\KSD}{\mathrm{KSD}}

\def\P{\mathrm{P}}
\def\Q{\mathrm{Q}}
\def\p{\mathrm{p}}
\def\q{\mathrm{q}}
\newcommand{\gvec}[0]{{\bm{g}}} %
\renewcommand{\sp}{\bm{s}_{\p}}
\newcommand{\spi}{s_{\P}^i}
\newcommand{\sq}{\bm{s}_{\q}}
\def\Pset{\mathcal{P}}
\newcommand{\embedpettis}[1]{\Pset_{\!\sss{\H_{#1}}}}
\newcommand{\embedtozero}[1][\Kb]{\Pset_{\!\sss{#1,0}}}
\newcommand{\kti}{\kappa}
\def\R{\mathbb{R}}
\def\N{\mathbb{N}}

\newcommand{\grad}{\nabla}
\newcommand{\inner}[2]{\langle{#1},{#2}\rangle} %
\def\Matern{Mat\'ern\xspace}

\newcommand{\langevin}[1][\p]{\mathcal{S}_{#1}} %
\def\balign#1\ealign{\begin{align}#1\end{align}}
\def\baligns#1\ealigns{\begin{align}#1\end{align}}
\def\balignat#1\ealign{\begin{alignat}#1\end{alignat}}
\def\balignats#1\ealigns{\begin{alignat*}#1\end{alignat*}}
\def\bitemize#1\eitemize{\begin{itemize}#1\end{itemize}}
\def\benumerate#1\eenumerate{\begin{enumerate}#1\end{enumerate}}

\def\Holder{H\"older\xspace}

\def\Matern{Mat\'ern\xspace}

\def\mbb#1{\mathbb{#1}}

\def\tbf#1{\textbf{#1}}

\def\X{\mathcal{X}}

\def\Id{\mathrm{Id}}
\newcommand{\defn}{\equiv}               %
\newcommand{\growth}{\theta}
\newcommand{\tilt}{a}

\renewcommand{\d}{\partial}
\newcommand{\dd}{\mathrm{d}}
\newcommand{\dx}{\d_{x^i}}
\newcommand{\dy}{\d_{y^i}}

\newcommand{\diff}{\mathop{} \! \mathrm{d}}

\newcommand{\ks}{\k_{\p}}%
\newcommand{\kb}{\k} %
\newcommand{\Kb}{K} %
\renewcommand{\k}{k}

\newcommand{\ksd}[2][\Kb]{\mathrm{KSD}_{#1, \P}(#2)}
\newcommand{\ksdsq}[2][\Kb]{\mathrm{KSD}_{#1, \P}^2(#2)}

\newcommand{\projP}{\Pi_{\P}}

\renewcommand{\S}{\mathcal{S}}
\newcommand{\sss}[1]{\scriptscriptstyle{#1}}

\def\indic#1{\mbb{I}\left[{#1}\right]} %

\newcommand{\metric}[2]{\inner{#1}{#2}}

\newcommand{\half}{\frac{1}{2}}
\newcommand{\thalf}{{\textstyle\frac{1}{2}}}

\newcommand{\qtext}[1]{\quad\text{#1}\quad}

\newcommand{\stext}[1]{\ \text{#1}\ }

\newcommand{\emb}{\Phi}
\newcommand{\embK}{\emb_{\k}}
\newcommand{\norm}[1]{\ensuremath{\left \| #1 \right \|}}
\newcommand{\twonorm}[1]{\norm{#1}}
\newcommand{\onenorm}[1]{\norm{#1}_1}
\newcommand{\infnorm}[1]{\norm{#1}_\infty}
\newcommand{\normK}[1]{\ensuremath{\norm{#1}_{\k}}}
\newcommand{\ipd}[2]{\ensuremath{\left \langle #1 \, , \, #2 \right \rangle}}
\newcommand{\ipdK}[2]{\ensuremath{\ipd{#1}{#2}_{\k}}}
\newcommand{\mmd}{\mathrm{MMD}}

\def\F{\mathcal{F}}
\def\H{\mathcal{H}}
\def\L{\mathrm{L}}
\def\B{\mathcal{B}}
\newcommand{\cts}{\mathscr{C}}
\newcommand{\HK}{\H_{\k}}

\newcommand{\Hks}{\H_{\ks}}

\newcommand{\C}[2]{\mathscr{C}^{\sss{#1}}_{\sss{#2}}}
\newcommand{\BB}[2]{\mathscr{B}^{\sss{#1}}_{\sss{#2}}}

\def\D{\mathrm{D}}

\newcommand{\M}{\mathcal{M}}

\newcommand{\Dset}{\mathscr{D}}

\newcommand{\DL}[1]{\Dset^{#1}_{\sss{L^1}}}
\newcommand{\supp}{\mathrm{supp}}

\providecommand{\function}[5]{} %
\renewcommand{\function}[5]{
	\ensuremath{
	\mathchoice{
	\ifthenelse{\equal{#1}{}}
            	{
            	\begin{array}[t]{ccl}
            		\ifthenelse{{\equal{#2}{}}}{
            		#4 & \longmapsto & #5}{
            		#2 & \longrightarrow & #3
            		 \ifthenelse{\equal{#4}{}} {} {\\
            		#4 & \longmapsto & #5}}
            	\end{array} \!
            	}
            	{
            	\begin{array}[t]{lccl}
            	#1 : 
            		\ifthenelse{{\equal{#2}{}}}{
            		& #4 & \longmapsto & #5}{
            		& #2 & \longrightarrow & #3
            		 \ifthenelse{\equal{#4}{}} {} {\\
            		 & #4 & \longmapsto & #5}}
            		 \end{array} \!
            	}
	}
	{
		\ifthenelse{\equal{#1}{}}
		{
			\ifthenelse{{\equal{#2}{}}}{
			{#4 \mapsto \, #5}}{
			{#2 \rightarrow \, #3}
			 \ifthenelse{\equal{#4}{}} {} {
			, \; {#4 \mapsto \, #5}}}
		}
		{
		#1 : 
			\ifthenelse{{\equal{#2}{}}}{
			{#4 \mapsto  \, #5}}{
			{#2 \rightarrow \, #3}
			 \ifthenelse{\equal{#4}{}} {} {
			, \; {#4 \mapsto \, #5}}}
		}	
	}
{}{}}}

\newenvironment{talign*}
 {\csname align*\endcsname}
 {\endalign}
\newenvironment{talign}
 {\csname align\endcsname}
 {\endalign}

\colorlet{colexam}{red!55!black} %
\newtcolorbox[auto counter,crefname={Application}{Applications}]{application}[2][]{%
    empty,%
    title={Application~\thetcbcounter: #2},%
    label={#1},
    attach boxed title to top left,
       minipage boxed title,
    boxed title style={empty,size=minimal,toprule=0pt,top=4pt,left=3mm,overlay={}},
    coltitle=colexam,fonttitle=\bfseries,
    before=\par\medskip\noindent,parbox=false,boxsep=0pt,left=3mm,right=0mm,top=2pt,breakable,pad at break=0mm,
       before upper=\csname @totalleftmargin\endcsname0pt, %
    overlay unbroken={\draw[colexam,line width=.5pt] ([xshift=-0pt]title.north west) -- ([xshift=-0pt]frame.south west); },
    overlay first={\draw[colexam,line width=.5pt] ([xshift=-0pt]title.north west) -- ([xshift=-0pt]frame.south west); },
    overlay middle={\draw[colexam,line width=.5pt] ([xshift=-0pt]frame.north west) -- ([xshift=-0pt]frame.south west); },
    overlay last={\draw[colexam,line width=.5pt] ([xshift=-0pt]frame.north west) -- ([xshift=-0pt]frame.south west); }}

\newtcolorbox[auto counter,crefname={Implication}{Implications}]{implication}[2][]{%
    empty,%
    title={Implication\ \thetcbcounter: #2},%
    label={#1},
    attach boxed title to top left,
       minipage boxed title,
    boxed title style={empty,size=minimal,toprule=0pt,top=4pt,left=3mm,overlay={}},
    coltitle=colexam,fonttitle=\bfseries,
    before=\par\medskip\noindent,parbox=false,boxsep=0pt,left=3mm,right=0mm,top=2pt,breakable,pad at break=0mm,
       before upper=\csname @totalleftmargin\endcsname0pt, %
    overlay unbroken={\draw[colexam,line width=.5pt] ([xshift=-0pt]title.north west) -- ([xshift=-0pt]frame.south west); },
    overlay first={\draw[colexam,line width=.5pt] ([xshift=-0pt]title.north west) -- ([xshift=-0pt]frame.south west); },
    overlay middle={\draw[colexam,line width=.5pt] ([xshift=-0pt]frame.north west) -- ([xshift=-0pt]frame.south west); },
    overlay last={\draw[colexam,line width=.5pt] ([xshift=-0pt]frame.north west) -- ([xshift=-0pt]frame.south west); },%
    }

\crefname{equation}{}{}
\crefname{proposition}{Proposition}{Propositions}
\crefname{lemma}{Lemma}{Lemmas}
\crefname{enumi}{}{}
\crefname{name}{}{} %

\newcommand{\tpdf}{\texorpdfstring}
\newcommand{\ncref}[1]{\cref{#1}: \nameref*{#1}} %

\def\ifcomments{\iftrue} %
\newcommand{\cj}[1]{\ifcomments \textcolor{blue}{CJ: #1} \fi}

\newcommand{\lm}[1]{\ifcomments \textcolor{cyan}{L: #1} \fi}
\newcommand{\notate}[1]{\lm{#1}}

\usepackage{lastpage}
\jmlrheading{25}{2024}{1-\pageref{LastPage}}{10/22; Revised
9/24}{12/24}{22-1123}{Alessandro Barp, Carl-Johann Simon-Gabriel, Mark Girolami, and Lester Mackey}
\ShortHeadings{Targeted Separation and Convergence with Kernel Discrepancies}{Barp, Simon-Gabriel, Girolami, and Mackey}

\firstpageno{1}

\begin{document}

\title{Targeted Separation and Convergence\\ with Kernel Discrepancies}

\author{\name Alessandro Barp$^\dagger$ \email alessandro.barp@ucl.ac.uk\\
      \addr University College London \& The Alan Turing Institute, GB
      \AND
      \name Carl-Johann Simon-Gabriel$^\dagger$$^\star$ \email cjsg@mirelo.ai\\
      \addr Mirelo AI
      \AND
      \name Mark Girolami \email mag92@cam.ac.uk \\
      \addr University of Cambridge \& The Alan Turing Institute, GB
      \AND
      \name Lester Mackey$^\dagger$ \email lmackey@microsoft.com \\
      \addr Microsoft Research, New England, US
      }

\editor{Bharath Sriperumbudur}

\maketitle

\def\thefootnote{$\dagger$}\footnotetext{These authors contributed equally.}\def\thefootnote{\arabic{footnote}}
\def\thefootnote{$\star$}\footnotetext{Work done while at AWS Lablets.}\def\thefootnote{\arabic{footnote}}

\begin{abstract}%
Maximum mean discrepancies (MMDs) like the
kernel Stein discrepancy (KSD)
have grown central to a wide range of applications,
including hypothesis testing, sampler selection, distribution approximation, and variational inference.
In each setting, these kernel-based discrepancy measures are required to $(i)$
separate a target $\mathrm{P}$ from other probability measures or even $(ii)$
control weak convergence to $\mathrm{P}$.
In this article we derive new sufficient and necessary conditions to ensure $(i)$ and $(ii)$.
For MMDs on separable metric spaces, we characterize those kernels that separate Bochner embeddable measures and
introduce simple conditions for separating all measures with unbounded kernels and for controlling convergence with bounded kernels.
We use these results on $\mathbb{R}^d$ to
substantially broaden the known conditions
for KSD separation and convergence control and to develop the first KSDs known to exactly metrize weak convergence to $\mathrm{P}$.
Along the way, we highlight the implications of our results for hypothesis testing, measuring and improving sample quality, and sampling with Stein variational gradient descent.
\end{abstract}


\section{Introduction}

Maximum mean discrepancies (MMDs)  %
like the Langevin kernel Stein discrepancy (KSD)  %
are kernel-based discrepancy measures widely used for
hypothesis testing \citep{gretton12,liu16kernelized,chwialkowski16kernel},
  sampler selection and tuning \citep{gorham17measuring}, parameter estimation \citep{briol2019statistical,barp2019minimum,dziugaite15training}, generalized Bayesian inference \citep{cherief2020mmd,matsubara2021robust,matsubara2022generalised,dellaporta2022robust},
 discrete approximation and numerical integration \citep{chen2019stein,chen18stein,barp2018riemannian},
 control variate design \citep{oates2014control,oates2019convergence,sun2023vector},
  compression~\citep{riabiz2022optimal}, and bias correction~\citep{liu2017black,Hodgkinson2020,riabiz2022optimal}.

 Each MMD uses a kernel function to
 measure the integration error between a pair of probability measures $\Q$ and $\P$, and,
  in each setting above, their successful application relies on either \emph{$\P$-separation}, that is $\mmd(\Q,\P) > 0$ whenever $\Q \neq \P$,
  or \emph{$\P$-convergence control}, namely $\mmd(\Q_n,\P) \to 0$ implies $\Q_n \to \P$ weakly.
 Unfortunately, these properties have so far only been established under overly restrictive assumptions, e.g.,
 for $\Q$ with continuously differentiable log densities~\citep{chwialkowski16kernel,liu16kernelized,barp2019minimum}, for $\P$ with strongly log concave tails and Lipschitz log density gradients $\sp = \partial \log \p$~\citep{gorham17measuring},
 or for bounded MMD kernels~\citep{sriperumbudur10hilbert,sriperumbudur16optimal,simon18kde,simon2023metrizing}.
  In this work, by fixing $\P$ as the target measure and allowing $\Q$ to vary, we establish new broadly applicable conditions for $\P$-separation and  $\P$-convergence control. Our main results include
  \begin{itemize}
 \item \tbf{Bochner $\P$-separation with MMDs:} \cref{thm:tight_convergence} exactly characterizes those
      MMDs that separate %
      $\P$ from Bochner embeddable measures on general Radon spaces.
 For MMDs with bounded kernels, this result exposes an important relationship between separation and convergence: separating $\P$ from all probability measures is equivalent to controlling $\P$-convergence for \emph{tight} sequences $(\Q_n)_n$.
 \item \tbf{Score $\P$-separation with KSDs:}
 \cref{coroKSDwc} shows that KSDs with standard characteristic kernels separate $\P$ from all measures $\Q$ that finitely integrate the score $\sp$.  This strengthens past work that only established separation from $\Q$ with continuously differentiable log densities \citep{chwialkowski16kernel,barp2019minimum}. %
 \item \tbf{$\L^2$ $\P$-separation with KSDs:}
 \cref{thm:L2 characteristic kernels,thm:KSD score Q formulation} show that KSDs with standard translation-invariant kernels separate $\P$ from all measures with densities $\q$ and finitely square-integrable $\q\sq$ and $\q\sp$.
 This strengthens past work that provided no examples of $\L^2$-separating kernels \citep{liu16kernelized}.
 \item \tbf{General $\P$-separation with MMDs:} \Cref{thm:bounded separating is P-characteristic} provides a simple sufficient condition for general $\P$-separation: any MMD---even one with an unbounded kernel---separates  $\P$ from \textbf{all} probability measures and controls tight convergence to $\P$ if the bounded functions in its associated reproducing kernel Hilbert space (RKHS) are $\P$-separating. %
 All of our remaining results explicitly check this new convenient condition.
 \item \tbf{General $\P$-separation with KSDs:} \Cref{Stein kernel control tight convergence via bounded P separation} shows that KSDs with standard translation-invariant kernels separate $\P$ from \tbf{all} probability measures and control tight $\P$-convergence whenever $\sp$ is continuous and grows at most root-exponentially.
 Prior $\P$-separation results applied only to a small subset of these targets, those with strongly log concave tails and Lipschitz $\sp$ \citep{gorham17measuring,HugginsMa2018,chen18stein}.
\item \tbf{Enforcing tightness with MMDs:} \Cref{tightness} provides a new sufficient condition for \emph{enforcing tightness}, i.e., for ensuring that $(\Q_n)_n$ is tight whenever $\mmd(\Q_n, \P)\to 0$: an MMD enforces tightness if elements of its RKHS suitably bound the indicators of compact sets.
Prior tightness-enforcing guarantees relied on a much stronger condition: the presence of a coercive (and hence unbounded) function in the RKHS \citep{gorham17measuring,HugginsMa2018,chen18stein,Hodgkinson2020}.
\item \tbf{Metrizing $\P$-convergence with KSDs:} Building on \Cref{tightness},  \cref{tilted-tightness} develops the first KSDs
known to \emph{metrize} weak convergence to $\P$ (i.e., $\KSD(\Q_n,\P)\to 0 \iff \Q_n\to\P$ weakly) by constructing \tbf{bounded} convergence-controlling Stein kernels.
Since all prior convergence-controlling KSDs featured unbounded Stein kernels, these are also the first KSDs known to satisfy the Stein variational gradient descent convergence assumptions of \citet{liu2017stein} (see \cref{svgd}).
\item \tbf{Failing to control $\P$-convergence:} Finally, \cref{thm:failure of convergence control} provides new necessary conditions for an MMD to control $\P$-convergence which notably fail to be satisfied when standard KSDs are paired with heavy-tailed targets.
\end{itemize}

As we highlight in the sections to follow, these results have immediate implications for a variety of inferential tasks in machine learning and statistics including goodness-of-fit testing (\cref{gof,gof2}), measuring and improving sample quality (\cref{quality}), and variational inference (\cref{svgd}).

\paragraph{Notation}
For a given separable metric space $\X$, we let $\C{}{}(\R^d)$ denote the space of continuous $\R^d$-valued functions on $\X$.
When $\X = \R^d$,
we say that the derivative of a set of $\R^\ell$-valued functions exists, if the functions in that set are differentiable, and
we additionally denote by
$\C{\ell}{}(\R^d)$ the space of
$\ell$-times continuously differentiable $\R^d$-valued functions on $\X$ (i.e., $f\in \C{\ell}{}(\R^d)$ if the partial derivatives of order $\ell$ of $f^i$ exist and are continuous for $i\in[d]\defn \{1,\dots,d\}$).
We let $\partial f$ denote the vector of partial derivatives of a function $f$, and, for each multi-index $p$, let $\partial^p f$ denote the $p$-th partial derivatives of $f$.
When $d=1$ or $\ell=0$ we will use the abbreviations $\C{\ell}{}\defn \C{\ell}{}(\R^1)$ or $\C{}{}(\R^d)\defn \C{0}{}(\R^d)$.
Decay requirements will appear as subscripts:  $\C{}{b}(\R^d)$, $\C{}{c}(\R^d)$, and $\C{}{0}(\R^d)$ will respectively denote the spaces of $\R^d$-valued continuous functions that are bounded, compactly supported, and vanishing at infinity.
Analogously, for each function $h : \X \to [0,\infty)$, $\C{}{h}(\R^d)$ and $\C{}{0,h}(\R^d)$ respectively denote the spaces of $\R^d$-valued  continuous functions $f$ with $f / (1+h)$ bounded or vanishing at infinity. 
Recall a function $f:\X \to \R^a$ vanishes at infinity if $\forall \epsilon >0$ there exists a compact set $C$ s.t., $\sup_{x \in C^c} \| f(x) \| \leq \epsilon $, where $C^c$ is the set complement of $C$, and $\| \cdot \|$ the Euclidean norm.
For any function of two arguments $\Kb(y,x)$, we write $\Kb_x \defn \Kb(\cdot, x)$, and $K \in \C{(1,1)}{b}(\R^d)$ if $ \partial_y^{p_y} \partial_x^{p_x} K(y,x)$ exists, is bounded, and is separately continuous for multi-indices satisfying $\|p_x\|_1,\|p_y\|_1 \leq 1$, where $\| \cdot\|_1$ is the Euclidean 1-norm (i.e., the multi-index absolute value).
Given a map $T: \mathcal S_1 \to \mathcal S_2$ between sets, we denote the image of $T$ by $T(\mathcal S_1) \defn \{ T(s) : s \in \mathcal S_1\}$.  
Given a measure $\mu$ and a $\mu$-integrable function $h$, we denote integration by $\mu h \defn \int h(x) \mu(\dd x) $, and we shall omit the domain of integration, which is always $\X$.
Some additional notation for the appendices is presented in \cref{app: notation}.

\section{Maximum Mean Discrepancies and Kernel Stein Discrepancies}
We begin by extending the usual notions of maximum mean discrepancy and kernel Stein discrepancy to accommodate both arbitrary probability measures $\Q$ and unbounded kernels.

Throughout, we let $\Pset$ the denote set of (Borel) probability measures on a separable metric space $\X$. Moreover, for any function $f:\X \to \R^\ell$, we let
  $\Pset_f \defn \{ \Q \in \Pset: \|f\| \in \L^1(\Q)\}$ denote the set of probability measures that finitely integrate $\norm{f} \defn \| \cdot \| \circ f$.

\subsection{Maximum mean discrepancies}
\label{sec:MMD}
Consider a (reproducing) kernel $\k$ on $\X$ with reproducing kernel Hilbert space $\HK$~\citep{aronszajn50trk,schwartz1964sous}.
Traditionally, the associated kernel MMD is defined as the worst-case integration error across test functions in the RKHS unit norm ball $\B_\k$ \citep{gretton12}:
\begin{align}\label{eq:usual MMD definition}
\mmd_\k(\Q,\P) \defn \sup_{h \in \B_\k} \left |  \Q h -  \P h \right |.
\end{align}
However, the expression $\Q h -  \P h$ is not well defined when either (i) both $\Q h$ and $\P h$ are infinite or (ii) $h$ is not integrable under $\Q$.
Unfortunately, both of these cases can occur when $\k$ is unbounded as $\B_\k$ then  necessarily contains an unbounded test function (see  \cref{thm:Characterization of bounded RKHS}).

Since we are interested in a fixed target measure $\P$, we address the first issue by focusing on kernels with finitely $\P$-integrable test functions, i.e., with $\B_\k \subseteq \L^1(\P)$.
To address the second issue, we extend the MMD definition \cref{eq:usual MMD definition} to all probability measures $\Q$ by taking the supremum only over the $\Q$-integrable elements of $\B_\k$, that is, $h$ with either $h_+ \defn \max \left( h,0 \right)  \in \L^1(\Q)$ or $h_-
 \defn \max \left( -h,0 \right) \in \L^1(\Q)$. In fact, since $\B_\k$ is a symmetric set, considering only $h$ with $h_+  \in \L^1(\Q)$ suffices to ensure $|\Q h-\P h| $ is well defined and belongs to $[0,\infty]$.

 \begin{definition}[Maximum mean discrepancy (MMD)]
 \label{def: definition MMD}
 For a given kernel $\k$, define the set of \emph{embeddable probability measures} $\embedpettis{\k} \defn \{ \Q \in \Pset : \H_\k \subseteq \L^1(\Q) \}$.
For any target measure $\P \in \embedpettis{\k}$,
 we define the \emph{maximum mean discrepancy} $\mmd_\k( \cdot \,, \P) : \Pset \to [0,\infty]$ by
 \begin{align}\label{eq:MMD definition}
 \mmd_\k(\Q,\P) \defn \sup_{h \in \B_\k :\, h_+ \in \,\L^1(\Q)  \ } \left | \Q h - \P h \right |.  %
\end{align}
 \end{definition}

\begin{remark}[Embeddability]
We show in \cref{app: embedding of distributions in RKHS} that (i) the embeddability condition $\P \in \embedpettis{\k}$
holds if and only if $x \mapsto \k(\cdot,x)$ is \emph{Pettis integrable by $\P$} and (ii) Pettis integrability in turn implies that the \emph{kernel mean} $\int \k(\cdot, x) \diff \P(x)$ belongs to the RKHS $\HK$. See \cref{thm:definitions of embeddability} for the definition of Pettis integrability.
\end{remark}
As we show in \cref{app: Bochner contained in Pettis}, one user-friendly sufficient condition for  $\Q \in \embedpettis{\k}$ is \emph{Bochner-embeddability}, that is,
$\Q\in\Pset_{\sk}$ where $\sqrt{\k}$ represents the function $x \mapsto \sqrt{k(x,x)}$.
When $\H_\k$ is separable, \citet{carmeli06vector} proved that one can alternatively check the weaker condition $\iint |\k ( x, y)| \dd \Q ( x)\dd \Q ( y) < \infty$.
The next proposition summarizes these convenient embeddability conditions.
\begin{proposition}[Embeddability conditions]
\label{thm: Bochner contained in Pettis}
The following claims hold true.
\begin{enumerate}[label=\textup{(\alph*)}]
\item $\Pset_{\sk}\varsubsetneq \embedpettis{\k}$.
\item If $\H_\k$ is separable,
    $\iint |\k ( x, y)| \dd \Q ( x)\dd \Q ( y)< \infty$ implies $\Q \in \embedpettis{\k}$ \citep[Cor.~4.3]{carmeli06vector}.
\end{enumerate}
\end{proposition}
\begin{remark}[Sufficient condition for separability]
Note that when $\X$ is a locally compact topological space,  for $\H_\k$ to be separable, it is sufficient that $\H_\k\subseteq\C{}{}$~\citep[Cor.~5.2]{carmeli06vector}.
Moreover, $\H_\k\subseteq\C{}{} \iff \k$ is locally bounded\footnote{Recall that a function $f$ from a topological space to a normed space is \emph{locally bounded} if every point in its domain has a neighbourhood $U$ for which the restriction of $f$ to $U$ is bounded.} and $\k_x\in\C{}{}$ for each $x$ \citep[Prop.~5.1]{carmeli06vector}.
\end{remark}
Moreover, when both $\Q$ and $\P$ are embeddable, the MMD can be re-expressed as a convenient double-integral~\citep[Prop. 13]{simon18kde}.
\begin{proposition}[MMD as a double integral]\label{thm:MMD as double integral}
If $\P\in\embedpettis{\k}$ and $\Q\in\embedpettis{\k}$, then
\begin{talign}
\mmd_\k^2(\Q,\P) =   \iint k(x,y)  \dd (\Q-\P)( x)  \dd (\Q-\P)(y). 
\end{talign}
\end{proposition}

\subsection{Kernel Stein discrepancies}
Building on the Stein discrepancy formalism of \citet{gorham15measuring} and the zero-mean reproducing kernel theory of \citet{oates2014control}, \citet{chwialkowski16kernel,liu16kernelized,gorham17measuring} concurrently developed special MMDs that can be computed without any explicit integration under the target $\P$.
When discussing these \emph{Langevin KSDs} we will restrict our focus to $\X = \R^d$ and assume the target $\P$ has a strictly positive
density $\p$ with respect to Lebesgue measure.
We will also make use of a matrix-valued kernel $\Kb:\R^d \times \R^d \to \R^{d\times d}$ which generates an RKHS $\H_{\Kb}$ of vector-valued functions; for an introduction to vector-valued RKHSes, please see \cref{app:vector-valued RKHS}.

The Langevin KSD is defined in terms of a matrix-valued \emph{base kernel} $\Kb$ and the differential operator
\begin{align}
 \langevin (v) \defn \frac1 \p \nabla \cdot (\p  v) \defn \frac1 \p \sum_j \partial_{x^j}(\p v^j),   
\end{align}
known as the \emph{Langevin Stein operator} in the machine learning and statistics communities \citep{gorham15measuring,anastasiou2023stein},
 which, under mild conditions, maps $\R^d$-valued functions $v=(v^1,\ldots, v^d):\R^d \to \R^d$ to $\R$-valued functions with mean zero under the target, $\P \langevin(v)=0$.
Specifically, for $\Kb$ chosen so that $\langevin (v)$ has expectation zero under $\P$ for each $v \in \H_\Kb$, \citet{chwialkowski16kernel,gorham17measuring,barp2019minimum} defined\footnote{The distinct definition of \citet{liu16kernelized} coincides with \cref{traditional-ksd} under the assumptions of their Thm.~3.8.}  the Langevin KSD as an integral probability metric~\citep{muller1997integral} over
$\langevin (\H_K)$:
\begin{align}\label{traditional-ksd}
    \ksd{\Q}
    \defn  \sup_{v \in \B_K} |\Q \langevin (v)|
    = \sup_{v \in \B_K} |\Q \langevin (v) -\P\langevin (v)  |.
\end{align}
However, $ \langevin (v)$ is often unbounded so that, for the same reasons described in \cref{sec:MMD},  the expression \cref{traditional-ksd} need not be well defined for all $\Q \in \Pset$.
To enable meaningful KSD evaluation for all probability measures, we follow the recipe of \cref{def: definition MMD} to extend the definition of KSD to all $\Q \in \Pset$.
\begin{definition}[Kernel Stein discrepancy (KSD)]\label{def:IPD definition KSD}
Consider a target $\P \in \Pset$ with density $\p>0$ and matrix-valued base kernel $K$ for which the set $\p \H_K = \{ \p h : h \in \H_K\}$ consists of partially differentiable functions.
When $\langevin(\H_K)\subseteq\L^1(\P)$
and $\P(\langevin (\H_K)) = \{0\}$, we define
 the \emph{kernel Stein discrepancy} $\KSD_{K,\P}: \Pset \to [0,\infty]$ by
\begin{align}\label{eq:IPM definition KSD}
    \ksd{\Q} \defn \sup_{v \in \B_K:\,
        \langevin (v)_+ \in \L^1(\Q)\ } \left | \Q \langevin (v) \right |.    %
\end{align}
\end{definition}

\begin{remark}[Relation to prior definitions of KSD]
\label{remark:relation to standard definitions of KSD}
For scalar kernels, \cref{def:IPD definition KSD} is identical to the definition of KSD given in two of the  papers that originally defined KSDs, \citet[Sec.~2.1]{chwialkowski16kernel} and \citet[Sec.~3.1 with $\norm{\cdot}=\norm{\cdot}_2$]{gorham17measuring}, except for the extra constraint $\langevin(v)_+\in \L^1(\Q)$ that we include simply to ensure that the $\ksd{\Q}$ is well defined for all probability measures $\Q$.   Moreover, for probability measures $\Q$ satisfying the constraint  $\langevin(v)_+\in \L^1(\Q)$ for all $v \in \B_K$ our definition exactly recovers those given by  \citeauthor{chwialkowski16kernel} and \citeauthor{gorham17measuring}.
However, unlike  \cref{def:IPD definition KSD}, the prior definitions of KSD from \citeauthor{chwialkowski16kernel} and \citeauthor{gorham17measuring} are not well defined for probability measures $\Q$ failing to satisfy the extra constraint, even though this restriction is not discussed explicitly in either work.
\end{remark}

Under additional assumptions, like Bochner embeddability of $\P$ and $\Q$ and continuous differentiability of  $K$ and $\p$, %
prior work showed
that the KSD \cref{eq:IPM definition KSD} is equivalent to an MMD with a scalar \emph{Stein kernel} $\ks$ and
that $\langevin (\H_K)$ defines a \emph{Stein RKHS} $\Hks$ of scalar-valued functions \citep{oates2014control,chwialkowski16kernel,liu16kernelized,gorham17measuring,barp2019minimum}.
Our next result, proved in \cref{app:ksdDef}, shows that \textbf{no} additional assumptions are necessary:
$\ksd{\Q} = \mmd_{\ks}(\Q, \P)$ and $\langevin (\H_K) =\Hks$
whenever the left-hand side quantities are well defined.

\begin{theorem}[KSD as MMD]\label{ksdDef}
    Consider a target $\P \in \Pset$ with density $\p>0$ and matrix-valued base kernel $K$ for which  $\p \H_K$ consists of partially differentiable functions.
    Then $\langevin (\H_K)$ is the \emph{Stein RKHS} $\H_{\ks}$ induced by the \emph{Stein kernel}\footnote{Note  we have $ \ks(x,y)={\sum\nolimits_{i,j=1}^d}\,
            \frac{1}{\p( x)\p( y)}\, \partial_{y^j} \partial_{x^i} (\p( x) \Kb_{ij}( x,  y) \p( y)).$}
    \begin{align}\label{eq:stein-kernel}
        \ks(x,y) \defn
    \frac1{\p( x)\p(y)} \nabla_y \cdot \nabla_x \cdot \left( \p( x)  \Kb( x, y) \p( y) \right).
    \end{align}
   Moreover, for target measures with zero-mean Stein RKHSes, i.e., for $\P$ in
    \begin{talign}
    \embedtozero
    \defn \{ \Q \in \Pset \stext{with density} \q>0\!:
    \, \q\H_{\Kb} & \stext{are partially differentiable functions,}  \\
    &
    \langevin[\q](\H_{\Kb})\subseteq\L^1(\Q), \text{and }
    \Q(\langevin[\q](\H_{\Kb})) = \{0\} \},
    \end{talign}
    the KSD matches the Stein kernel MMD:
    $$
    \ksd{\Q} = \mmd_{\ks}(\Q, \P)
    \qtext{for all} \Q \in \Pset.
    $$
\end{theorem}
\begin{remark}[Scalar kernel KSD]
    When $K= \kb \, \Id$ for a scalar kernel $\kb$, we will say that $\ks$ is induced by $\kb$ and write $\embedtozero[\kb]\defn\embedtozero[\kb \Id]$.  In this case,
    \begin{align} \label{eqn:stein-kernel-scalar}
        \ks( x,  y) = {\sum\nolimits_{i=1}^d}\,
            \frac{1}{\p( x)\p( y)}\,\dx \dy (\p( x) \kb( x,  y) \p( y)).
    \end{align}
\end{remark}

The zero-mean condition $\P\in \embedtozero$ ensures that all functions in the Stein RKHS integrate to zero under the target measure so that the KSD can be evaluated without any explicit integration under $\P$.
Moreover, by \cref{thm:MMD as double integral,ksdDef}, when $\Q$ embeds into the Stein RKHS, the KSD takes on its more familiar double integral form.
\begin{corollary}[KSD as a double integral]
If $\P\in\embedtozero$ and $\Q\in\embedpettis{\ks}$, then
\begin{talign}
\ksdsq{\Q} =   \iint \ks(x,y) \dd \Q(x) \dd \Q(y). 
\end{talign}
\end{corollary}

Finally, the following result proved in \cref{app:embeddability_conditions}    provides user-friendly sufficient conditions for verifying that $\P \in \embedtozero$,
which requires verifying that $\H_{\ks} = \langevin(\H_K) \subseteq \L^1(\P)$, and $\P(\H_{\ks})=0$.
Hereafter, we let $\sp \defn \partial \log \p$ denote the ``score'' function of $\P$ whenever $\log \p$ is partially differentiable.
\begin{proposition}[Stein embeddability conditions]\label{thm:embeddability_conditions}
Consider a target $\P \in \Pset$ with density $\p>0$ and  matrix-valued base kernel $K$ for which  $p\H_K$ consists of partially differentiable functions.
The following claims hold true.
\begin{enumerate}[label=\textup{(\alph*)}]
\item $\langevin (\H_{\Kb}) \subseteq \L^1(\P) \iff \P\in \embedpettis{\ks}$.
\item If $\P \in \Pset_{\sks}$, then $\P\in \embedpettis{\ks}$.
\item If $\P \in \Pset_{\sp}$ and all $v$ in $\H_K$ are bounded with bounded partial derivatives, then $\P \in \Pset_{\sks}$.

\item If $\P\in \embedpettis{\ks}$, then
$\iint \ks(x,y) \dd\P(x)\dd\P(y) = 0 \Leftrightarrow \P \in \embedtozero$.
\item If $\P\in \embedpettis{\ks}$ and $\H_K \subseteq \L^1(\P) \cap \C{1}{}(\R^d)$, then $\P \in \embedtozero$.
\end{enumerate}
\end{proposition}

\begin{remark}[User-friendly conditions on $\H_K$]
The requirements on $\H_K$ in \cref{thm:embeddability_conditions} (c) and (e) can often be verified by examining simple properties of the base kernel $K$. 
For example, by \cref{thm:Characterization of bounded RKHS}, 
all $v$ in $\H_K$ are bounded iff 
$x \mapsto \|\Kb(x,x)\|$ is bounded, and all $v$ in $\H_K$ have bounded $x^i$-partial derivatives if $(x,y) \mapsto \norm{\partial_{x^i}\partial_{y^i} K(x,y)}$ exists and is bounded for any matrix norm $\norm{\cdot}$.
 In particular, if $K \in \C{(1,1)}{b}(\R^d)$, then $\H_K \subseteq \C{1}{b}(\R^d)$.
 Moreover, by \citet[Thm. 2.11]{micheli2013matrix}, 
 $\H_K \subseteq \C{1}{}(\R^d)$ iff
  $(x,y) \mapsto \partial_{x^i}\partial_{y^i} K(x,y)$  is separately continuous and locally bounded.
\end{remark}

\section{Conditions for Separating Measures}
Our first goal is to identify when an MMD distinguishes $\P$ from other measures.
Given a set of probability measures $\M \subseteq \Pset$,
we will say that $\k$  \emph{separates $\P$  from} $\M$ if for any $\Q \in \M$,  $\mmd_\k(\Q,\P)=0$ implies that $\Q=\P$.
When $\k$ separates $\P$ from all probability measures $\Pset$ we say simply that $\k$ is \emph{$\P$-separating}.
We will first discuss restricted $\P$-separation---that is, separation from a distinguished subset of measures $\M\neq \Pset$---in \cref{sec:bochner-separation,sec:score-separation} and then turn to general $\P$-separation---separation from all probability measures $\Pset$---in \cref{sec:general-separation}.
\subsection{Bochner \tpdf{$\P$}{P}-separation with MMDs}\label{sec:bochner-separation}

Our first result,
proved in \cref{app:proof tight convergence},
exactly
characterizes the kernels that separate $\P$ from  Bochner embeddable measures on Radon spaces~\citep[Def.~5.1.4]{ambrosio05gradient}.

Recall that a set of probability measures $\M \subseteq \Pset$ is \emph{tight} when for each
$\epsilon >0$ there exists a compact set $S \subseteq \X$ such that $\Q(S^c)\leq \epsilon$ for all $\Q \in\M$.
We also say that a measurable function $\varphi:\X \to \R$ is \emph{uniformly integrable} by $\M\subseteq \Pset$
if for each $\eps > 0$ there exists $r>0$ such that
$\sup_{\mu \in \M} \int_{\{ x: \, |\varphi|( x)>r\}} |\varphi| \diff \mu  < \eps$.%

\begin{theorem}[Bochner \tpdf{$\P$}{P}-separation with MMDs]%
\label{thm:tight_convergence}
    Let $\k$ be a continuous kernel over a Radon space $\X$ (for example, a Polish space).
    Then $\k$  separates $\P \in \Pset_{\sk}$ from
    $\Pset_{\sk}$ iff, for any sequence $(\Q_n)_n \subseteq \Pset_{\sk}$,
    \begin{align}
        \label{eq:tight-convergence}
        \Q_nh \to \P h \quad \forall h \in \C{}{\sk}
        \qquad \Longleftrightarrow \qquad
        \left \{
        \begin{array}{ll}
            \mathrm{(a)} &\mmd_{\k}(\Q_n, \P) \rightarrow 0 \\
            \mathrm{(b)} &(\Q_n)_n \text{ is tight} \\
            \mathrm{(c)} &(\Q_n)_n
                \text{ uniformly integrates } \sqrt{\k} .
        \end{array}
        \right .
    \end{align}
\end{theorem}

\cref{thm:tight_convergence} exposes an important relationship between our two goals of separation and convergence control.
In particular, when $\k$ is bounded, the uniform integrability condition $(c)$ always holds, %
 $\Pset_{\sk}$ is the set of all probability measures $\Pset$, $\C{}{\sk}$ is the set of all bounded continuous functions, and the convergence on the left-hand side of \cref{eq:tight-convergence} is the usual weak convergence in $\Pset$.
Hence for bounded kernels we obtain \cref{thm:tight_convergence_bounded_kernels}: separating $\P$ from all probability measures is equivalent to \emph{controlling tight $\P$-convergence}, i.e., having $\Q_n \to \P$ weakly whenever
$\mmd_\k(\Q_n,\P) \to 0$ and $(\Q_n)_n$ is tight.

\begin{corollary}[\tpdf{$\P$}{P}-separation with bounded kernels]%
\label{thm:tight_convergence_bounded_kernels}
    Let $\k$ be a continuous bounded kernel over a Radon space $\X$ (for example, a Polish space).
    Then $\k$  separates $\P \in \Pset$ from
    $\Pset$ iff, for any sequence $(\Q_n)_n \subseteq \Pset$,
    \begin{align}
        \Q_nh \to \P h \quad \forall h \in \C{}{b}
        \qquad \Longleftrightarrow \qquad
        \left \{
        \begin{array}{ll}
            \mathrm{(a)} &\mmd_{\k}(\Q_n, \P) \rightarrow 0 \\
            \mathrm{(b)} &(\Q_n)_n \text{ is tight}.
        \end{array}
        \right .
    \end{align}
\end{corollary}

\begin{remark}[Comparison with \citet{simon2023metrizing}]
When $\X$ is also locally compact and Hausdorff, for instance when $\X =\R^d$, Theorem 9 in \citet{simon2023metrizing} implies that, if $\HK \subset \C{}{0}$ and $\k$ separates every finite measure $\mu$ from the set of finite measures, then $\k$ metrizes the weak convergence of probability measures (i.e., for every probability measure $\P$, $\mmd_{\k}(\Q_n,\P) \to 0 \iff \Q_n \to \P$ weakly). 
Comparing with \cref{thm:tight_convergence_bounded_kernels}, we observe no explicit tightness requirement appears. 
This is because the assumption of separation for \textbf{every finite measure} $\mu$ (instead of separation of a single finite measure $\P$ from $\Pset$) implicitly does the work of enforcing tightness.  
In the proof of \cref{thm:tight_convergence} we can see the role of tightness is to ensure relative compactness, which in turn allows us to use the existence of convergent subsequences to promote the separation assumption into a convergence control.
But $\Pset$ is a bounded and thus relatively compact subset of the space of finite measures \citep[Thm. 33.2]{treves67}.  Hence, by \citet[Prop. 32.5]{treves67}, the assumption of $\k$ separating all finite measures is enough 
  to guarantee the equivalence between $\Q_n h \to  \P h$ for all $h\in \HK$ and $\Q_n h \to  \P h$ for all $h\in \C{}{0}$.  The latter is further equivalent to $\Q_n h \to  \P h$ for all $h\in \C{}{b}$ by \citet[Cor. 2.4.3]{berg84harmonic}.
\end{remark}

\subsection{Score \tpdf{$\P$}{P}-separation with KSDs}\label{sec:score-separation}
The standard practice in the KSD literature is to identify easily-verified properties of the base kernel $\Kb$, target $\P$, and alternative measure $\Q$ that ensure separation.
One class of KSD separation conditions---introduced by \citet[Thm.~2.2]{chwialkowski16kernel} and generalized by \citet[Prop.~1]{barp2019minimum}---applies to measures that finitely integrate the score $\sp$ but additionally requires $\Q$ to have a continuously differentiable log-density.
The first main result of this work, proved in \cref{app: proof coroKSDwc}, removes the extraneous continuity conditions and extends $\P$-separation to \emph{all} measures $\Q\in\Pset_{\sp}$ under a standard separating assumption on the base kernel, \emph{$\DL{1}(\R^d)$-characteristicness}.

\begin{theorem}[Score \tpdf{$\P$}{}-separation with KSDs]%
\label{coroKSDwc}
Suppose a matrix-valued kernel $K$ with $\H_K \subseteq \C{1}{b}(\R^d)$ is $\DL{1}(\R^d)$-characteristic.
    If  $\P \in \embedtozero$,
    then $\ks$ separates $\P$ from $\Pset_{\sp}$.
\end{theorem}
\opt{considerlast}{\notate{can we achieve this conclusion for the set of bochner embeddable Q as well? see \url{https://steinkernelsquad.slack.com/archives/CRY6GATLN/p1657837065736219}}}

We provide formal definitions of $\DL{1}(\R^d)$ and characteristicness in \cref{def:D1L1} and \cref{def: characteristicness} respectively. 
    In brief, $\DL{1}(\R^d)$ is the  $d$-dimensional product of the space $\DL{1}$ of finite measures and their distributional derivatives\footnote{Distributional derivatives extend the usual notion of derivative to objects that are not smooth, in particular to non-smooth distributions $\Q$. When $\Q$ has a differentiable Lebesgue density $\q$ then we recover the usual derivative, $\partial_{x^j} \Q = \partial_{x^j} q \,  \dd x$, while in general  $\partial_{x^j} \Q$ will be a Schwartz distribution~\citep{schwartzTD}.}, and 
a $\DL{1}(\R^d)$-characteristic kernel is one that can separate any pair of $\DL{1}(\R^d)$ elements.
Our proof of \cref{coroKSDwc} builds on the kernel Schwartz distribution theory of \citet{simon18kde}, wherein the space $\DL{1}$ naturally arises from the construction of the Stein RKHS via the Langevin Stein operator $\langevin$.  Specifically, we show in \cref{app:From separating measures in Hks to separating Schwartz distribution in HK} that the Stein kernel $\ks$ separates $\P$ from $\Q\in\embedpettis{\ks}$ if and only if the base kernel $\Kb$ separates the Schwartz distribution $\sp \Q -  \partial_{x^j} \Q$
from the zero measure. Moreover, $\sp \Q - \partial_{x^j} \Q \in \DL{1}(\R^d)$ when $\Q \in \Pset_{\sp}$, which yields \cref{coroKSDwc}.

\begin{application}[gof]{Goodness-of-fit Testing}

In goodness-of-fit (GOF) testing, one uses a sequence of datapoints $X_1, \dots, X_n$ generated from a Markov chain to test whether the chain's stationary distribution $\Q$ coincides with a target distribution $\P$.
KSDs with $\DL{1}$-characteristic  translation-invariant base kernels are commonly used as GOF test statistics, and such tests are known to consistently reject $\Q= \P$ whenever $\KSD(\Q,\P) > 0$ \citep{chwialkowski16kernel,liu16kernelized,gorham17measuring}.
However, prior to this work, the separating condition $\KSD(\Q,\P) > 0$ had only been established
for a restricted class of alternatives (continuous $\Q\in\Pset_{\sks}$ with differentiable log densities satisfying $\Q(\norm{\sp-\sq}) < \infty$,  \citealt[Prop.~1]{barp2019minimum}) or a restricted class of targets ($\P$ with Lipschitz $\sp$ and strongly log concave tails, \citealt[Thm.~7]{gorham17measuring}).
The former restriction excludes discrete and  discontinuous $\Q$, as well as $\Q$ with tails heavier than $\P$ or non-differentiable densities.
Meanwhile, the latter restriction excludes $\P$ with tails heavier than or lighter than a Gaussian.
\cref{coroKSDwc} in the present work ensures that $\KSD(\Q,\P) > 0$ for \emph{any} $\P \in \embedtozero$ and $\Q \in \Pset_{\sp}$.
In particular, this accommodates discontinuous or non-smooth $\Q$ and all targets $\P$ for which the KSD \cref{eq:IPM definition KSD} is defined.
Moreover, \cref{coroKSDwc} holds for all $\DL{1}$-characteristic kernels, a strict superset of the $\C{1}{0}$-universal kernels  \citep[Def.~4.1]{carmeli2010vector} assumed in prior work.
\end{application}

Indeed, \citet[][Thm.~12, Tab.~1, and Cor.~38]{simon18kde} showed that any $ \C{1}{0}$-universal $\k$ and any $\C{(1,1)}{}$ translation-invariant $\k$ with fully supported spectral measure is $\DL{1}$-characteristic.
These results already cover all of the translation-invariant base kernels commonly used with KSDs including
Gaussian, inverse multiquadric (IMQ), log inverse, sech, \Matern, B-spline, and Wendland's compactly supported kernels.
Moreover, as we prove in \cref{app:composition-kernel-properties}, characteristicness to $\DL{1}$ is preserved under the following operations, which allows one to construct even more flexible base kernels.

\begin{proposition} [Preserving  characteristicness]\label{thm:composition-kernel-properties}
Suppose a matrix-valued kernel $K$ with $\H_K \subseteq \C{1}{b}(\R^d)$ is $\DL{1}(\R^d)$-characteristic.  Then the following claims hold true.
\begin{enumerate}[label=\textup{(\alph*)}]
\item  If $\tilt\in\C{1}{b}$ is strictly positive, then $\tilt( x) K( x,  y)
    \tilt( y)$ is $\DL{1}(\R^d)$-characteristic. %
\item  If $b : \R^d\to\R^d$ is a Lipschitz $\C{1}{}(\R^d)$-diffeomorphism, then the composition kernel $K(b( x), b( y))$  %
is $\DL{1}(\R^d)$-characteristic.
\item If $\k_j$ is $\DL{1}$-characteristic for $j\in [d]$, then  $\mathrm{diag}(\k_1,\ldots, \k_d)$
is $\DL{1}(\R^d)$-characteristic.
 \opt{considerlast}{\notate{\item add result that iff diagonal elements are char. then K is char.}}
\end{enumerate}
\end{proposition}

As a final remark on  \cref{coroKSDwc}, we note that the score embedding measures $\Pset_{\sp}$ and the Bochner embeddable measures $\Pset_{\sks}$ exactly coincide under mild conditions satisfied by every $\C{(1,1)}{}$ translation-invariant base kernel $K$.
See \cref{app:score equal bochner} for the proof of this result.

\begin{proposition}[Score vs.\ Bochner embeddability]\label{score equal bochner}
Under the assumptions of \cref{coroKSDwc}, $\Pset_{\sp} \subseteq \Pset_{\sks}$.
If, in addition, $x\mapsto\sqrt{\metric{\sp(x)}{K(x,x) \sp(x)}} / \|\sp(x)\|$ is bounded away from zero, then  $\Pset_{\sp} = \Pset_{\sks}$.
\end{proposition}
\opt{considerlast}{\cj{After Arxiving: In particular, if we add a kernel that ensures $\sks$-uniform integrability, similar to what we do in \cref{tilted-tightness}, then the new kernel would metrize $\ks$-weak convergence, which is even stronger than weak convergence.To do so, can we replace the $x^i y^i$ in \cref{tilted-tightness} by $\ks(x^i, y^i)$ or $\sp(x)^T \sp(y)$?}}

\subsection{\tpdf{$\L^2$ $\P$}{L2 P}-separation with KSDs}
\citet{liu16kernelized} introduced a second class of KSD separation conditions based on an $\L^2$ separating property of the base kernel.
We say that a matrix-valued kernel $\Kb$ is \emph{$\L^2(\R^d)$-integrally strictly positive definite (ISPD)} if
$\H_\Kb \subseteq \L^2(\R^d)$ and
\begin{talign}
    g\in \L^2(\R^d)
    \qtext{and} g \neq 0
    \quad\Rightarrow\quad \iint g(x)^T \Kb(x,y) g(y) \dd x \dd y > 0.
\end{talign}
Unfortunately, the $\L^2$ requirement on $\H_\Kb$
excludes certain popular base kernels  like slowly decaying IMQ and log inverse kernels, and  \citet{liu16kernelized} did not provide any examples of kernels satisfying the $\L^2$-ISPD conditions.
Our next result fills this gap by showing that many standard kernels are $\L^2$-ISPD, including Gaussian, \Matern, sech, B-spline, faster decaying IMQ, and Wendland's compactly supported kernels, along with their tilted variants.
The proof can be found in \cref{app:L2 characteristic kernels}.

\begin{theorem}[\tpdf{$\L^2$}{L2}-ISPD conditions]
\label{thm:L2 characteristic kernels}
The following claims hold true for a matrix-valued kernel $\Kb$.
\begin{enumerate}[label=\textup{(\alph*)}]
\item Suppose $(\k_j)_{j=1}^d$ are translation-invariant continuous kernels with $\H_{\k_j} \subseteq \L^2$.
If the  spectral measure of each $\k_j$ is fully supported, then $\Kb=\mathrm{diag}(\k_j)$ is $\L^2(\R^d)$-ISPD.
\item If $K$ is $\L^2(\R^d)$-ISPD and $A: \R^d \to \R^{d\times d}$ is bounded measurable with $A(x)$ invertible for each $x$, then the tilted kernel $A(x)K(x,y)A(y)^T$ is also  $\L^2(\R^d)$-ISPD.
\item  If $\H_\Kb$ is separable,  $\sup_x \| \Kb_x u \|_{\L^1}<\infty$, and $\Kb_x u \in \L^2(\R^d)$ for each $x$ and $u\in \R^d$, then $\H_K \subseteq \L^2(\R^d)$.
\item Suppose $\Kb_x u\in\L^1(\R^d)$ for some $u \in \R^d$.
If $\Kb$ is translation-invariant or, more generally, if $\Kb_x u$ is bounded, then $\Kb_x u\in\L^2(\R^d)$. 
\end{enumerate}
\end{theorem}

Previous results in the literature have focused on properties similar to but distinct from the $\L^2(\R^d)$-ISPD condition. 
These include conditions under which kernels are (i) $\L^p(\mu)$-ISPD for $p \in [1,\infty)$ with respect to a probability measure $\mu$ in place of Lebesgue measure \citep{carmeli2010vector}; 
(ii)  ISPD, meaning that $\iint \k(x,y) \dd \mu(x) \mu(y)>0$ for all non-zero finite measures $\mu$ \citep{sriperumbudur11}; 
 (iii) $\L^1$ integrally \emph{non-strictly} positive definite (INPD), meaning $g\in\L^1 \Rightarrow \iint g(x)^T \Kb(x,y) g(y) \dd x \dd y\geq0$~\citep{bochner1932vorlesungen,stewart1976positive}; 
(iv) $\L^2_c$-INPD where $\L^2_c$ is the space of compactly supported $\L^2$ functions~\citep{cooper1960positive};
or (v) $\L^2$-INPD for continuous $k\in\L^2$ when $d=1$ \citep[Rem.~2.10]{buescu2004positive} or translation-invariant $k$ with $k_x\in\L^1$ \citep[Thm.~2.5.1]{phillips2018positive}.
\opt{considerlast}{\notate{Could we extend the L2 results to kernels that do not have HK in L2 in a manner analogous to our MMD extension for unbounded kernels?}}
\citet[Prop.~3.3 \& Thm.~3.8]{liu16kernelized} showed that KSDs with $\L^2$-ISPD base kernels separate certain measures with continuously differentiable densities $\q$.
\cref{thm:KSD score Q formulation}, proved in \cref{app:KSD score Q formulation}, generalizes this finding to matrix-valued $\Kb$ and partially differentiable $\q$ and provides user-friendly $\L^2$ conditions for ensuring that $\Q$ can be separated.

\begin{theorem}[\tpdf{$\L^2$}{L2} \tpdf{$\P$}{P}-separation with KSDs]
\label{thm:KSD score Q formulation}
Suppose $\P\in \embedtozero[\Kb]$ for a matrix-valued kernel $\Kb$.  The following claims hold true.
\begin{enumerate}[label=\textup{(\alph*)}]
\item If $\Kb$ is $\L^2(\R^d)$-ISPD, then $\k_{\p}$ separates $\P$ from $\{\Q\in \embedpettis{\k_{\p}}\cap\embedtozero[\Kb]: (\sp-\sq)\q \in \L^2(\R^d)\}$.
\item If
$\Q\in\embedpettis{\k_{\p}}$, $(\sp-\sq)\q \in \L^2(\R^d)$ and $\H_{\Kb} \subseteq \L^2(\R^d)\cap \L^\infty(\R^d)$,
then $\Q \in \embedtozero$.
\item If $\H_{\Kb} \subseteq \L^2(\R^d)$, $\partial \H_{\Kb} \subseteq \L^\infty(\R^d)$,
 and
$\sp \q \in \L^2(\R^d)$, then
$\Q\in\embedpettis{\k_{\p}}$.
\end{enumerate}
\end{theorem}
While \cref{thm:KSD score Q formulation} only applies to continuous $\Q$, it does cover certain measures excluded by \cref{coroKSDwc}.
For example, \cref{thm:KSD score Q formulation} implies that Cauchy alternatives $\Q$ are separated from Gaussian targets $\P$ since $\q \norm{\sp}^2$ and $\q\norm{\sq}^2$ are bounded and hence $\q \sp, \q \sq\in\L^2(\R^d)$.
Meanwhile, \cref{coroKSDwc} cannot be applied to these $(\Q,\P)$ pairings as the heavy tails of a Cauchy cannot finitely integrate a Gaussian score $\sp$.

\subsection{General \tpdf{$\P$}{P}-separation}\label{sec:general-separation}

The results in the preceding sections only
yield general $\P$-separation when applied to bounded kernels,
and indeed this has been the standard in much of the MMD literature~\citep{sriperumbudur10hilbert,sriperumbudur16optimal,simon18kde,simon2023metrizing}.
To accommodate the unbounded Stein kernels that often arise in KSDs, our next definition and result (proved in \cref{app:bounded separating is P-characteristic}) provide a new, convenient means to check that \emph{unbounded} kernels separate $\P$ from $\Pset$.

\begin{definition}[Bounded \tpdf{$\P$}{P}-separating property]
We say a set of functions $\F$  is \emph{bounded $\P$-separating}  if $\L^\infty \cap \F$ is $\P$-separating,
i.e., if $\Q \in \Pset$ and $\Q h = \P h$ for all $h \in \L^\infty \cap \F$ then $\Q =\P$.
\end{definition}

\begin{theorem}[Controlling tight convergence with bounded  separation]
\label{thm:bounded separating is P-characteristic}
  If $\H_{\k}$ is bounded $\P$-separating, then $\k$ is $\P$-separating and controls tight $\P$-convergence.
\end{theorem}

According to \cref{thm:bounded separating is P-characteristic}, to establish general $\P$-separation, it suffices to restrict focus to the bounded functions in an RKHS.
Moreover, \cref{thm:bounded separating is P-characteristic} suggests a convenient  strategy for proving $\P$-separation with unbounded kernels $\k$: (i) identify a sub-RKHS of bounded functions that belongs to $\H_\k$  and (ii) appeal to a broadly applicable bounded-kernel result to establish the $\P$-separation of the bounded sub-RKHS.

To apply this strategy to KSDs, we first show in \cref{app:proof tilted controls tightness} that any suitably tilted $\DL{1}(\R^d)$-characteristic base kernel yields a bounded and $\P$-separating Stein kernel:

  \begin{theorem}[Controlling tight convergence with bounded Stein kernels]\label{thm: tilted controls tightness}
  Suppose a matrix-valued kernel $K$ with $\H_K \subseteq \C{1}{b}(\R^d)$ is $\DL{1}(\R^d)$-characteristic. %
If $\|\sp( x) \| \leq \growth(x)$  for $\growth\in\C{1}{}$ with  $ \frac{1}{\theta} \in \C{1}{b}$, then the Stein kernel induced by the tilted base kernel $\frac{\Kb(x,y)}{\growth(x)\growth(y)}$ is bounded and $\P$-separating and controls tight $\P$-convergence.
\end{theorem}

Next we show that standard translation-invariant base kernels have sub-RKHSes of precisely the form needed by \cref{thm: tilted controls tightness}:

\begin{theorem}[Translation-invariant kernels have rapidly decreasing sub-RKHSes]
\label{thm:Schwarz-tilted sub-RKHS}
Suppose a kernel $\k$ with $\H_{\k} \subseteq \C{1}{}$ is translation invariant with a spectral density bounded away from zero on compact sets.
Then  there exist a translation-invariant,  $\DL{1}$-characteristic kernel $\k_s\in\C{(1,1)}{}$  and, for each $ c>0$, a positive-definite
function $f$ with 
$\frac{1}{f} \in \C{1}{},$ 
$\max(|f(x)|, \norm{\partial f(x)}) = O(e^{-c\sum_{i=1}^d \sqrt{|x^i|}}),$ 
  and
$$\H_{\k_f}\subseteq\H_{\k_{fs}}\subseteq \H_{\k}
\qtext{for} \k_f( x, y) \defn f( x)\k_s( x, y)f( y)
\qtext{and} \k_{fs}( x, y) \defn \k_s( x, y) f( x- y).$$
\end{theorem}
\opt{considerlast}{\notate{Can we get rid of the "bounded away from zero on compact sets" part of our assumption?  This would allow us to handle second-order Stein kernels directly.
One possibility: instead of finding a single bounded sub-kernel kf that separates all measures, can we show that, for each alternative measure Q, there is a bounded sub-kernel kQ that separates Q from P?  That should greatly reduce the requirements on the original kernel k and in.}}
\cref{thm:Schwarz-tilted sub-RKHS} applies to all of the translation-invariant base kernels commonly used with KSDs including Gaussian, IMQ, log inverse, sech, \Matern, B-spline, and Wendland's compactly supported kernels.
Moreover, our proof in \cref{app:Schwarz-tilted sub-RKHS} explicitly constructs the $\DL{1}$-characteristic kernel $\k_s$ and the rapidly decreasing tilt function $f$ and may be of independent interest.

We now apply our Stein operator to the base kernels of \cref{thm:Schwarz-tilted sub-RKHS} and invoke \cref{thm: tilted controls tightness} to deduce the second main result of this work: KSDs based on standard translation invariant kernels achieve general $\P$-separation, even when their Stein kernels are unbounded.
The proof of this result can be found in \cref{app: Stein kernel control tight convergence via bounded P separation}.

\begin{theorem}[Controlling tight convergence with KSDs]
\label{Stein kernel control tight convergence via bounded P separation}
For $\kb$ as in \cref{thm:Schwarz-tilted sub-RKHS},
define the tilted kernel $\kb_\tilt( x, y) = \tilt(x)\kb(x,y)\tilt(y)$ for each
strictly positive $\tilt\in \C{1}{}$.
\begin{enumerate}[label=\textup{(\alph*)}]
 \item If $\P\in\embedtozero[\kb]$  and
$\| \sp \|$ has at most root exponential growth,\footnote{A function $\tilt$ has \emph{at most root exponential growth} if  $ \tilt(x) =O(\exp(c{\sum_{i=1}^d \sqrt{|x^i|}}))$ for some $c>0$.
} then the Stein kernel induced by $\kb$
 is bounded $\P$-separating and  controls tight $\P$-convergence.
    \item Moreover, if $\P\in\embedtozero[\kb_\tilt]$  and  $\tilt$, $\partial \tilt$, and
$\tilt \| \sp \|$ have at most root exponential growth, then the Stein kernel induced by $\kb_{\tilt}$
 is bounded $\P$-separating and  controls tight $\P$-convergence.
\end{enumerate}
\end{theorem}

\begin{application}[gof2]{Goodness-of-fit Testing, continued}

In the testing setting of \cref{gof},
\cref{Stein kernel control tight convergence via bounded P separation} extends the reach of KSD GOF testing by guaranteeing
$\KSD(\Q,\P) > 0$ for
\emph{all} alternatives $\Q$ whenever $\norm{\sp}$ has at most root exponential growth.
Since the Stein kernels of \cref{Stein kernel control tight convergence via bounded P separation} are also bounded $\P$-separating, the same consistency guarantees immediately extend to the computationally efficient stochastic KSDs of \citet[Thm.~4]{gorham2020stochastic}.
\end{application}

\section{Conditions for Convergence Control}

Having derived sufficient conditions on the RKHS to separate measures and control tight convergence, we now present both sufficient and necessary conditions to ensure that an MMD controls weak convergence to $\P$.
Hereafter, we will say that $\k$ \emph{controls weak convergence to $\P$} or \emph{controls $\P$-convergence} whenever
$\mmd_\k(\Q_n,\P) \to 0$ implies
$\Q_n \to \P$ weakly.
Moreover, we will say that $\k$ \emph{enforces tightness}
whenever
$\mmd_\k(\Q_n,\P) \to 0$ implies that $(\Q_n)_n$ is tight.
Enforcing tightness is central to our developments as, if $\k$ controls tight weak convergence to $\P$ and enforces tightness, then it also controls weak convergence to $\P$.

\subsection{Sufficient conditions}\label{sec:sufficient conditions for convergence control}

We begin by introducing a new sufficient condition to ensure that MMDs and integral probability metrics more generally enforce tightness.

\begin{definition}[$\P$-dominating indicators]\label{def:indic-approx}
Consider a set of  functions $\F \subseteq \L^1(\P)$.  We say that $\F$ \emph{$\P$-dominates indicators} if, for each $\epsilon > 0$, there exists a compact set $S \subseteq \X$ and a function $h \in \F$ that satisfy
\begin{talign}
h -\P h \geq \indic{ S^c} - \eps. \label{eq:indic-approx}
\end{talign}
\end{definition}

\cref{def:indic-approx} ensures that a sequence $(\Q_n)_n$ can only approximate $\P$ well if it places uniformly little mass outside of a compact set $S$.
As we show in \cref{app:proof of enforcing tightness}, this is sufficient to ensure that integral probability metrics like the MMD enforce tightness.
\begin{theorem}[Controlling \tpdf{$\P$}{P}-convergence by dominating indicators]\label{tightness}
If $\F\subseteq \L^1(\P)$ $\P$-dominates indicators then $(\Q_n)_{n}$ is tight whenever the integral probability metric
$$
d_{\F}(\Q_n, \P) \defn \sup_{h\in\F:\  h_+ \in\, \L^1(\Q_n) \textup{ or } h_{-}  \in \,\L^1(\Q_n) } |\Q_n h -\P h| \to 0.$$
Hence, if $\P\in\embedpettis{\k}$ and $\HK$ $\P$-dominates indicators then $(\Q_n)_{n}$ is  tight whenever $\mmd_{\k}(\Q_n, \P) \to 0$.
If, in addition, $\k$ controls tight $\P$-convergence, then $\k$ also controls $\P$-convergence.
\end{theorem}

We can now combine \cref{tightness} with any of our KSD tight convergence results to immediately obtain $\P$-convergence control for KSDs.

\begin{corollary}[Controlling $\P$-convergence with KSDs]\label{ksd-tightness}
Under the conditions of \cref{coroKSDwc}, \ref{thm: tilted controls tightness}, or \ref{Stein kernel control tight convergence via bounded P separation}, if $\H_{\ks}$  $\P$-dominates indicators, then $\ks$ controls $\P$-convergence.
\end{corollary}

Before we discuss applications of these results, let us compare them to existing results in the literature.
Prior work relied on a stronger, {coercive function} condition to establish that KSDs enforce tightness with generalized multiquadric \citep[Lem.~16]{gorham17measuring}, IMQ score (\citealt[Thm.~4]{chen18stein}; \citealt[Ex.~6]{Hodgkinson2020}), log inverse \citep[Thm.~3]{chen18stein}, or unbounded tilted translation invariant \citep[Thm.~3.2]{HugginsMa2018} base kernels.
\citet{Hodgkinson2020}
used the following general definition of coercivity.
\begin{definition}[Coercive function {\citep[Assump.~1]{Hodgkinson2020}}]
We say a function $h :\X \to \R$ is \emph{coercive} if, for any $M >0$, there exists a compact set $S \subseteq \X$ such that
$\inf_{x \in S^c} h(x) > M$.
\end{definition}
\begin{remark}[Bounded coercive functions]
Any continuous coercive function is also bounded below as continuous functions are bounded on compact sets.
\end{remark}
Our next result, proven in \cref{app:proof of enforcing tightness with coercivity},  shows that this coercive function condition is stronger than our $\P$-dominating indicator condition.

\begin{lemma}[Coercive functions dominate indicators] \label{coercive-tightness}
If $h\in\HK$ is coercive and bounded below and  $\P\in\embedpettis{\k}$, then $\HK$ $\P$-dominates indicators.
\end{lemma}

As a first application of \cref{ksd-tightness}, we show that KSDs with IMQ base kernels enforce tightness and control convergence whenever the dissipativity rate of the target dominates the decay rate of the kernel.
Generalizing the argument in \citet[Lem.~16]{gorham17measuring}, our proof in \cref{app:imq-tightness} explicitly constructs a coercive function in the associated Stein RKHS.

\begin{theorem}[IMQ KSDs control \tpdf{$\P$}{P}-convergence] \label{imq-tightness}
Consider a target measure $\P\in\Pset$ with score $\sp \in \C{}{}(\R^d) \cap \L^1(\P)$.
If, for some dissipativity rate $u > 1/2$ and $r_0, r_1,r_2 > 0$, $\P$ satisfies the \emph{generalized dissipativity} condition %
\begin{align}
    \label{eq:generalized-dissipativity}
-\inner{\sp( x)}{ x} - r_0 \onenorm{\sp( x)}\geq r_1\twonorm{ x}^{2u} - r_2 \qtext{for all}  x\in\R^d.
\end{align}
If $\kb(x,y) = (c^2 + \norm{ x- y}^2)^{-\gamma}$ for $c > 0$ and $\gamma \in (0, 2u-1)$, then $\Hks$ $\P$-dominates indicators and enforces tightness.
If, in addition, $\norm{\sp}$ has at most  root exponential growth, then $\ks$  controls $\P$-convergence.

\end{theorem}
\opt{considerlast}{\notate{May be clearer if we drop $- s \onenorm{\partial \log \p( x)}$ term from both generalized dissipativity conditions}}

\begin{application}[quality]{Measuring and Improving Sample Quality}

Because the KSD provides a computable quality measure that requires no explicit integration under $\P$, KSDs are now commonly used to select and tune MCMC sampling algorithms \citep{gorham17measuring}, generate accurate discrete approximations to $\P$ \citep{liu2016stein,chen18stein,chen2019stein,futami2019bayesian}, compress Markov chain output \citep{riabiz2022optimal}, and correct for biased or off-target sampling~\citep{liu2017black,Hodgkinson2020,riabiz2022optimal}.
Each of these applications relies on KSD  convergence control, but past work only established convergence control for $\P$ with Lipschitz $\sp$ and strongly log concave tails (\citealt[Lem.~16]{gorham17measuring}; \citealt[Thm.~3]{chen18stein}; \citealt[Thm.~3.2]{HugginsMa2018}).
Notably, these conditions imply generalized dissipativity \cref{eq:generalized-dissipativity} with $u=1$ but exclude all $\P$ with tails lighter than a Gaussian.
\cref{ksd-tightness,imq-tightness} significantly relax these requirements by providing convergence control for all dissipative $\P$ with lighter-than-Laplace tails.
\end{application}

Much of the difficulty in analyzing KSDs stems from the fact that all known convergence-controlling KSDs are based on unbounded Stein kernels $\ks$.
As a second illustration of the power of  \cref{ksd-tightness}, \cref{tilted-tightness} develops the first KSDs known to \emph{metrize $\P$-convergence} (i.e., $\KSD(\Q_n,\P)\to 0 \iff \Q_n \to \P$ weakly), by constructing \tbf{bounded} convergence-controlling Stein kernels.
The following theorem is proved in
 \cref{app:proof tilted tightness}.

\begin{theorem}[Metrizing \tpdf{$\P$}{P}-convergence with bounded Stein kernels] \label{tilted-tightness}
Consider a target measure $\P\in\Pset$ with score $\sp$
that, for some dissipativity rate $u > 1/2$ and $r, r_1,r_2 > 0$, satisfies the generalized dissipativity condition \cref{eq:generalized-dissipativity}.  %
Define the Stein kernel with base kernel $K( x, y) = \mathrm{diag} \left( a(\twonorm{ x}) ( x^i  y^i +\k( x,  y)) a(\twonorm{ y}) \right)$, i.e.,
\[
\ks( x, y) =
	\sum_{1 \leq i \leq d} \frac{\dx \dy(p( x) a(\twonorm{ x}) ( x^i  y^i +\k( x,  y)) a(\twonorm{ y}) p( y) )}{p( x)p( y)},
\]
for $\k$ characteristic to $\DL{1}$ with $\HK \subseteq \C{1}{0}$
\opt{considerlast}{\notate{any reason why we need C10 and not C1b as we have in other statements? Alessandro and I initially discussed this, and should work for C1b as the functions we're approximating compactly supported functions, and we can approximate them arbitrarily well by choosing the right decaying functions gamma.  But then we discovered we were potentially mistaken, and it might be difficult
}}
and $a(\twonorm{ x}) \defn (c^2 + \twonorm{ x}^2)^{-\gamma}$ a tilting function with $c > 0$ and $\gamma \leq u$.
The following statements hold true:
\begin{enumerate}[label=\textup{(\alph*)}]
    \item If $\P\in\embedtozero$, then $\H_{\ks}$ $\P$-dominates indicators and enforces tightness.
    \item If $\P\in\embedtozero$, $\gamma \geq 0$, and $\twonorm{\sp(x)} \leq (c^2 + \twonorm{ x}^2)^{\gamma}$, then $\ks$ is bounded $\P$-separating and controls $\P$-convergence.
    \item If $\twonorm{\sp( x)}\cdot \twonorm{x} \leq (c^2 + \twonorm{ x}^2)^{\gamma}$  and $\sp\in\C{}{}$, then $\H_{\ks} \subseteq \C{}{b}$  and $\ks$ metrizes $\P$-convergence.
\end{enumerate}
\end{theorem}

\begin{application}[svgd]{Sampling with Stein Variational Gradient Descent}

Stein variational gradient descent (SVGD)  is a popular technique for approximating a target distribution $\P$ with a collection of $n$ representative particles.
The algorithm proceeds by iteratively updating the locations of the particles according to a simple rule determined by a user-selected KSD.
\citet{liu2017stein} showed that the SVGD approximation converges weakly to $\P$ as the number of particles and iterations tend to infinity, provided that the chosen KSD controls $\P$-convergence \textbf{and} that the Stein kernel is bounded.
However, prior to this work, no bounded convergence-controlling Stein kernels were known.
\cref{tilted-tightness} therefore provides the first instance of a Stein kernel satisfying the SVGD convergence assumptions of \citet{liu2017stein}.
\end{application}

\subsection{Necessary conditions}
\label{sec:necessary conditions for convergence control}

We finally conclude with a necessary condition for an MMD to control weak convergence to $\P$,
which recovers and broadens the KSD failure derived by
\citet[Thm.~6]{gorham17measuring}.
For each RKHS $\H_\k \subseteq \L^1(\P)$,
define the $\P$-centered RKHS $\H_{\k^\P} \defn\{ h-\P h : h\in \H_\k\}$
with $\P$-centered kernel
\begin{talign}
\k^\P(x,y) \defn \k(x,y) - \int\k(x,y) \dd\P(y) - \int \k(x,y) \dd\P(x) + \iint \k(x,y) \dd\P(x)\dd\P(y).
\end{talign}
\cref{thm:failure of convergence control} shows that $\k$ \textbf{fails} to control $\P$-convergence whenever its $\P$-centered RKHS functions all vanish at infinity; notably, this occurs whenever $\k^\P$ is bounded with $\k^\P_x\in \C{}{0}$ for each $ x$ \citep[Prop.~3]{simon18kde}.
The proof in \cref{app:failure of convergence control} relies on the fact that $\k$ and $\k^\P$ induce exactly the same MMD.
\opt{considerlast}{\notate{does this result also hold w/o continuity of the rkhs functions, just w/ vanishing at infinity? Relatedly, does part b of this result also hold more generally if $\k^\P(x,\cdot)\in \C{}{0}$ is vanishing at infinity for each x (even if $\k^\P(x,x)$ is not uniformly bounded)? Also, does $\k^\P(x,\cdot)\in \C{}{0}$ is vanishing at infinity for each x imply that all functions in the RKHS vanish at infinity (even if they need not be continuous)? This question was motivated by showing that for example Stein kernels with Gaussian base kernels do not control weak convergence in dimension 3 or greater; in that case, the Stein kernel is not uniformly bounded and the rkhs is therefore not in C0.}}

\begin{theorem}[Decaying \tpdf{$\P$}{P}-centered kernels fail to control \tpdf{$\P$}{P}-convergence]\label{thm:failure of convergence control}
Suppose that $\X$ is locally compact but not compact. 
  If $\H_{\k^\P} \subseteq \C{}{0}$, then  $\k$ does not control $\P$-convergence.

\end{theorem}

\begin{implication}[heavy]{Standard KSDs fail for heavy-tailed $\P$!}

Since Stein kernels are already $\P$-centered by design (i.e., $(\ks)^\P=\ks$),
\cref{thm:failure of convergence control} holds dire consequences for standard KSDs with heavy-tailed targets $\P$.
As noted by \citet[Thm.~10]{gorham17measuring}, if the score function is bounded (as is common for super-Laplace distributions), then the KSD fails to control $\P$-convergence whenever a $\C{1}{0}$ base kernel is used.
Moreover, our more general \cref{thm:failure of convergence control} result implies that if the score function is decaying (as is true for any Student's t distribution), then the KSD fails to control $\P$-convergence for any bounded base kernel.
This result suggests that the standard KSD practice of using a $\C{1}{0}$ base kernel is unsuitable for heavy-tailed targets and that one should instead choose a base kernel with growth sufficient to counteract the decay of $\sp$.
\end{implication}

\section{Discussion}
This article derived new sufficient and necessary conditions for kernel discrepancies to enforce $\P$-separation and control $\P$-convergence.
We characterized all MMDs that separate $\P$ from Bochner embeddable measures,
proposed novel sufficient conditions for separating all measures and enforcing tightness,
strengthened all prior guarantees for KSD separation and convergence control on $\R^d$,
and derived the first KSD known to exactly metrize (as opposed to strictly dominating) weak $\P$-convergence on $\R^d$.

These developments point to several opportunities for further advances.
First, while we have focused on weak convergence in this article, we believe many of the tools and constructions can be adapted to study the control of other modes of convergence.
Natural candidates include \emph{$\alpha$-Wasserstein convergence} \citep{ambrosio05gradient}, i.e., weak convergence plus the convergence of $\alpha$ moments, and  \emph{$\C{}{\sqrt{\k}}$ convergence}, i.e., expectation convergence for all continuous test functions bounded by $\sqrt{\k}$.
When $\sqrt{\k}$ is unbounded, \cref{thm:tight_convergence} exposes an important relationship between separation and $\C{}{\sqrt{\k}}$ convergence: $\P$-separating Bochner embeddable measures is equivalent to controlling  $\C{}{\sqrt{\k}}$ convergence to $\P$ for sequences that uniformly integrate $\sqrt{\k}$.
Hence, to control $\C{}{\sqrt{\k}}$ convergence, it remains to identify those kernels that simultaneously separate and enforce uniform integrability.

Second, while we have focused on canonical KSDs defined by the Langevin Stein operator and a bounded base kernel, our tools are amenable to analyzing other kernel-based Stein discrepancies like the diffusion KSDs of \citet{gorham2019measuring,barp2019minimum}, the second-order KSDs studied in~\citet{barp2018riemannian,liu2018riemannian,Hodgkinson2020,barp2022geometric}, the gradient-free KSDs of \citet{han2018stein,fisher2022gradient}, and the random feature Stein discrepancies of~\citet{HugginsMa2018}.
In fact, employing a diffusion KSD with an unbounded diffusion coefficients is one promising way to overcome the heavy-tailed-target failure mode highlighted in \cref{heavy}.

Finally, while we have focused on KSDs for measures defined on $\R^d$, the very recent work of~\citet{wynne2022spectral} provides a template for studying measure separation on  infinite-dimensional Hilbert spaces.

\acks{We thank Bharath Sriperumbudur for suggesting the extent of the correspondence between $L^2$ integrally strictly positive definiteness and characteristicness for translation-invariant kernels and Heishiro Kanagawa for identifying several typos in an earlier version of this manuscript.
AB and MG were supported by the Department of Engineering at the University
of Cambridge, and this material is based upon work supported by, or in part by, the U.S. Army Research Laboratory and the U.S. Army Research Office, and by the U.K. Ministry of Defence and under the EPSRC grant [EP/R018413/2].
AB gratefully acknowledges support from the Turing-Roche strategic partnership.}

\newpage

\appendix

\section{Appendix Notation}
\label{app: notation}

Throughout we denote by $(e_\ell)_\ell$ the canonical basis of $\R^d$ and by $(e^\ell)_\ell$ its dual basis.

The spaces $\C{\ell}{c}(\R^d)$ and $\C{\ell}{0}(\R^d)$  will be equipped with their canonical topologies.
However on $\C{\ell}{b}(\R^d)$ we will  use the strict topology,
written $\C{\ell}{b}(\R^d)_{\beta}$, because, for $
\ell=0$, its dual is the space of finite (Radon) measures~\citep{conway1965strict} whenever $\X$ is a locally compact Hausdorff space (e.g., when $\X = \R^d$).
Note in general, any topology between the weak paired topology and the Mackey topology yields the space of finite measures as its (continuous) dual~\citep[Sec.~4]{buck1958bounded}.

In fact we will often use a generalization of $\C{\ell}{b}(\R^d)$:
given a continuous function
$\growth:\R^d \to [c,\infty)$ for some $c>0$,
we will need to construct  a generalisation of the space $\C{1}{b}(\R^d)_\beta$, denoted
$\C{1}{b,\growth}(\R^d)_\beta$, and defined as the vector space of $\C{1}{}(\R^d)$ functions for which
$\growth f \in \C{}{b}(\R^d) $, and
$\partial f \in \C{}{b}(\R^{d\times d})$,
with the topology defined by the family of seminorms
$$ \| f \|\defn \sup_x \| \gamma(x) \growth(x) f(x) \|, \qquad \| f \| \defn  \sup_x \| \gamma(x)  \partial^p_x f \|  $$
where $\gamma \in \C{}{0}$ and $|p| =1$. 
In other words $f_\alpha \to f$ in
$\C{1}{b,\growth}(\R^d)_\beta$ iff
$(\growth f_\alpha, \partial f_\alpha ) \to  (\growth f,  \partial f )$ in $\C{}{b}(\R^d)_\beta \times \C{}{b}(\R^d)_\beta$.
We mention that in \cref{def:B^1_theta} we will similarly construct $\BB{1}{\theta}(\R^d)$, a Banach space that plays a similar role to $\C{1}{b,\theta}(\R^d)$ but is simpler to work with (however it is not general enough for our purposes).

Given a topological vector space (TVS) $\F$, its (continuous) dual will be denoted $\F^*$.
Given a subset $\M \subseteq \F^*$, and $D_\alpha, D \in \M$ we will write $ D_\alpha \overset{\M}{\to} D$ when
$D_\alpha(f) \to D(f), \forall f \in \F$ (i.e., $D_\alpha$ converges to $D$ in weak star topology).
When $\F =\C{}{b}$, and $\M = \Pset$, we say that \emph{$D_\alpha$ converges weakly to $\D$}.
More generally, we define \emph{weak convergence in
$\Pset_{\sk}$} (notice the ``in $\Pset_{\sk}$'' part!) using $\C{}{\sk}$, where
$\C{}{\sk}$ (resp. $\C{}{0, \,\sk}$) is the space of continuous
functions $f$  with $1+\sqrt{\k}$ growth, i.e.,  such that
$ \nicefrac{f}{(1+\sqrt{\k})}$ is bounded,  (resp. in $\C{}{0}$).
Thus
\begin{align}
    \Q_n \overset{\Pset_{\sk}}{\to} \P
        \quad \iff \quad
        \Q_n, \P \in \Pset_{\sk}\ ,
        \quad \text{and} \quad
        \Q_n(f) \to \P(f) \quad \forall f \in \C{}{\sk} .
\end{align}
Notice that $\C{}{b} \subseteq \C{}{\sk}$ and $\Pset_{\sk} \subseteq \Pset$ with
equality \emph{if and only if (iff)} $\k$ is bounded.
Recall here that, for a $\R^\ell$-valued function $f$ such as $\sqrt{\k}$, $\Pset_f \defn \{\Q \in \Pset : \| f \| \in \L^1(\Q)\} $.

Given TVSs $\F_1$ and $\F_2$, we denote by $\B(\F_1,\F_2)$ the set of continuous linear functionals from $\F_1$ to $\F_2$. The transpose of a continuous linear functional $T$ is denoted $T^*$.

Given a Radon measure $\mu$ on $\R^d$, its distributional $x^i$-derivative will be denoted $\partial_{x^i} \mu:\C{\infty}{c} \to \R$.
Recall that the distributional derivative is equal to $\partial_{x^i} \mu =  - \mu \circ \partial_{x^i}$ on $\C{\infty}{c}$.

\section{Vector-Valued RKHSes and Stein RKHSes}
\label{app:vector-valued RKHS}

Let $\X$ an open subset of $\R^d$.
Let $\Gamma(Y)$ denotes the set of maps $\X \to Y$.
Matrix-valued kernels are typically defined via a  feature map, i.e.,  a map $\xi^{*}: \X \to \B(\H, \R^d)$ (see \cref{thm:definitions of embeddability}), which generates the kernel
$$ K(x,y) \defn \xi^{*}(x) \circ \xi(y).$$
In particular if $\H \subseteq \Gamma(\R^d)$ is a RKHS of $\R^d$-valued functions, i.e., a Hilbert space on which the evaluation functionals $\delta_x  \in \B(\H_K,\R^d)$ are continuous, then $\H \defn \H_K$ where $K(x,y) \defn \delta_x \circ \delta_y^* \in \R^{d\times d}$.
The transpose  of $\delta_y$ is usually denoted
$K_y \defn \delta_y^*$, and $K_y^v \defn \delta_y^*(v) \in \H_K$, so
$K_xv(y)= \delta_y K_x v = \delta_y \delta_x^*v = K(x,y)^*v = K(y,x)v$ for any $v \in \R^d$,
thus $K_x = K(\cdot,x)$.
We can tilt matrix-valued kernels via a matrix-valued function $m \in \Gamma(\R^{d\times d})$, indeed $\xi^{*}_m \defn m \circ \xi^{*}$ is a new feature map, and its kernel is
$$ K_m(x,y) \defn  \xi^{*}_m(x) \circ \xi_m(y) =
m(x) \circ \xi^{*}(x) \circ \xi(y) \circ m^T(y) = m(x)K(x,y)m(y)^T. $$

Given an RKHS $\H_K$ of continuously differentiable $\R^d$-valued functions, we can obtain a scalar-valued kernel via the Stein operator $\langevin$.\footnote{Note $\langevin$ is a special instance of the canonical operator associated to ``measures" equivalent to the Lebesgue one with differentiable densities (or more precisely, the canonical operator induced by positive 1-densities)~\citet{barp2022geometric}.} 
Let $\tilde \xi_{\P}^m \defn \langevin \circ m \circ \tilde \xi : \H_K \to \Gamma(\R)$,
where
$\tilde \xi(h)(x) \defn \xi^{*}(x)(h) $.
Then $\xi_\P^m : \X \to \H_K$ is a feature map for the Stein kernel
$\ks$  \citep{barp2019minimum}, i.e.,
$$ \ks(x,y) = \metric{\xi^m_\P(x)}{\xi^m_\P(y)}_K.$$
Since the matrix $m$ just corresponds to a change of matrix kernel $K \mapsto K_m$, we can restrict to the identity case
$\xi_P \defn \xi_\P^{\mathrm{Id}}$.
In other words, for  the family of ``diffusion'' Stein operators \citep{gorham2019measuring}
$$
\langevin ^m(v) \defn \frac1 \p \nabla \cdot (\p m v) \ ,
$$
the matrix-valued  function $m$   can be thought of as a transformation of the  base RKHS $\H_K$ into $ \H_{mKm^T}$, i.e.,
$$  \langevin ^m ( \H_K) =  \langevin ^{\Id}(m \H_K) =  \langevin ^{\Id}(\H_{mKm^T}).$$
Since $K$ is arbitrary, without loss of generality we may choose
$m=\Id$, $ \S_\p \defn \S_\p^{\Id}$.
Note that the matrix functions $m$ obtained by the generator of $\P$-preserving diffusions
can be characterized on any manifold \citep{barp2021unifying}.

Similarly, the Stein kernel obtained via the second-order Stein operator \citep{barp2018riemannian} can be recovered by setting $K$ to be the diagonal matrix kernel of partial derivatives of a scalar kernel. 
We finally recall the equivalence between universality, characteristicness, and strict positive definiteness of (scalar-valued) kernels~\citep[Thm. 6]{simon18kde}, noting it carries on to the case of matrix-valued kernels.

\section{Embedding Schwartz Distributions in an RKHS}
\phantomsection
\label{app: embedding of distributions in RKHS}

Given a continuous linear map $T$ between TVS, we denote by $T^*$ its transpose, and, similarly, if  $h$ belongs to a Hilbert space, we will denote by $h^*$ the associated element in the dual space, i.e., $h^*(f) \defn \ipd{h}{f}$ for any $f$ in that Hilbert space.
\begin{definition}[Kernel embeddings and Pettis integrals]
\label{thm:definitions of embeddability}
    Let $D$ be a linear functional on a vector space $\F$
    containing the RKHS $\H_K$ of a matrix-valued kernel $K$.
\begin{enumerate}[label=\textup{(\alph*)}]
\item  We say that $D$ embeds into $\H_K$ if $D |_{\H_K}$ is continuous, i.e.,\ if
    there exists a function $\Phi_K(D) \in \H_K$ such that
    for all $h \in \H_K$: $D(h) = \metric{\Phi_K(D)}{h}_K$.
    We call $\Phi_K$ the kernel embedding and $\Phi_K(D)$ the (kernel or RKHS) embedding of $D$.
    It is given by
    \begin{talign}
    \Phi_K(D)(x) =  \sum_i e_i D (K_x^{e_i}).
    \end{talign}
\item
   Given a \emph{feature map}, i.e., a function $\xi:\X \to \B(\R^d, \H_K)$,
   we denote by $\xi^*:\X \to \B(\H_K,\R^d)$ the map $x \mapsto \xi(x)^*$ and define the \emph{feature operator} $\tilde \xi : \H_K \to (\X \to \R^d)$ as $\tilde \xi(h)(\cdot)\defn \xi^*(\cdot)(h)$.
   We say $\xi$
   is
    \emph{Pettis-integrable} with respect to $D$
    if $\tilde \xi(\H_K)\subseteq\F$  and the linear functional  $D \circ \tilde  \xi $   embeds into $\H_K$.
    The RKHS embedding, $\Phi_K(D \circ \tilde  \xi )$, of $D \circ \tilde  \xi$ is known as the Pettis-integral of $\xi$ with respect to $D$. We will also call the map from $D \mapsto \Phi_K(D \circ \tilde  \xi )$  the RKHS embedding of $\xi$. %
\end{enumerate}
\end{definition}

When $\M$ is a set of embeddable linear functionals,  for any $D,\tilde D \in \M$ we can define
$$  \mmd_K(D,\tilde D) \defn \| D-\tilde D \|_K \defn  \| \Phi_K(D)-\Phi_K(\tilde D) \|_K,$$
where $\Phi_K:\M \to \H_K$ is the kernel embedding,
which recovers  \cref{eq:MMD definition} when $\Q$ and $\P$ are embeddable probability measures.
In that case, $k$ separates $\P$ from $\M$ iff $\Phi_k(\, \cdot \, - \, \P ) |_\M$
vanishes only at $\P$.

Hereafter, we will say that a kernel is \emph{characteristic} to a set of embeddable linear functionals $\M$ when the RKHS embeddings of two distinct elements in $\M$ are always distinct.

\begin{definition}[Characteristicness]
\label{def: characteristicness}
Given a set $\M$ of embeddable linear functionals (see \cref{thm:definitions of embeddability}),
we say $K$ is \emph{characteristic to $\M$} when
$\Phi_K$ is injective over $\M$.
\end{definition}

When $\mu$ is a finite ($\R$-valued) measure on $\X$, then a natural set of functions that $\mu$ can act on is the set of finitely $\mu$-integrable functions $\L^1(|\mu|)$.
Now, if a function $\function{\xi}{\X}{\HK}{}{}$ is to be Pettis-integrable by $\mu$, then the very least is that the functions $\tilde \xi(h)$ be contained in $\L^1(|\mu|)$ for every $h \in \HK$.
Interestingly, we will now see that, because $\H_{\k}$ is a Hilbert space (not just Banach), this condition is also sufficient to guarantee $\mu$-Pettis integrability.

\begin{proposition}[Finite measures embed into $\HK$ iff $\HK$ is finitely integrable]
\label{thm:measure embed iff RKHS integrable}
    Let $\mu$ be a finite $\R$-valued measure (e.g., a probability measure), seen as a linear functional over $\L^1(|\mu|)$.
    Then a function $\function{\xi}{\X}{\HK}{}{}$ is $\mu$-Pettis integrable if and only if $\tilde \xi(\HK) \subseteq \L^1(|\mu|)$.
    In particular, if $\mu = \Q \in \Pset$, then the following claims hold.
    \begin{enumerate}
        \item Using  $\xi: x \to \k_x $, it follows that $\Q$ is embeddable into $\HK$ iff $\HK \subseteq \L^1(\Q)$.
        \item If $\partial_{x^i}\H_\k$ exists, then via
        $\xi: x \to \partial_{x^i} \k_x$ we obtain that
        $\partial_{x^i} \Q$ embeds iff $\partial_{x^i} \H_\k \subseteq \L^1(\Q)$.
    \end{enumerate}
\end{proposition}

\begin{proof}
    Since Hilbert spaces are canonically isomorphic to their dual (i.e.,\ $\H^{*} = \H$), Gelfand-integration and Pettis-integration coincide.
    Therefore, Proposition~3.4 in \citet{musial02pettis} --~which asserts that every scalarly $\mu$-integrable function $\xi: \X \to \H_\k \cong \H_\k^*$ is Gelfand $\mu$-integrable~-- concludes the first part.

    Then (1) follows directly from  $\ipdK{\k_x}{h} = h(x)$.
    For (2),
    note that if $\partial_{x^i}\H_\k$ exists, then $\xi: x \mapsto \partial_{x^i} \k_x \in \H_\k $
  and $\metric{h}{\partial_{x^i} \k_x}_{\k} = \partial_{x^i}h(x)$ by \cref{thm:Differential reproducing property}.
  Thus $\tilde \xi = \partial_{x^i}$ so  $\partial_{x^i} \H_\k \subseteq \L^1(\Q)$ iff it is Gelfand $\Q$-integrable, in which case
  \begin{talign}
    -\partial_{i} \Q h = \int  \partial_i h \dd \Q = \int \metric{h}{\partial_{x^i} \k_x}_{\k} \dd \Q(x) = \metric{h}{ \int \xi \dd \Q}_\k 
  \end{talign}
  where $ \int \xi \dd \Q$ is the Pettis integral. Hence $\partial_i \Q$ embeds into $\H_\k$.
\end{proof}

The embeddability of distribution in a Stein RKHS can be analysed in terms of the embeddability of the associated (via pull-back) Schwartz distributions in the base RKHS, as  \cref{Continuous linear functional shifted by Feature Operator} shows, by generalizing  \citep[Prop.~14]{simon18kde}.

\begin{lemma}[Embedding functionals on RKHS defined by feature maps]
\label{Continuous linear functional shifted by Feature Operator}
    Let $\H$ be a Hilbert space, $\H_\Kb$ an RKHS of $\R^\ell$-valued functions on $\X$, and $\xi: \X \to \mathscr{B}(\R^\ell,\H)$ be a feature map for $\Kb$, i.e., $\Kb( x, y) = \xi^*( x)\circ \xi( y)$.
    Then  a linear functional $D: \H_\Kb \to \R$  embeds into $\H_\Kb$ iff $D \circ \tilde \xi : \H \to \R$ embeds into $\H$.
    Here $\tilde \xi: \H \to \H_\Kb$ is the feature operator (see \cref{thm:definitions of embeddability}).

    For any $\Q \in \Pset$, $$ \Q(\sqrt{k}) = \Q(\| \xi \|_\H) $$
    so $\Q$ is Bochner integrable in $\H_\k$ iff $\| \xi \|_\H \in \L^1(\Q)$.

    Moreover,   if $D$ embeds into $\H_\Kb$ then the transpose of $\tilde \xi$ is an isometry:
    $$\| D \|_{\H_{\Kb}} = \| D \circ \tilde \xi \|_\H.$$
    In particular, if $D=\Q \in \Pset$ and  $\H_\k \subseteq \L^1(\Q)$, then
    \begin{talign}
    \| \Q\|_{\H_k} = \| \int \xi \diff \Q  \|_\H.
    \end{talign}
\end{lemma}

See \cref{app:proof Continuous linear functional shifted by Feature Operator} for the proof.
Applying this result to a Stein RKHS, we immediately obtain the following corollary.

\begin{corollary}[Embedding measure in Stein RKHS and base RKHS]\label{ksdAlt}
Consider a Stein kernel $\ks$ \cref{eq:stein-kernel} with base kernel $K$, and 
fix any $\Q\in\Pset$. The following are equivalent:
    \begin{enumerate}[label=\textup{(\alph*)}]
        \item $\Q$ embeds into $\H_{\ks}$.
    \item $\Q$ embeds into $\H_K$ via the feature map $\xi_\P: \X \to \H_K$ with  $\xi_\P( x) = K_{ x} \sp(x)+ \nabla_{ x} \cdot K_x$.
    \end{enumerate}
    If either holds and $\P\in\embedtozero$, then
    \begin{talign}
        \ksd{\Q} = \| \int \xi_\P \, \mathrm{d} \Q \, \|_{\H_K},
    \end{talign}
    where $\int \xi_\P \diff \Q $ is the Pettis integral.
\end{corollary}

We now formally introduce $\DL{1}(\R^d)$, the $d$-dimensional product space of finite measures and their distributional derivatives.
\begin{definition}[The space $\DL{1}(\R^d)$]
\label{def:D1L1}
We write $\DL{1}(\R^d)$ to represent the vector space of continuous linear functionals on $\C{1}{0}(\R^d)$ or, equivalently, on $\C{1}{b}(\R^d)_\beta$ and define $\DL{1} \defn \DL{1}(\R^1)$.
Notably, $D \in \DL{1}(\R^d)$ iff it can be expressed as a finite sum $D = \sum_{j=1}^l D_j e^j$  where each $D_j$ is a finite Radon measure on $\R^d$ or a distributional derivative thereof.
The topology on  $\DL{1}(\R^d)$ is the canonical dual topology induced by $\C{1}{0}(\R^d)$~\citep[pg. 200]{schwartzTD}.

Following \cref{def: characteristicness}, when the elements of $\DL{1}(\R^d)$ embed into $\H_K$, for instance when $\H_K \subseteq \C{1}{b}(\R^d)$, we shall say that $K$ is characteristic to $\DL{1}(\R^d)$ when the kernel embedding $\Phi_K : \DL{1}(\R^d) \to \H_K$ is injective.
\end{definition}

Importantly, for embeddable probability measures $\Q$, the KSD is given by the norm of a vector $D_\Q$ that 
can be understood as a distributional derivative of $\Q$ with respect to a differential operator induced by $\P$. 
When $\Q$ is smooth $D_\Q$, will be a vector measure, but when $\Q$ is not assumed to be smooth,  $D_\Q$ will be a more general (vector) Schwartz distribution. 
The space $\DL{1}(\R^d)$ assumes a central  role
in analysing $D_\Q$ and determining when kernel discrepancies separate $D_\Q$ from zero. This in turn helps us understand when KSDs effectively  distinguish the target $\P$ from alternatives $\Q$.

To define $D_\Q$, note that since the feature operator of $\xi_\P$ in \cref{ksdAlt} is the Stein operator $\langevin$,
setting
\begin{talign}
    D_\Q |_{\H_K} \defn \Q \circ \langevin  : \H_K \to \R
\end{talign}
we obtain that the KSD is given by evaluating the norm of $D_\Q$ in the base RKHS,
\begin{talign}
        \ksd{\Q} = \| D_\Q |_{\H_K}\, \|_{\H_K}.
    \end{talign}

    More generally, $D_\Q$ can act on any function $f \in \C{1}{}$ such that  $\langevin (f)  \in  \L^1(\Q)$, and we will omit $|_{\H_K}$ when we do not specify its domain of definition.
In addition, observe that
when  $\| \sp \|$ is integrable with respect to the probability measure $\Q$, then both
$\spi\Q$ and $\Q$ are finite measures. 
Consequently, using the distributional derivative,  we can write
\begin{talign}
D_\Q = \sum_i (\sp^i \Q - \partial_{x^i}\Q)e^i \defn \sum_i D_i e^i
\end{talign}
 with  $D_i \in \DL{1}$, the space of finite measures and their distributional derivatives.
 Hence, 
 $D_\Q$ is a (vector) Schwartz distribution that belongs to the space $\DL{1}(\R^d)$.

When
$\Q$ also has a strictly positive differentiable density with respect to the Lebesgue measure, then $D_\Q$ simplifies to a vector measure absolutely continuous with respect to the Lebesgue measure,
\begin{talign} 
D_\Q = \sum_i (\sp^i- \sq^i) \Q \, e^i. \end{talign}

The following lemma provides bounds on $\Q(\sqrt{ \ks})$ in terms of the base kernel $K$ and the target score $\sp$, and thus sufficient conditions for a probability measure to be able to Bochner integrate $\ks$.
See \cref{app: proof bound on integral root kp} for the proof.

\begin{proposition}[Bochner embeddability vs. score integrability]\label{bound on integral root kp}
   Consider a Stein kernel $\ks$ \cref{eq:stein-kernel} with base kernel $K$, and fix any $\Q\in\Pset$.
    We have
    \begin{talign}
    \Q(\sqrt{\ks})  
    &\leq  \int \left(
    \| K_x \|_{\mathrm{op}} \| \sp(x)\|_{\R^d} + \| \nabla_x \cdot K_x \|_K \right) \Q(\dd x),
    \qtext{and} \\
     \Q(\sqrt{\ks}) 
     &\geq \int \left | \sqrt{\metric{\sp(x)}{K(x,x)\sp(x)}_{\R^d}} - \| \nabla_x \cdot K_x \|_K \right | \Q(\dd x) .
    \end{talign}
    Now suppose $\H_K \subseteq \cts^1_b(\R^d)$. 
    Then the following claims hold.
    \begin{itemize}
        \item The maps $x \mapsto \| K_x \|_{\mathrm{op}}  = \sqrt{\| K(x,x)\|} $ and  $x \mapsto  \| \nabla_x \cdot K_x \|_K$ are bounded.
        \item If $\Q(\| \sp \|) < \infty $, then $\ks$ is Bochner integrable by $\Q$, i.e., $\Q(\sqrt{\ks}) < \infty$.
        \item If $x \mapsto K(x,x)$ is uniformly positive definite (i.e., $\exists c>0$ such that for all $v\neq 0\in \R^d$, $v^T K(x,x) v \geq c \| v \|_{\R^d}^2>0$ for all $x$), then $\Q(\sqrt{\ks})<\infty$ implies $\Q(\| \sp \|) < \infty$.
    \end{itemize}
    Hence, if $K$ is a diagonal kernel with each component satisfying $\inf_x \kb^i(x,x) >0$,
    then $\Q(\sqrt{\ks}) < \infty$ iff $\Q(\norm{\sp}) < \infty$, i.e.,
    $\Pset_{\sks} = \Pset_{ \sp  }$.
     In particular, if $K = \kb \Id$ where $\kb$ is translation-invariant (and not equal to the null function), then
 $\Pset_{\sks} = \Pset_{ \sp  }$.
\end{proposition}

\begin{remark}[Scalar base kernel norms]
When $K= \kb \mathrm{Id}$, we have
    \begin{talign}\| K_x \|_{\mathrm{op}}   =  \sqrt{\kb(x,x)}, \quad \text{ and } \quad  \sqrt{\metric{\sp(x)}{K(x,x)\sp(x)}_{\R^d}} = \sqrt{\kb(x,x)} \| \sp(x) \|_{\R^d}.\end{talign}
\end{remark}

\subsection{Proof of \ncref{Continuous linear functional shifted by Feature Operator}}
\phantomsection
\label{app:proof Continuous linear functional shifted by Feature Operator}

 By \citep[Prop 1]{carmeli2010vector}, $\tilde \xi$ is a surjective partial isometry from $\H$ onto $\H_K$. Hence it is continuous, and $\tilde \xi |_{\ker \tilde \xi^{T}}: \ker \tilde \xi^{T} \to \H_K$ is an isometric isomorphism, where $\ker \tilde \xi^{T}$ is the orthogonal complement to the kernel of $\tilde \xi$.
    If $D$ is continuous, so is $D \circ \tilde \xi$ since it is the composition of continuous maps.
    For the converse,
    note that $ \tilde \xi \circ (\tilde \xi |_{\ker \tilde \xi^{T}})^{-1}: \H_K \to \H_K$ is the identity so
    $D = D \circ \tilde \xi \circ (\tilde \xi |_{\ker \tilde \xi^{T}})^{-1}  $, which is continuous if $D \circ \tilde \xi$ is.

    For the second claim,
    we apply the first claim to $D \defn \Q |_{\H_\k} $.
    Noting that the RKHS embedding $\Phi_\k(\Q) \in \H_\k$ is the function $x \mapsto \Q \k_x$, we have
    \begin{talign}
    \| \Q \|_{\H_k}^2 &= \Q ( x \mapsto \Q k_x) = \iint \k(x,y) \Q(\dd y) \Q(\dd x) =
     \iint \metric{\xi(x)}{\xi(y)}_\H \Q(\dd y) \Q(\dd x) \\
     &=
     \iint \tilde \xi (\xi(x)) (y) \Q(\dd y) \Q(\dd x) =
      \int \Q \circ \tilde \xi (\xi(x)) \Q(\dd x) =
      \int  \metric{(\Q \circ \tilde \xi)^*}{ \xi(x)}_\H \Q(\dd x) \\
      &=\int  \tilde \xi((\Q \circ \tilde \xi)^*)(x)  \Q(\dd x) =
        \| \Q \circ \tilde \xi \|_\H^2,
     \end{talign}
     where as usual $(\Q \circ \tilde \xi)^*$ denoted the embedding of $\Q \circ \tilde \xi$ into $\H$.

  To generalise the above to $\Q$ being any embeddable functional $D$:
  letting $\xi_\delta: \X \to \B(\R^d,\H_{\Kb})$ denote the canonical feature map, $\xi_\delta(x) = \Kb(\cdot,x)$, then
  $$\xi_\delta = \tilde \xi \circ \xi,$$
  since for any $x,y \in \X$, $c \in \R^d$
 $(\xi_\delta(y)c)(x)
 = \xi^*_\delta(x) \xi_\delta(y)c
 = K(x,y)c =  \xi^*(x) \xi(y)c
 = \tilde \xi(\xi(y)c)(x)$.
   Moreover
   $$
   \Kb_x^{e_i} =\tilde \xi \xi(x)e_i
   $$
   and if $S$ is embeds into $\H$ and $\xi_f:\X \to\B(\R^d, \H)$ is a feature map for $\Kb$, then
   $$ \tilde \xi_f( S^*)=e_i S \circ \xi_f(\cdot)e_i,$$
   since
   $ \tilde \xi_f(S^*)(x) = \xi_f^*(x)(S^*)= e_i(\xi_f^*(x)(S^*))^i =e_i \metric{e_i}{\xi_f^*(x)(S^*)} =  e_i S\xi_f(x)e_i$.
   Hence
   $$
    \| D\circ \tilde \xi \|^2_\H
    =
    D \circ \tilde \xi  (D\circ \tilde \xi)^*
    =
    D e_i D\circ \tilde \xi \circ \xi(\cdot)e_i
    =    D e_i D \Kb_{\cdot}^{e_i}
    = DD^*=
   \| D \|_{\H_k}^2.
   $$

\subsection{Proof of \ncref{bound on integral root kp}}
\phantomsection
\label{app: proof bound on integral root kp} 

The fact that
$x\mapsto \| K_x \|_{\mathrm{op}}$  and $x \mapsto  \| \nabla_x \cdot K \|_K$ are bounded follows from  \cref{thm:Characterization of bounded RKHS}:

\begin{lemma}[RKHS boundedness conditions]
\label{thm:Characterization of bounded RKHS}
If $\H_K$ is a RKHS of $\R^d$-valued functions,
then the following claims hold.
\begin{enumerate}[label=\textup{(\alph*)}]
    \item $\H_K \subseteq \L^\infty(\R^d)$ iff
  $x \mapsto \| K(x,x) \|$ is bounded.
  \item If $\partial_{x^\ell} \H_K$ exists, then
  $\partial_{x^\ell} \H_K \subseteq \L^\infty(\R^d)$
  iff   $x \mapsto \| \partial_{x^\ell} \partial_{y^\ell} K(x,y)\|$ is bounded.
  \item If $\partial_{x^\ell} \partial_{y^\ell}K$ exists, then $\partial_{x^\ell} \H_K$ exists.
  \item If $K \in \C{(1,1)}{b}(\R^d)$, then $\H_K \subseteq \C{1}{b}(\R^d)$.
\end{enumerate}

\end{lemma}
\begin{proof}
    (a) If $\H_K \subseteq \L^\infty(\R^d)$, proceeding as in \cref{app:proof of characteristic Stein kernel via B1a spaces}, we have
    $\|K^*_x h \| = \| h(x) \| \leq \| h \|_\infty $ for any $h \in \H_K$,
    so the Banach-–Steinhaus Theorem  implies $\sup_x \|K_x^*\| = \sup_x \sqrt{\| K(x,x) \|}$ is finite.
    Conversely, when
    $x \mapsto \| K(x,x) \|$ is bounded,
    then $\| h(x) \| = \| K_x^*h \| \leq \| h \|_K \| K_x^* \| \leq  \|h \|_K \sqrt{\| K(x,x) \|} \leq  \|h \|_K  \sup_x \sqrt{\| K(x,x) \|} $.

    (b) Similarly, if $\H_K$ is a RKHS of differentiable functions,
    then $\partial_{x^\ell} H_K$
    is a RKHS with matrix-valued kernel
    $(x,y)\mapsto \partial_{x^\ell} \partial_{y^\ell} K(x,y)$ by
    \cref{thm:Differential reproducing property}.
    Thus, from above, $\partial_{x^\ell} H_K \subseteq \L^\infty(\R^d)$ iff
    $x \mapsto \| \partial_{1^\ell} \partial_{2^\ell} K(x,x)\|$ is bounded.

    (c) If  $\partial_{1^\ell} \partial_{2^\ell}K$ exists, then the argument of \citet[p.~8 near Eq.~(5)]{micheli2013matrix}
    shows $\partial^\ell h$ exists for all $h\in \H_K$.

    (d) If $K \in \C{(1,1)}{b}(\R^d)$, then $\H_K \subseteq \C{1}{}(\R^d)$ by \citet[Thm.~2.11]{micheli2013matrix}, and by above
    $\H_K \subseteq \C{}{b}(\R^d)$.
    Proceeding as above, we have of any $v \in \R^d$,
    $ |\metric{v}{\partial^p h(x)}| = |\metric{h}{\partial^p_2 K^{v}(.,x)}|
     \leq \| h \|_K \| \partial^p_2 K^{v}(.,x) \|_K
     = \| h \|_K \sqrt{v^T \partial^p_1 \partial^p_2 K(x,x) v }$
     which is bounded in $x$,
     and thus $\partial^p h$ is bounded.
\end{proof}

Now, by definition, since $\xi_\P$ is a feature map for $\ks$,
\begin{talign}
\Q(\sqrt{\ks}) = \int \sqrt{\metric{\xi_\P(x)}{\xi_\P(x)}_K} \dd \Q = \Q( \| \xi_\P\|_K).
\end{talign}
Recall $\xi_\P(x)= K_x \sp(x)+ \nabla_x \cdot K$.
By the triangle inequalities
\begin{talign}
\int \left| \| K_x \sp(x) \|_K - \| \nabla_x \cdot K \|_K \right| \Q(\dd x) \leq  \Q( \| \xi_\P\|_K) \leq  \int \left( \| K_x \sp(x) \|_K + \| \nabla_x \cdot K \|_K \right) \Q(\dd x). 
\end{talign}
The result follows  by continuity of $K_x =\delta_x^* \in \mathscr{B}(\R^d,\H_K)$, and the assumptions on $K$:
$$
 \| K_x \|_{\mathrm{op}} \| \sp(x)\|_{\R^d} \geq \| \delta_x^* \sp(x)\|_K = \sqrt{ \metric{\delta_x^* \sp(x)}{\delta_x^* \sp(x)}_K} = \sqrt{\metric{\sp(x)}{K(x,x)\sp(x)}_{\R^d}}.$$

 \subsection{Proof of \ncref{score equal bochner}}
 \phantomsection
 \label{app:score equal bochner}

The result follows by \cref{bound on integral root kp}.

 \subsection{Proof of \ncref{thm: Bochner contained in Pettis}}
 \label{app: Bochner contained in Pettis}

  That $\Pset_{\sk}\subseteq \embedpettis{\k}$ follows directly from the embeddability of the Dirac measures:
    $\P |h| = \P | \xi_\delta^*(h)|  \leq \| h\|_\k  \P \| \xi_\delta^* \|_{\mathrm{op}}
    =\| h\|_\k  \P \sqrt k$, where $\xi_\delta^* : x \mapsto \delta_x |_{\HK}$ is the canonical feature map.
    When $\H_\k$ is separable, then by
    \citet[Cor.~4.3 and Prop.~4.4]{carmeli06vector}
    a sufficient condition for $\Q$ to embed is $|\k| \in \L^1(\Q \otimes \Q)$.
    Note that $\Q \in \Pset_{\sk}$ implies $|\k| \in \L^1(\Q \otimes \Q)$,
    as $\iint |\k(x,y)| \dd \Q(x) \dd \Q(y)
    = \iint |\metric{\k_x}{\k_y}_\k |
    \dd \Q(x) \dd \Q(y)
\leq \iint  \| \k_x \|_\k \| \k_y \|_\k  \dd \Q(x) \dd \Q(y)
=
\iint \sqrt{ \k(x,x)} \sqrt{ \k(y,y)}  \dd \Q(x) \dd \Q(y) = (\Q \sqrt{\k})^2$.

    The following example is adapted from \citep[p.204]{berlinet04reproducing}.
Take $\k(i,j) = \indic{i = j}$ for $i,j \in \N^*$, i.e.\ $\HK = \ell^2(\N^*)$,
and consider the Radon measure $\mu(i) = \nicefrac{1}{i}$. Then it is easy to
see that $\mu$ is Pettis-embeddable (and satisfies $|\mu|\otimes|\mu|(\k) <
\infty$), but that it is not Bochner-embeddable into $\HK$, since
$|\mu|(\sqrt{\k}) = \infty$.

\subsection{Proof of \ncref{ksdDef}}
\label{app:ksdDef}

First let us show the following differential reproducing property, which is a mild generalization of results in \citep{steinwart08support,zhou2008derivative,micheli2013matrix}. 
    In contrast to the results provided in these references, \cref{thm:Differential reproducing property} does not assume continuity of the derivatives.
    We note that the proof is essentially identical to the continuous derivative case in \citet[Thm.~2.11]{micheli2013matrix} (which also deals with matrix-valued kernels).
    This  generalized form is important for our results as it allows us to establish that MMD and KSD coincide under no additional conditions than those required for them to be well-defined.

\begin{lemma}[Differential reproducing property]
\label{thm:Differential reproducing property}
Suppose that $ \partial_{x^\ell} \H_K$ exists.
Then for all $c \in \R^d$ and $h \in \H_K$
\begin{align}
    \metric{\partial_{x^\ell} K(\cdot,x)c}{h}_K = \metric{c}{\partial_{x^\ell} h(x)}.
\end{align}
Moreover  $ \partial_{x^\ell} \H_K$ is a RKHS with kernel
$(x,y) \mapsto \partial_{x^\ell} \partial_{y^\ell} K(x,y)$.
\end{lemma}
\begin{proof}
   Given a sequence $\epsilon_n>0$ with $\epsilon_n \to 0$, define $\Delta_{\epsilon_n} \defn \left(K(\cdot,x+\epsilon_n e_\ell)c- K(\cdot,x)c\right)/\epsilon_n \in \H_K $.
   Since $K(\cdot,x)c \in \H_K$  its  partial derivative in direction $e_\ell$ exists, and thus $K(y,\cdot)c$ also has a partial derivative in direction $e_\ell$, hence
   $\Delta_{\epsilon_n}$ converges pointwise to
   $\partial_{x^\ell} K(\cdot,x)c$.
   Moreover, for any $h\in \H_K$, $\metric{\Delta_{\epsilon_n}}{h}_K = \metric{c}{(h(x+\epsilon_n e_\ell)-h(x))/\epsilon_n}$ converges to $\metric{c}{\partial_{x^\ell} h(x)}$ as $\epsilon_n \to 0$ (since $\partial_{x^\ell} h$  exists).
   By the Banach–Steinhaus theorem
   $\{\Delta_{\epsilon_n}\}_n$ is thus a bounded subset of $\H_K$,
   and by \citet[Cor.~2.8]{micheli2013matrix}
   it follows that
   $\partial_{x^\ell} K(\cdot,x)c \in \H_K$ and
   $\metric{c}{ \partial_{x^\ell} h(x)}=\lim_{n \to \infty} \metric{\Delta_{\epsilon_n}}{h}_K  = \metric{\partial_{x^\ell} K(\cdot,x)c}{h}_K $ for all $h \in \H_K$.

   We now show $\xi^*\defn  \partial_{x^\ell}:\X \to \mathcal B(\H_K, \R^d)$, with $\xi^*(x) \defn \partial_{x^\ell}|_x$, is a feature map with associated kernel $\partial_{y^\ell} \partial_{x^\ell} K$.
   By above $\metric{e_i}{\xi^*(x) }$ is continuous for all $x,i$, so indeed $\xi^*$ is $ \mathcal B(\H_K, \R^d)$-valued.
   Note $\tilde \xi = \partial_{x^\ell}$,
   so by \citet[Prop.~2.4]{carmeli06vector}
 $\partial_{x^\ell} \H_K$ is a RKHS with kernel $\overline K$ s.t.,
$$
\overline K(x,y)c = \xi^*(x) \xi(y)c =\partial_{x^\ell}|_{x}
\partial_{y^\ell} K(\cdot,y)c=
\partial_{x^\ell}
\partial_{y^\ell} K(x,y)c.$$
\end{proof}

 Now we apply \cref{thm:Differential reproducing property} to the RKHS  $\p\H_K$ whose kernel is $\p(x)\Kb(x,y) \p(y)$ (see \cref{app:vector-valued RKHS}),
 in order to show that
 $\xi_\p : \X \to \H_K$ defined by
 \begin{talign}
\xi_\p(x) \defn \frac{1}{\p(x)} \nabla_x \cdot (\p K) \defn
 \frac{1}{\p(x)} \sum_i \partial_{x^i}\left(\p K(\cdot,x)e_i \right)
 \end{talign}
 is a feature map for $\ks$.
 Indeed by \cref{thm:Differential reproducing property},  and using ~\citep[Prop.~1]{carmeli2010vector} to relate the inner products of $\Kb$ and $\p \Kb \p$,
 \begin{talign}
 \langevin (h)(x)&=\sum_i\frac{1}{\p(x)}
 \partial_{x^i}(\p h^i)  =
 \sum_i\frac{1}{\p(x)}  \metric{\partial_{x^i}(\p(\cdot) K(\cdot,x)\p(x)e_i)}{\p h}_{\p \Kb \p} \\
 &=
 \sum_i\frac{1}{\p(x)}  \metric{\partial_{x^i}( K(\cdot,x)\p(x)e_i)}{ h}_{ \Kb }
 =   \metric{ \sum_i \frac{1}{\p(x)} \partial_{x^i}( K(\cdot,x)\p e_i)}{ h}_{ \Kb }.
 \end{talign}
 This shows $\langevin$ is the feature operator associated to $\xi_\p$, and thus
 $\H_{\ks} \defn \langevin(\H_K)$ is a RKHS with kernel $ \ks(x,y) \defn \metric{\xi_\p(x)}{\xi_\p(y)}_K$~\citep[Prop.~1]{carmeli2010vector}.

The following lemma \cref{thm:feature operators preserve unit ball} concludes.

 \begin{lemma}[Feature operators preserve unit balls]
 \label{thm:feature operators preserve unit ball}
 Suppose $\tilde \xi: \H_K \to \H_{\overline K}$ is  a feature operator.
 Then $\tilde \xi( \B_K)=
 \B_{\overline K}$.
 \end{lemma}
 \begin{proof}
 Recall  $\tilde \xi$ is a surjective partial isometry~\citep[Prop.~1]{carmeli2010vector},
 in particular $\| \xi(h) \|_{\overline K} \leq \| h \|_K$,
 so $\tilde \xi( \B_K) \subseteq \B_{\overline K}$.
 Moreover, since
 $\tilde \xi |_{A} $ is an isometric isomorphism from the orthogonal complement of its kernel $A \defn \ker \tilde \xi^\perp$ onto $\H_{\overline K}$, it follows that for any $g \in \B_{\overline K}$ there exists $g \in A $ s.t.,
 $1\geq \| g \|_{\overline K} = \| h \|_K$,
 which concludes.

 \end{proof}

\subsection{Proof of \ncref{thm:embeddability_conditions}}
\phantomsection
\label{app:embeddability_conditions}

The result is an immediate consequence of  \cref{bound on integral root kp} and the following proposition.

\begin{proposition}[Stein RKHS with  vanishing $\P$-expectations]
\label{thm:Stein RKHS with vanishing P-expectations}
Suppose $\H_{\ks} $ and $\H_K \subseteq \C{1}{}$ are subsets of $\L^1(\P)$. Then $\P h =0$ for all $h \in \H_{\ks}$.
\end{proposition}
\begin{proof}
   The result follows from \citet[Thm.~2.36]{pigola2014global}, after observing that the distributional and usual derivatives of $\C{1}{}$ functions coincide.
\end{proof}

\section{Proof of \ncref{thm:tight_convergence}}
\phantomsection
\label{app:proof tight convergence}

In the proof we will use the fact that if we define the tilted
reproducing kernel
$\tilde{\k}( x,
 y) \defn
\nicefrac{\k( x, y)}{(1+\sqrt{\k}( x))(1+\sqrt{\k}( y))}$,
whose  RKHS is $\nicefrac{\H_k}{1+\sk}$,
we have the following immediate relation between the MMDs of $\k$ and $\tilde{\k}$, which, for instance, may be used to generalize some results from bounded to unbounded kernels:
\begin{proposition}[Kernel tilting]\label{remark:bounded to unbounded}
With the notation above
    $$ \mmd_{\k}(\Q_n, \P)  =\mmd_{\tilde \k}(\tilde \Q_n, \tilde \P), $$
where $\tilde \Q \defn (1+\sqrt k) \Q$ for any $\Q \in \Pset_{\sk}$,
and
 $$\Q_n f \to
    \P f \text{ for any } f \in \C{}{\sk} \,\ \iff \,\  \tilde{\Q}_n g
    \to \tilde{\P} g \text{ for any } g \in \C{}{b}.$$
\end{proposition}
The rationale for the above proposition is that the map
 $f \mapsto (1+\sqrt{k})f$ is a vector space isomorphism from $\C{}{b}$ (resp. $\C{}{0}$) to $\C{}{\sk}$ (resp. $\C{}{0,\, \sk}$),
 which induces TVS and isometric isomorphisms once appropriate topologies have been introduced.
 In that case the map
 $\Q \mapsto \tilde \Q$ identifies $\C{*}{\sk}$ (resp. $ \C{*}{0,\,\sk}$) with
  $\C{*}{b}$ (resp. $ \C{*}{0}$).

Coming back to the proof of \cref{thm:tight_convergence}, note
the kernel $\k$ needs to be characteristic to $\P \in \Pset_{\sk}$, since if it
    was not, there would be a measure $\Q \in \Pset_{\sk}$ not equal to $\P$ such that $\normK{\Q -
    \P} = 0$, hence $(\Q_n) \defn (\Q)$ would satisfy (a), and (b) since every distribution is tight on a Radon space, while $(c)$ holds since $(1+\sqrt{k})\Q$ is a finite measure; yet $(\Q_n)$ does not converge
    weakly to $\P$ in $\Pset_{\sk}$, since it does not converge weakly in $\Pset$ as a result of the fact $\C{}{b}$ is a separating set on any metrisable space.

    Conversely, let us assume  that $\k$ is characteristic to $\P \in
    \Pset_{\sk}$.  We will show that, given (c), (a)-(b) is equivalent to
    (usual!) weak convergence in $\Pset$. So, applying Lemma~5.1.7 of
    \citet{ambrosio05gradient} to every $f \in \C{}{\sk}$, gives the equivalence
    in \cref{eq:tight-convergence} and concludes. Intuitively speaking, (c)
    lifts weak convergence in $\Pset$ to weak convergence in $\Pset_{\sk}$.

    Assume (a)-(c). By (b) any subsequence of $(\Q_n)$ is tight, so, by
    Prokhorov's theorem \citet[Thm.~5.1.3]{ambrosio05gradient}, it is relatively compact in $\Pset$ (equipped with the weak topology) and thus contains yet another subsequence $(\P_l)$ that
    converges weakly in $\Pset$ to some probability distribution $\P'$.
     Since $1+\sqrt{k}$ is continuous, using condition (c) and (5.1.23b) in \citet[][Lem.~5.1.7]{ambrosio05gradient}
    further implies that $\P' \in \Pset_{\sk}$.
    Moreover, by (a) and continuity of the inner product,
    $\ipdK{\P'}{f} = \lim_l \ipdK{\P_l}{f} = \ipdK{\P}{f}$ for any $f \in \HK$.
    So, by the Pettis property, the embeddings of $\P'$ and $\P$ coincide:
    $\embK(\P') = \embK(\P)$.  Since $\k$ is characteristic to $\P \in
    \Pset_{\sk}$, we get $\P'=\P$. So we have shown that, out of any subsequence
    of $(\Q_n)$, we can extract a (sub)subsequence that converges weakly to $P$
    in $\Pset$.  By a classical argument, the original sequence $(\Q_n)_n$ thus
    converges weakly to $\P$ in $\Pset$.

    For the converse we essentially rely on \cref{remark:bounded to unbounded}.
    Note that if we assume that $\Q_n \to \P$ in $\Pset$,
    then,
     by Lemma~5.1.7 of \citet{ambrosio05gradient},  (c) is equivalent to weak convergence in $\Pset_{\sk}$.
     Define the measures
    $\tilde{\P}$ and $\tilde{\Q}_n$ as $\tilde{\P}(A) \defn \int_{A} (1 + \sqrt{\k}) \diff
    \P$ and $\tilde{\Q}_n(A) \defn \int_{A} (1 + \sqrt{\k}) \diff
    \Q_n$ for any measurable Borel set $A \subseteq \X$, and let $\tilde{\k}( x,
     y) \defn
    \nicefrac{\k( x, y)}{(1+\sqrt{\k}( x))(1+\sqrt{\k}( y))}$. By \citet{lecam1957convergence} since $\mathcal X$ is Radon,
     weak convergence in $\Pset$ implies tightness, i.e.\
    (b).
    Since $\Q_n(f) \to
    \P(f)$ for any $f \in \C{}{\sk}$  can be re-written as $\tilde{\Q}_n(g)
    \to \tilde{\P}(g)$ for any $g \in \C{}{b}$, it follows that $(\tilde{\Q}_n)$ converges
    weakly to $\tilde{\P}$ (in the usual sense). Moreover, (c) also shows that
    $\tilde{\Q}_n$ and $\tilde{\P}$ are finite (non-negative) measures. So we can
    apply Prop.~2.3.3 of \citet{berg84harmonic}, which says that the tensor
    product of finite, non-negative measures is weakly continuous, and get
    \begin{align}
        \mmd_{\k}(\Q_n, \P)^2
            &= (\Q_n - \P) \otimes (\Q_n - \P) (\k)
            = (\tilde{\Q}_n - \tilde{\P} ) \otimes (\tilde{\Q}_n - \tilde{\P})
                (\tilde{\k}) \\
            &= \tilde{\P} \otimes \tilde{\P}(\tilde{\k})
                - 2 \tilde{\Q}_n \otimes \tilde{\P} (\tilde{\k})
                + \tilde{\Q}_n \otimes \tilde{\Q}_n (\tilde{\k})
            \quad \longrightarrow \quad 0 .
    \end{align}

\subsection{Weak convergence in \tpdf{$\Pset_{\sk}$}{P-sqrt\{k\}} and Wasserstein metric}
\phantomsection
\label{sec:wasserstein convergence}

The following proposition first gives an alternative characterization of
$\C{}{\sk}$, $\Pset_{\sk}$ and weak convergence in $\Pset_{\sk}$. See also \citep[Section 3.1]{kanagawa2022controlling}.

\begin{proposition}[Wasserstein vs $\C{}{\sk}$ convergence]
    Let $\k$ be a continuous strictly positive definite kernel over a separable metric space $\X$. Let
    $d_{\k}( x, y) \defn \normK{\delta_{ x} - \delta_{ y}}$ be the metric
    induced by $\k$ over $\X$. Then $\C{}{\sk}$ is the set of functions with
    1-growth and $\Pset_{\sk}$ the probability measures with finite first-order
    moments, both w.r.t.\ the metric $d_{\k}$.

    Let $W_{d_{\k}}^1$ denote the Wasserstein-1 distance over $\Pset_{\sk}$
    w.r.t.\ $d_{\k}$. Then $(\X, d_{\k})$ is a separable metric space, and $W_{d_{\k}}^1$
    metrizes weak convergence in $\Pset_{\sk}$, i.e., for $\P_n,\P \in \Pset_{\sk}$
    \begin{align}
        \P_n h \to \P h \quad \forall h \in \C{}{\sk}
        \qquad \Longleftrightarrow \qquad
        W_{d_{\k}}^1(\P_n, \P)
    \to 0.
    \end{align}
    Moreover $(\Pset_{\sk}, W_{d_{\k}}^1)$ is complete whenever $(\X,
    d_{\k})$ is.
\end{proposition}

Note that $(\X, d_{\k})$ is complete whenever $d_{\k}$ is stronger than the
original metric $d$ (i.e.\ whenever there exists $C > 0$ such that $d( x,  y)
\leq C d_{\k}( x, y)$), which is for example the case when $\X = \R^d$
equipped with its usual Euclidian metric, and $\k$ is polynomial kernel of
order $\geq 1$.

\begin{proof}
    We will prove that there exists constants $C,C' > 0$ such that
    \begin{align}\label{eq:dk-inequality}
        C(1+\sqrt{\k}( x, x))
            \leq 1 + d_{\k}( x, x_0) \leq
            C'(1+\sqrt{\k}( x, x))
    \end{align}
    for all $ x, x_0 \in \X$. This shows that $\C{}{\sk}$ is indeed the set of
    functions with 1-growth for the metric $d_{\k}$ in the sense of (5.1.21) in
    \citet{ambrosio05gradient}; and that $\Pset_{\sk}$ is indeed the set of
    probability measures $\P$ with finite first-order moments, i.e.\ such that for
    an arbitrary (and then any) $ x_0 \in \X$, $\P(d_{\k}(.,  x_0)) < \infty$.

    The space $(\X, d_{\k})$ is separable whenever $\X$ is (independently of
    characteristicness), because, since $\k$ is continuous, $d_{\k}$ is also
    continuous, so the topology defined by $d_{\k}$ is weaker than the original
    one. Finally, when $\k$ is characteristic over $\Pset_{\sk}$, then $d_{\k}$ becomes a metric (i.e.\ additionally satisfies $d_{\k}( x, y) =
    0$ iff $ x =  y$). So Theorem~7.1.5 of \citet{ambrosio05gradient}
    concludes on the completeness condition, and on the equivalence between
    weak convergence in $\Pset_{\sk}$ and Wasserstein-1 convergence
    $W_{d_{\k}}^1$.

    We know prove \cref{eq:dk-inequality}. Let $\k_0 \defn \k( x_0, x_0)$ and
    $\k_{ x} \defn \k( x, x)$.  First, notice that
    \begin{align}
        d_{\k}( x, x_0)^2
            = \k_0 - 2\k( x, x_0) + \k_{ x}
            \leq \k_0 + 2 \sqrt{\k}_0\sqrt{\k}_{ x}
               + \k_{ x}
            = (\sqrt{\k}_0 + \sqrt{\k}_{ x})^2
    \end{align}
    Therefore
    $  %
        1 + d_{\k}( x,  x_0)
            \leq 1 + | \sqrt{\k}_0 + \sqrt{\k}_{ x} |
            \leq 1 + \sqrt{\k}_0 + \sqrt{\k}_{ x}
            \leq C'(1 + \sqrt{\k}_{ x})
    $  %
    with $C' \defn 1 + \sqrt{\k}_0$.  Conversely and similarly,
    $  %
        d_{\k}( x, x_0)^2
            = (\sqrt{\k}( x_0, x_0) - \sqrt{\k}( x, x))^2.
    $  %
    Therefore
    \begin{align}
        1 + d_{\k}( x,  x_0)
            \geq 1 + | \sqrt{\k}_{ x} - \sqrt{\k}_0 |
            \geq 1 + \max(\sqrt{\k}_{ x} - \sqrt{\k}_0, 0)
            \geq C (1 + \sqrt{\k}_{ x})
    \end{align}
    with $C = 1 / (1 + \sqrt{\k}_0)$.
\end{proof}

\section{Proof of \ncref{coroKSDwc}}
\label{app: proof coroKSDwc}
When  $\| \sp \|$ is finitely integrable with respect to a probability measure $\Q$, then both
$\sp^i \Q$ and $\Q$ are finite measures, so
\begin{talign}
D_\Q = \sum_i (\sp^i \Q - \partial_{x^i}\Q)e^i \defn \sum_i D_i e^i
\end{talign}
 where $D_i \in \DL{1}$ and thus $D_\Q \in \DL{1}(\R^d)$.

Using \cref{ksdAlt}
\begin{talign}
\ksd{\Q} = \| \int \xi_\P  \Q \, \|_{\H_K} =  \| D_Q\|_{\H_{K}}.
\end{talign}
Hence
$\ksd{\Q} =0$ iff $D_\Q =0$.
 Now since the   matrix kernel $K$  is characteristic to $\DL{1}(\R^d)$,
 and $D_\P =0$ by the divergence theorem, we finally obtain
 $\ksd{\Q} =0$ iff $D_\Q=0$  iff $\Q =\P$.

\section{Proof of \ncref{characteristic transformed RKHS}}

\begin{theorem}[Characteristicness of transformed kernel]\label{characteristic transformed RKHS}
    Let $\phi :  \mathscr{E} \to  \mathscr{F}$ be a linear continuous map that restricts to a feature operator $\overline{\phi}:\H_K \to \H_{\overline K}$.     Suppose the following diagram commutes, where all maps are continuous
    $$
\begin{tikzcd}
\H_K \arrow{r}{ \iota_K} \arrow[swap]{d}{\overline \phi} & \mathscr{E} \arrow{d}{\phi} \\
\H_{\overline K} \arrow{r}{\iota_{\overline{K}}} & \mathscr{F}
\end{tikzcd}
$$
If $\phi(\mathscr{E})$ is dense in $\mathscr{F}$,  then
$\overline K$ is characteristic to $\mathscr{F}^*$ when
$K$  is characteristic to $\mathscr{E}^*$.
\end{theorem}
\begin{proof}
\label{app: characteristic transformed RKHS}
Taking the  transpose of the commutative diagram
$\phi \circ \iota_K = \iota_{\overline K } \circ \overline \phi: \H_K \to \F $ yields
$$ \iota^*_K \circ \phi^* = \overline \phi ^* \circ \iota_{\overline K}^* : \mathscr{F}^* \to \H_K^*. $$
We want to show $\iota_{\overline K}^*$ injective given that $\iota_{K}^*$ is injective.
Note that $\overline \phi ^* \circ \iota_{\overline K}^*$ injective implies $\iota_{\overline K}^*$ injective, so it is sufficient to show
    $ \iota^*_K \circ \phi^*$ injective.
    Hence it is sufficient to show
    $\phi^* : \mathscr{F}^* \to \mathscr{E}^*$ injective, which is equivalent to $\phi(\mathscr E)$ dense in $\mathscr{F}$ by \citet[Chap 18 Cor.~5]{treves67}.
\end{proof}

Concretely,
it will usually suffice to verify that the image $\phi(\mathscr{E})$
contains the smooth compactly supported functions,
since these typically form a dense subset of $\mathscr{F}$.

\subsection{Proof of \ncref{thm:composition-kernel-properties}}
\label{app:composition-kernel-properties}

The result follows from our general theorem on characteristicness-preserving transformations \cref{characteristic transformed RKHS}.

For the first claim, note that $\phi:f \mapsto \tilt f$ is a continuous map from $\C{1}{b}(\R^d)_\beta$ to itself.
Moreover $\phi(\C{1}{c}(\R^d)) = \C{1}{c}(\R^d)$
since $f/\tilt \in \C{1}{c}(\R^d)$ for all $f \in \C{1}{c}(\R^d)$, and $\C{1}{c}(\R^d)$ is dense in the predual  $\C{1}{b}(\R^d)_\beta$ of  $\DL{1}(\R^d)$.

Recall the family of semi-norms defining the
$\C{1}{b}(\R^d)_\beta$ are parametrized by $\gamma \in \C{}{0}$ and have the form
$f \mapsto \| \gamma f \|_\infty$,
$f \mapsto \| \gamma \partial f \|_\infty$.
We first show that
$\phi: f \mapsto f \circ b$ is continuous
from $\C{1}{b}(\R^d)_\beta$ to itself.
Note that
if $\gamma \in \C{}{0}$
then $\gamma \circ b^{-1} \in \C{}{0}$,
since for any $\epsilon >0$ we have
$\overline{b \circ \gamma^{-1} \circ \| \cdot \|^{-1}[\epsilon,\infty)}=
b \circ  \overline{\gamma^{-1} \circ \| \cdot \|^{-1}[\epsilon,\infty)}$
because $b$ is a homeomorhism,
and the resulting set is compact since $\overline{\gamma^{-1} \circ \| \cdot \|^{-1}[\epsilon,\infty)}$ is compact (because $\gamma $ is $\C{}{0}$) and $b$ preserves compactness by continuity.
Thus
$$\| \gamma f \circ b \|_\infty
= \| \gamma \circ b^{-1} \circ b \, f \circ b \|_\infty =
\| \gamma \circ b^{-1}  f  \|_\infty$$
which is a semi-norm in $\C{1}{b}(\R^d)_\beta$.
The same proof works for the family of semi-norms
$\| \gamma \partial(f \circ b) \|_\infty$ once we have observed that $\partial(f \circ b) = \partial f \circ b \cdot \partial b$
(where $\cdot$ denotes matrix multiplication).
Indeed
$$\| \gamma \partial(f \circ b) \|_\infty =
\| \gamma  \partial f \circ b \cdot \partial b \|_\infty
= \| \gamma \circ b^{-1}  \partial f  \cdot \partial b \circ b^{-1} \|_\infty
\leq  \|  \partial b  \|_\infty \| \gamma \circ b^{-1}  \partial f \|_\infty $$
since $\partial b$ is bounded as $b$ is Lipschitz.
Thus $\phi$ is continuous.
It remains to show that
$\phi(\C{1}{c}(\R^d)) = \C{1}{c}(\R^d)$, which follows from the fact that
$f \in   \C{1}{c}(\R^d)$ implies
$ f \circ b^{-1} \in \C{1}{c}(\R^d)$
since
$\mathrm{supp}( f \circ b^{-1}) = \overline{b \circ f^{-1}(\{ 0\}^c)} =
b \circ \overline{ f^{-1}(\{ 0\}^c)}$ which is compact.

Finally, (c) follows from the more general statement
\begin{proposition}[Scalar vs.\ vector characteristicness]
\label{thm:scalar Vs vector characteristicness}
    Consider a matrix-valued kernel $K$ with $\H_K \hookrightarrow \F^d$ for some topological vector space $\F$.
    Then $K$ is universal to $\F^d$ iff $K_{ii}$ is universal to $\F$ for all $i$.
\end{proposition}
\begin{proof}
Recall that $h \in \H_K$ iff $h^i \in \H_{K_{ii}}$ for all $i$.
 Note
   $(f^1,\ldots, f^d) \in  \F^d$ iff $f^i \in \F$ for any $i \in [d]$,
   and  $h^i_n \to f^i$ in $\F$ for all $i$  iff
 $(h^1_n,\ldots, h^d_n) \to (f^1,\ldots, f^d)$
 in $\F^d$.

\end{proof}

 \section{Proof of \ncref{thm:L2 characteristic kernels}}
 \phantomsection
\label{app:L2 characteristic kernels}

(a) Let $\hat \kti_j \dd x$ be the Bochner measure of $ \kti_j(x-y) \defn \k_j(x,y)$.
Then $\hat \kti_j \dd x$, has full support, i.e., $\supp \, \hat \kti_j \dd x = \R^d$, since this is equivalent to the characteristicness of $\kti_j$ \citep[][Thm.17]{simon18kde}.
Moreover $\hat \kti_j \in \L^2$ since $\kti_j \in \L^2$.
Since $\H_{\k_j} \subseteq \L^2$, then any measure of the form $f \dd x$, with $f \in \L^2$ embeds into $\H_{\k_j}$ by \cref{thm:measure embed iff RKHS integrable}.

Then, if $g \in \L^2(\R^d)$, using \citet[Appendix 4]{barp2019minimum}
\begin{talign}
    \| \Phi_K(g \dd x) \|_K^2 &= \sum_i \| \metric{e_i}{\Phi_K(g \dd x)} \|_{\k_i}^2
     =
     \sum_i \| \Phi_{\k_i}(g_i \dd x) \|_{\k_i}^2\\
     &
     = \sum_i \iint g_i(x) \kti_i(x-y)g_i(y) \dd x \dd y.
\end{talign}
Moreover, since  Plancherel theorem and the convolution theorem are valid for $\L^2$ functions \citep[Remarque pg. 270]{schwartzTD}, using the fact that
$g_j,\kti_j$ and $\kti \star g_j$ are in $\L^2$ by \citet[Prop.~4.4]{carmeli06vector}, then
\begin{talign}
 \iint g_i(x) \kti_i (x-y)g_i(y) \dd y \dd x
 &= \int g_i(x) \kti_i\star g_i(x) \dd x
 = \int \hat g_i(w) \hat \kti_i(w) \hat g_i(w) \dd w\\
 &
 =\int |\hat g_i(w)|^2 \hat \kti_i(w) \dd w.
\end{talign}
Hence, whenever $g \dd x$ is non-zero, i.e., $\| g \|_{\L^2}>0$, then
\begin{talign}
\| \Phi_K(g \dd x) \|_K^2 = \sum_i \int |\hat g_i(w)|^2 \hat \kti_i(w) \dd w >0
\end{talign}
since $\hat \kti_i(w) \dd w$ is a fully supported non-negative measure.

(b) 
This follows directly by the definition of ISPD,
together with the fact htat
$A\left(\L^2(\R^d)\right) \subseteq \L^2(\R^d)$ by 
boundedness,
and that if 
$ \|Ag \|_{\L^2(\R^d)} =0$
then $\| A g \| =0$ a.e. so 
$\|g \|=0$ a.e. and 
thus 
$ \|g \|_{\L^2(\R^d)} =0$.

(c) Let us first discuss the scalar (and matrix diagonal case), i.e.,
 we show that for a scalar reproducing kernel $\k$, if $\H_\k$ is separable,  $\sup_x \| \kb_x \|_{\L^1}<\infty$, and $\k_x \in \L^2$ for each $x$, then $\H_{\k \Id} \subseteq \L^2(\R^d)$.

Write $\k^*f(x) \defn (\k^*f)(x) \defn \int \k(x,y)f(y) \dd y$ when the integral is well-defined.
Note, for each $y$,
$\k^*k_y \in \L^1$,
since
\begin{talign}
\|  \k^*\k_y \|_{\L^1} &= \int | \k^*\k_y |(x) \dd x
=\int | \int \k(x,z)\k(z,y) \dd z | \dd x
=
\iint  | \k(x,z)\k(z,y)| \dd x  \dd z \\
&= \int | \k(z,y)| \| \k_z \|_{\L^1} \dd z
\leq  \| \k_y \|_{\L^1} \sup_z \| \k_z \|_{\L^1} < \infty.
\end{talign}
Note the integral swap is justified by Fubini's theorem.
Moreover,
$\sup_y\|  \k^*\k_y \|_{\L^1} \leq \sup_z \| \k_z \|_{\L^1}^2$.

Now, if $f \in \L^2$,
then
\begin{talign}
\iint |f(x)^2 \k^*\k_x(z)| \dd z \dd x = \int f(x)^2 \| \k^*\k_x \|_{\L^1} \dd x \leq \sup_x\| \k^*\k_x \|_{\L^1 } \| f \|_{\L^2}, 
\end{talign}
so by Fubini
$\iint |f(x)^2 \k^*\k_x(z)| \dd x \dd z<\infty$
and thus $\int |f(x)^2 \k^*\k_x(\cdot)| \dd x $ is finite a.e., i..e,
$\sqrt{|f^2 \k^*\k_z|} \in \L^2$ for almost all $z$.
It follows that
$z \mapsto \| \sqrt{|f^2 \k^* k_z| }\|_{\L^2} \in \L^2$,
since
$\int \| \sqrt{|f^2 \k^* k_z| }\|_{\L^2} \dd z
=\iint |f(x)^2 \k^*\k_x(z)| \dd z \dd x$,
and similarly
$z \mapsto |f(z)|\|\sqrt{|\k^* k_z| }\|_{\L^2} \in \L^2$.

Hence,
$\int \int |f(x) \k^* k_z(x) f(z)| \dd x \dd z < \infty$
since
$f \k^* k_z \in \L^1$ (being the product of the $\L^2$ functions
$\sqrt{|f^2 \k^* k_z| }$  and $\sqrt{|\k^* k_z| }$)
and
\begin{talign}
  \iint |f(x) \k^* k_z(x) f(z)| \dd x \dd z &\leq
\int \left( \| \sqrt{|f^2 \k^* k_z| }\|_{\L^2} |f(z)|\|\sqrt{|\k^* k_z| }\|_{\L^2} \right) \dd z  < \infty.
\end{talign}

Finally,
\begin{talign}
    \| \k^*f \|_{\L^2}^2 &=
    \int (\k^* f)^2(y) \dd y
    = \iint \k(y,x)f(x) \dd x \int \k(y,z) f(z) \dd z \dd y \\
    &=
     \iint f(x) f(z) \k^*\k_z(x) \dd x \dd z
     \leq
 \iint | f(x) f(z) \k^*\k_z(x) | \dd x \dd z< \infty.
 \end{talign}

It follows by \citep[Prop.~4.4]{carmeli06vector} that $\H_\k \subseteq \L^2$.

For a general matrix-valued kernel $\Kb$,
set
$G_{ij}(x,z) \defn \int \metric{e_i}{\Kb(x,y)\Kb(y,z)e_j} \dd y
= \sum_l \int \Kb_{il}(x,y)\Kb_{lj}(y,z) \dd y $.
Note $G_{ij}(x,z)= G_{ji}(z,x)$,
and set $G_{ij}^z \defn G_{ij}(\cdot,z)$.
Now for $f \in \L^2(\R^d)$,
if $\iint |f_i(x) f_j(z) G_{ij}^z(x)| \dd x \dd z < \infty$
we have
\begin{talign}
    \| K^* f \|^2_{\L^2(\R^d)} &=  \int  \| K^*f \|^2(y) \dd y
    = \sum_l \int |\metric{e_l}{ K^*f}|^2(y) \dd y
    =
    \sum_l \int \metric{e_l}{ K^*f}|(y) \metric{e_l}{ K^*f}|(y) \dd y\\
    &=
    \sum_{lij}\iiint K_{li}(y,x) f_i(x) K_{lj}(y,z)f_j(z) \dd x \dd z \dd y
    \\
    &=
    \sum_{lij}\iiint K_{il}(x,y) f_i(x) K_{lj}(y,z)f_j(z) \dd x \dd z \dd y
    \\
    &= \sum_{ij} \iint f_i(x) f_j(z) G_{ij}^z(x) \dd x \dd z.
\end{talign}
We can now proceed as in the scalar case with $G_{ij}^z$ taking the place of $\k^* \k_z$.
Indeed
$$
\| G_{ij}^z \|_{\L^1}  \leq  \| \Kb_{ij}^z \|_{\L^1} \sup_x \| \Kb_{ij}^x \|_{\L^1},
$$
and note
$| \Kb_{ij}^z | = | \metric{e_i}{\Kb_z e_j} \leq \| e_i \| \| \Kb_z e_j \| $
so
$\| \Kb_{ij}^z \|_{\L^1} \leq
 \| e_i \| \| \Kb_z e_j \|_{\L^1(\R^d)} $.

(d) Note that if $\Kb$ is a matrix-valued kernel,
then  $|\Kb_{ij}(x,y)| = |\metric{e_i}{\xi(x)^*\xi(y)e_j}|  \leq \| \xi(x) e_i \| \| \xi(y) e_j \| = \sqrt{\Kb_{ii}(x,x)} \sqrt{\Kb_{jj}(y,y)} $ (for any feature map $\xi$),
so if $\Kb$ is translation-invariant then
$\Kb$ is bounded.

In general, if $\Kb_x u$ is both finitely-integrable and bounded, then
$  \|\Kb_x u \|_{\L^2(\R^d)} \leq
\|\Kb_x u \|_{\L^1(\R^d)} \|\Kb_x u \|_{\L^\infty(\R^d)}$.

\section{Proof of \ncref{thm:KSD score Q formulation}}
\phantomsection
\label{app:KSD score Q formulation}

(a)
Let us first show that under the assumptions of the result,
the alternative definition of KSD used in
\citet{liu16kernelized} is equivalent to ours.

By \cref{Continuous linear functional shifted by Feature Operator}, since $\Q$ embeds we have $\ksd{\Q}= \| D_Q \|_K$.
Moreover since $\Q \in \embedtozero$,
we know $T_Q \defn \Q \circ \langevin[\q]$
embeds to zero in $\H_K$, $\Phi_K(T_\Q)=0$,
and for any $h\in \H_K$,
since $\langevin(h)$ and $\langevin[\q](h)$ are finitely $\Q$-integrable,
we have
$D_Q -T_Q (h) = \Q \langevin(h)- \Q \langevin[\q](h)= \Q(\langevin(h)-\langevin[\q](h))
= (\sp-\sq)\Q (h)$.
Hence
\begin{talign}
     \ksdsq{\Q} &= \| D_Q \|_K^2 =
 \| D_Q -T_Q\|_K^2
 =\| (\sp-\sq)\Q \|_K^2
 \\
 &= \iint \metric{\sp(y)-\sq(y)}{K(y,x)(\sp(x)-\sq(x))} \dd \Q(y) \dd \Q(x).
\end{talign}

Now, recall $\Kb$ is $\L^2(\R^d)$ ISPD iff it is characteristic to
$\L^2(\R^d)$ \citep[Thm.~6]{simon18kde}, so the result follows from
$\ksd{\Q}= \mmd_K((\sp-\sq)\Q,0)$,
and $(\sp-\sq)\Q \in \L^2(\R^d)$.

(b)
We want to show that $\H_\Kb$ and $\S_\q(\H_{\Kb})$ are subsets of $\L^1(\Q)$
in order to apply (g) in  \cref{thm:embeddability_conditions}.
By assumption $\H_\Kb \subseteq \L^\infty(\R^d)\subseteq \L^1(\Q)$, and
$\langevin(\H_{\Kb}) \subseteq \L^1(\Q)$.
Note $\S_\q(h) =\langevin(h) +\metric{\sq-\sp}{h}$,
so we have for any $h\in\Kb$
$$\Q | \S_\q(h)| = \Q |\langevin(h) +\metric{\sq-\sp}{h}| \leq
\Q |\langevin(h)| + \Q|\metric{\sq-\sp}{h}| .$$
Moreover
$\Q|\metric{\sq-\sp}{h}|
= \metric{\q(\sq-\sp)}{h}_{\L^1(\R^d)}
\leq \| \q(\sq -\sp) \|_{\L^2(\R^d)} \|h\|_{\L^2(\R^d)}$, which concludes.

(c)
For any $h\in \H_{\Kb}$
we have 
\begin{talign}
\Q| \langevin(h) | \leq
\sum_i \Q |\partial_i h^i| + \Q|\metric{\sp}{h}| \leq
\sum_i \|\partial_i h^i\|_\infty + \| \q \sp \|_{\L^2(\R^d)} \|h\|_{\L^2(\R^d)}<\infty.
\end{talign}

 \section{Fourier Transforms}
 \label{app: Fourier transforms}

If $\mu$ is a finite measure, its Fourier transform is
$\hat \mu(x) \defn \int e^{-ix^T w} \mathrm{d} \mu(w)$, which is a positive definite function when $\mu$ is a non-negative measure.
More generally, if $T$ is a tempered distribution (a.k.a.\ slowly increasing distribution, see \citealt[Chap.25]{treves67}), i.e., an element of the dual of the Schwartz space (a.k.a.\ space of rapidly decaying functions, see \citealt[Chap.10, Example IV]{treves67}),
we define its distributional Fourier transform $\hat T$ by
$$
    \hat T(\gamma) \defn T(\hat \gamma) \ ,
$$
for  any function $\gamma$ in the Schwartz space.
In particular if $T = \Phi(x) \mathrm{d} x $,
with $\Phi$ continuous and slowly increasing, and there exists  $f \in \L^2_{\mathrm{loc}}(\R^d/\{0\})$ such that $  \hat T = f \dd x$, then $f$ is known as the generalized Fourier transform (of order $0$) of $\Phi$, denoted $\hat \Phi$ \citep[Def. 8.9]{wendland04scattered}.
The above formula then reads
$ \int  \hat \Phi(x) \gamma(x) \dd x = \int \hat \gamma(x) \Phi(x) \dd x$.

\section{Proof of \ncref{thm:bounded separating is P-characteristic}}
\phantomsection
\label{app:bounded separating is P-characteristic}

Since $\k$ is continuous, $\H \subseteq \cts$, and since $h \in \H_b \subseteq \C{}{b} $ is integrable by any finite measure, we have (with $0/0=0$)
\begin{talign}
| \Q h - \P h | \leq
  \left| \Q \frac{ h}{ \| h \|_\k} - \P \frac{h}{ \| h \|_\k} \right| \, \| h \|_\k \leq \sup_{ h \in \B_\k \cap \L^1(\Q)} | \Q h-\P h | \| h \|_\k \leq  \mmd_\k(\Q,\P) \| h \|_\k.
\end{talign}
Hence if $\mmd_\k(\Q_n,\P) \to 0$ then
$  \Q_n h \to \P h $ for all $h \in \H_b$.
Taking $\Q_n = \Q$ for all $n$, the $\P$-bounded separating assumption implies that  $\k$ is characteristic to $\P\in \Pset$.
On the other hand, if $(\Q_n)$ is a tight sequence,
then it is sequentially compact, i.e., any subsequence has  a $\cts_b$-convergent subsequence, whose  weak limit must be $\P$ by $\P$-characteristicness \citep[Lemma 4.3]{ethier2009markov}, which in turn implies that $\Q_n \to_{\cts_b} \P$ (see also proof of \ncref{thm:tight_convergence} for additional details).
Hence $k$ controls tight convergence.

\section{Proof of \ncref{thm: tilted controls tightness}}
\label{app:proof tilted controls tightness}

By \cref{cor: B1a universality from transformation} the shifted kernel
$\Kb(x,y)/\growth(x)\growth(y) $ is universal  to (i.e., dense in) $\C{1}{b,\growth}(\R^d)$.
Moreover, by assumption on the score growth and base kernel the associated Stein RKHS 
 consists of bounded functions.
 The following lemma concludes:

\begin{lemma}[Controlling tight convergence with bounded Stein kernels]\label{thm:Decaying and universal base RKHS are characteristic}
Suppose a matrix-valued kernel $K$ with $\H_K \subseteq \C{1}{b}(\R^d)$ is characteristic to $\C{1}{b,\growth}(\R^d)_\beta^*$.
If $\P\in\embedtozero$ and $\ks$ is bounded, then $\ks$ is $\P$-separating and controls tight $\P$-convergence. %
Moreover, $\ks$ is bounded iff $x\mapsto\sqrt{\metric{\sp(x)}{K(x,x) \sp(x)}}$ is bounded.
\end{lemma}
\begin{proof}
We apply
\cref{thm:characteristic to P using C^1_a} with
$\growth(x) \defn \| \sp(x)\| +1$.
Note that if $\|\sp(x)\| \H_K$ is bounded, then
$\H_K \subseteq \C{1}{b,\growth}(\R^d)$.
Moreover
$\|\sp(x)\| \H_K = \H_{\tilde K}$
where $\tilde K(x,y) \defn \| \sp(x) \| K(x,y) \| \sp(y) \|$,
and $\H_{\tilde K}$ is a RKHS of bounded functions
iff
$x \mapsto \|\sp(x)\| \sqrt{\|  K(x,x) \|}$ is bounded by \cref{thm:Characterization of bounded RKHS}.

Moreover, since $\ks$ consists of bounded functions, it then follows by \cref{thm:bounded separating is P-characteristic}   that when $\k_\P$ is $\P$-separating then it controls tight convergence to $\P$.

\end{proof}

\section{Proof of \ncref{thm:Schwarz-tilted sub-RKHS}}
\label{app:Schwarz-tilted sub-RKHS}

We will use the following result proved in \cref{app:convolution bound}.

\begin{lemma}[Convolution decay bound]
\label{thm:convolution bound}
Fix any $u,v \in \L^1$ and any subadditive function $\rho$ satisfying 
$|u(x)| \leq U(\rho(x))$ and $|v(x)| \leq V(\rho(x))$ for non-increasing $U,V$   with $U\circ \rho,V\circ \rho$ finitely integrable.
Then
\begin{align}
\label{eq:convolution bound}
 u \star v \,(x) 
    \leq 
\inf_{\alpha\in[0,1]}
\| u \|_{\L^1} V(\alpha \rho(x))+
   \| v \|_{\L^1} U((1-\alpha) \rho(x))
\end{align}
\end{lemma}

Let us quote  Bochner's theorem~\citep[Thm.~6.6]{wendland04scattered} (we refer to \cref{app: Fourier transforms} for definitions of Fourier transforms).
\begin{theorem}[Bochner's theorem]\label{bochner}
 A continuous $\R$-valued function on $\R^d$ is positive definite if and only if it
is the Fourier transform of a non-negative finite measure.
\end{theorem}

Moreover we will use the following lemma, which follows
by
combining \cref{thm:Differential reproducing property} with the fact that vector-valued RKHSes of continuous functions have locally bounded kernels \citep[Prop.~5.1]{carmeli06vector},
or by
the closed graph theorem as shown \cref{app:continuity of RKHS inclusion}.
\begin{lemma}[Continuity of RKHS inclusion]
\label{thm:continuity of RKHS inclusion}
Let $\F$ be a complete metrizable TVS, continuously included in the space of functions $\X \to \R^d$.
Then
 $\H_K \subseteq \F$ implies $\H_K \hookrightarrow \F$.
In particular  $\C{s}{}(\R^d)$ is a complete metrizable TVS.
\end{lemma}

We now show $\kb$ is characteristic to $\DL{1}$.
Indeed, since $\HK \subseteq \C{1}{}$, then
$\HK \hookrightarrow \C{1}{}$,
so
we know $\partial_{i}\partial_{i+d}\k$ exists and is separately continuous for all $i$
by \citet[Thm.~2.11]{micheli2013matrix}.
Since $\partial_{i}\partial_{i+d}\k$ is translation-invariant, it is further continuous.
Thus
 $\kb(x,y) $ is $\C{(1,1)}{}$,
  and  is characteristic to $\DL{1}$ by \citet[Thm.~17]{simon18kde}.

Let $\kti(x) \defn  \kb(x,0)$.
Now, given the spectral density $\hat \kti:\R^d \to [0,\infty]$,
we   define the \emph{ironed} radial kernel $\kti_{\mathrm{iron}}$  on $\R^d$  by
$$
\hat \kti_{\mathrm{iron}}(y)  \defn \inf_{w \in \R^d: \, 0 \leq  \| w \| \leq \|y\|} \hat \kti(w) \,
$$
and show it is characteristic to $\DL{1}(\R^d)$.
Note $\hat \kti_{\mathrm{iron}} :\R^d \to [0,\infty]$,
and is finite-valued except at the origin (since $\hat \kti$ is).
Since $\hat \kti_{\mathrm{iron}} \leq \hat \kti$, $\hat \kti \in \L^1(\R^d)$,
$\infty >\int \hat \kti(w) \dd w \geq \int \hat \kti_{\mathrm{iron}}( w ) \dd w = \| \hat \kti_{\mathrm{iron}} \|_{\L^1(\R^d)} $,
 so $\hat \kti_{\mathrm{iron}}(w) \dd w$
is a finite non-negative measure whose Fourier transform defines a continuous positive definite (radial) kernel $\kti_{\mathrm{iron}}$ by \cref{bochner}.
Moreover $\hat \kti_{\mathrm{iron}}$ is strictly positive
since $\hat \kti$  is bounded away from zero on the compact set $\overline{\B_r}$ for any $r \geq 0$.
In particular $ \kti_{\mathrm{iron}}$ is characteristic to $\DL{0}$~\citep[Thm.~17]{simon18kde}.

Moreover, by \citet[Sec.~4.3]{steinwart08support}
$\partial_{x^i} \partial_{y^i}\kb(x,y)$ is a continuous translation-invariant kernel, and
$\partial_{x^i} \partial_{y^i}\kb (x,y) = - \partial^2_i \kti (x-y)$.
This implies that $\int \| w \|^2 \hat \kti (w) \dd w <\infty$.
Indeed $\widehat{( - \partial_i^2 \kti)}(w) = w_i^2 \hat \kti(w)$ \citep[Thm.~25.7]{treves67}, and by \cref{bochner} the generalized Fourier transform of a continuous translation-invariant kernel is integrable, i.e., $w \mapsto w_i^2 \hat \kti(w) \in \L^1(\R^d)$.
From $\kti \geq \kti_{\mathrm{iron}}$, it follows that
$w \mapsto \| w \|^2 \hat \kti_{\mathrm{iron}}(w) \in \L^1(\R^d)$,
and thus
$\kti_{\mathrm{iron}} \in \C{2}{}$, since by Leibniz integral rule the second partial derivatives exist and are continuous.
Hence $\kb_i$ is characteristic to $\DL{1}$ by~\citet[Thm.~17]{simon18kde}.

Now define a radial kernel $\kti_s$ on $\R^d$ by $\hat \kti_s(x) = \hat \kti_{\mathrm{iron}}(2x)$ which is strictly positive so also characteristic to $\DL{0}$ (more generally we can compose $\hat \kti_{\mathrm{iron}}$ with any homeomorphism, as their preimage commutes with closure), and
$\kti_s \in \C{2}{}$ so it
characteristic to $\DL{1}$ (by an argument analogous to that of the previous paragraph).

We now discuss a general mechanism to construct a Schwartz  function $f$ that is strictly  positive,   and has a non-negative Fourier transform with compact support (which even makes $f$ an \emph{entire} function). Later on we will apply this construction to obtain a particular $f$ with root exponential decay.
Let us first choose a function $g \in C^\infty$ that is non-negative and compactly supported (we will identify an explicit choice of such a function later in the proof).
Then we set $f \defn \hat G \star \hat G$ where $G \defn g \star g$. Note that $G$ is non-negative, smooth and compactly supported with non-negative Fourier transform since by the convolution theorem $\hat G = (\hat g)^2 $.
Moreover the
Schwartz's Paley–Wiener theorem \citep[Thm.~29.1]{treves67} then asserts that its Fourier transform (more precisely, its real part restricted to real inputs, i.e., complex numbers with zero imaginary part) $\hat G$ is an indefinitely (real) differentiable function that decays faster than any polynomial, that is for all positive integer $m$ we have constant $C_m>0$ such that
$|\hat G(x)| \leq \frac{C_m}{(1+ \| x \|)^m}$.
Hence so does $f$ by \cref{thm:convolution bound} applied with $U,V$ of the form $r \mapsto (1+r)^{-m}$.
Moreover $\hat f(x) = (G(-x))^2 $ which is non-negative  with compact support, and $f$ is strictly positive since $\hat G$ (viewed as function of arbitrary complex variables) is entire by the Schwartz's Paley–Wiener theorem, and thus  holomorphic, and thus has finitely many isolated zeros, the set of which has Lebesgue-measure zero - hence $f$ is the integral of an almost-everywhere strictly positive function.
Moreover, the derivative of $f$ decays faster than any polynomial.
Indeed, since $f = \hat G \star \hat G$, where $\hat G$ is a smooth function decaying faster than any polynomial, Leibniz' integral rule yields
 $\partial_{x^j}  f = (\partial_{x^j} \hat G) \star \hat G $.
 By \citet[Thm.~5.16 (6)]{wendland04scattered},
 $\partial_{x^j} \hat G(x)= \widehat{(-iy^jG(y))}(x) $, and since $y\mapsto -iy^jG(y)$ is smooth and compactly supported, Schwartz's Paley–Wiener theorem implies that its Fourier transform will decay faster than any polynomial.
 The convolution bound \cref{thm:convolution bound} then implies that
 $\partial_{x^j} \hat G \star \hat G$ will also decay faster than any polynomial.  
 The above argument can then be iterated to show that $f$ belongs to the Schwartz space.

If in the above we specifically choose the function $g = \psi \star \psi$ with
$\psi(x) \defn \varphi(x^1)\cdots \varphi(x^d)$
where
$\varphi(x^1) \defn \exp(- 1/(1- c^2|x^1|^2)) \indic{|x^1| < 1}$ for some $c>0$,
then $\hat \psi(x) = \hat \varphi(x^1)\cdots \hat \varphi(x^d)$,
which implies $|\hat \psi (x)| = O(e^{- c\sum_i \sqrt{|x^i|}})$
by
\citet{johnson2015saddle}.
It follows that 
 $|\hat g (x)| = O(e^{- 2c\sum_i \sqrt{|x^i|}})$, 
 and thus
$\hat G = (\hat g)^2=  O(e^{- 4c\sum_i \sqrt{|x^i|}})$, so 
$f = \hat G \star \hat G =O(e^{-2c\sum_i \sqrt{|x^i|}})$ by the convolution bound \cref{thm:convolution bound}
with $\alpha=\nicefrac{1}{2}$ and the subadditive function $\rho(x) = \sum_i \sqrt{|x^i|}$.
Similarly,
$|\partial_{x^j} \hat G| = 2 |\hat g \partial_{x^j} \hat g|  \leq 2\| \partial_{x^j} \hat g  \|_\infty | \hat g |$,
and $\partial_{x^j} \hat g$ is bounded by the Schwartz's Paley-Wiener theorem as above,
so $\partial_{x^j} \hat G =  O(e^{- 2c\sum_i \sqrt{|x^i|}})$,
so $\partial_{x^j} f = (\partial_{x^j} \hat G) \star \hat G =  O(e^{- c\sum_i \sqrt{|x^i|}})$ by the convolution bound \cref{thm:convolution bound}
with $\alpha=\nicefrac{1}{2}$ and the subadditive function $\rho(x) = \sum_i \sqrt{|x^i|}$.

Now define $\kb_f(x,y) \defn f(x) \kb_s(x,y) f(y)$,
and we will show that $\H_{\kb_f} \subseteq \H_{\kb}$, by leveraging the translation-invariance of the kernel
$\kb_{fs}(x,y) \defn \kti_s(x-y) f(x-y)$, which is a kernel since $f$ is a positive definite function (since its  Fourier transform is a positive function) and thus defines a reproducing kernel.
Then $\H_{\kb_f} \subseteq \H_{\kb_{fs}}$.
Indeed,
the former RKHS is simply the set of functions $f \H_{\kb_s}$~\citep[Prop.~6.2]{paulsen2016introduction}.

On the other hand, \citet[Thm.~II Sec.~8]{aronszajn50trk} implies the latter product RKHS $\H_{\kb_{fs}} = \H_{\kb_s} \times \H_{f}$ consists of the functions in the tensor product RKHS  $\H_{\kb_s} \otimes \H_{f}$
restricted to the diagonal set $\{(x,x)\} \subseteq \R^d \times \R^d$, while \citet[Thm.~13]{berlinet04reproducing} shows the tensor product RKHS contains the functions of the form $(x,y) \mapsto g(x)h(y)$ where $g \in \H_{\kb_s}$ and
$h \in \H_{f}$ (i.e., it is the pullback RKHS defined by the diagonal map $x \mapsto (x,x)$), and hence
$\H_{\kb_{fs}}$ contains all the functions of the form
$x \mapsto g(x)h(x)$. Restricting to $h(x) = f(x-0) \in \H_{f}$ yields the subset inclusion
$\H_{\kb_f} \subseteq \H_{\kb_{fs}}$, which is moreover a continuous inclusion
$\H_{\kb_f} \hookrightarrow \H_{\kb_{fs}}$ because inclusions of RKHS are always continuous \citep[Prop.~2]{schwartz1964sous}.

Hence to show that $\H_{\kb_f}$ is a subset of $\H_\kb$, it is sufficient to show that $\H_{\kb_{fs}}\subseteq \H_\kb$.
But, conveniently, since $\kb_{fs}$ is translation invariant,
we can now
apply \cref{thm:shifted translation invariant RKHS contained}, proved in \cref{app:shifted translation invariant RKHS contained}, to verify this inclusion.

\begin{lemma}[RKHS inclusion of product RKHS] \label{thm:shifted translation invariant RKHS contained}
Let $\k,\k_2$ be kernels and   $\star$ denote the convolution operator.
The following claims hold true.
\begin{enumerate}[label=\textup{(\alph*)}]
    \item  If there exists a $\lambda \geq 0$ for which $\lambda k- kk_2$ is a kernel, then $h\H_\k \subseteq \H_\k$ for any $h \in \H_{\k_2}$.

    \item Suppose $\k,\k_2$ are continuous translation invariant kernels. By Bochner's theorem (\cref{bochner}), such kernels are the Fourier transform of some finite positive measures which we will call $\mu$ and $\nu$.
    If  $\mu \star \nu \ll \mu$ and the density  $\frac{\mathrm{d}\mu \star \nu}{\mathrm{d} \mu}$ belongs to $\mathrm{L}^{\infty}(\mu)$, then  $h\H_\k \subseteq \H_\k$ for any $h \in \H_{\k_2}$.

    \item Moreover, if $\mu$ (resp. $\nu$)  above is equivalent to (resp. absolutely continuous with respect to) the Lebesgue measure on $\R^d$, with density $q_\mu$ (resp. $q_\nu$), then
    $h \H_{\k } \subseteq \H_{\k}$ for any $h\in \H_{k_2}$ if $q_\mu \star q_\nu / q_\mu \in \L^\infty(\R^d)$.

    \item  Similarly, if $f:\R^d \to \R$ is a continuous positive definite function with  generalized Fourier transform $\hat f$,
    and $q_\mu \star \hat f  / q_\mu \in \L^\infty(\R^d)$,
    then
    $h\H_\k \subseteq \H_\k$ for any $h \in \H_{\k_2}$, where $\k_2(x,y)\defn f(x-y)$.
\end{enumerate}
\end{lemma}

Since $\hat \kb$ is strictly positive, the result states
that $\H_{\kb_{fs}}\subseteq \H_\kb$ iff
$\hat{\kti}_{fs}/ \hat{\kti} \in \L^\infty$,
which, by the convolution theorem, can be written as
$\hat \kti_s \star \hat f / \hat \kti \in \L^\infty$.
To show this we will use \cref{thm:convolution bound} to obtain the convolution bound
$$
|( \widehat{\kti_s} \star \hat f )(w)| \leq \| \widehat{\kti}_s  \|_{\L^1} U(\thalf \| w \| ) + \| \hat f \|_{\L^1} V(\thalf \| w \| ) \ ,
$$
where $U$ and $V$ are non-increasing functions that upper bound
$\hat f$ and $\hat \kti_s$ respectively.
We let $U$ be the envelope above $f$,  $U(r) \defn \sup \{ \hat f(w) : \| w \| \geq r \} $,  and since by construction $\hat \kti_s$ is non-increasing, we can set
$V(r) \defn \hat \kti_s(r)$.
Thus
$$
|( \widehat{\kti_s} \star \hat f )/\hat \kti (w)| \leq \| \widehat{\kti}_s  \|_{\L^1} U(\thalf \| w \| )/\hat \kti(w) + \| \hat f \|_{\L^1} \hat \kti_s(\thalf \| w \| )/\hat \kti(w)\ .
$$
By
construction, for any $w \in \R^d$ we have $\hat \kti_s(\frac1 2  w  ) =
\hat \kti_{\mathrm{iron}}(  w  ) \leq \hat \kti(w)$, and thus
$ \hat \kti_s(\frac1 2 \cdot  )/\hat \kti(\cdot) \in \L^\infty$.
Moreover $\hat f$ has compact support, and thus so does $U$,
hence $|( \widehat{\kti_s} \star \hat f )/\hat \kti |  \in \L^\infty$ and $\H_{\kb_{fs}} \subseteq \H_\kb$ as claimed.
Moreover $\H_{\kb_f} \hookrightarrow \H_{\kb}$ \citep[Prop.~2]{schwartz1964sous}.

\subsection{Proof of \ncref{thm:convolution bound}}
\phantomsection
\label{app:convolution bound}
Fix any $\alpha \in [0,1]$ and let $S_x \defn \{ y\in \R^d: \rho(x-y) \leq \alpha \rho(x) \}$.
On this set $\rho(y) \geq (1-\alpha) \rho(x)$ by subadditivity.
Now
\begin{talign}
u \star v (x) = \int u(y) v(x-y) \dd y =
 \int_{S_x} u(y) v(x-y) \dd y +  \int_{S_x^c} u(y) v(x-y) \dd y.
\end{talign}
Moreover,
\begin{talign}
\int_{S_x} |u(y) v(x-y)| \dd y \leq
\int_{S_x} U(\rho(y)) |v(x-y)| \dd y
\leq \int_{S_x} U((1-\alpha) \rho(x)) |v(x-y)| \dd y
\leq U((1-\alpha) \rho(x)) \| v \|_{\L^1}.
\end{talign}
On the other hand
\begin{talign}
\int_{S_x^c} |u(y) v(x-y)| \dd y
\leq
\int_{S_x^c} |u(y)| V(\rho(x-y)) \dd y
\leq
\int_{S_x^c} |u(y)| V(\alpha \rho(x)) \dd y
\leq V(\alpha \rho(x)) \| u \|_{\L^1}.
\end{talign}

\subsection{Proof of \ncref{thm:continuity of RKHS inclusion}}
\phantomsection
\label{app:continuity of RKHS inclusion}

Consider a convergent sequence
$(h_n,h_n) \to (h,f)$ in $\H_K \times \F$.
Since $\H_K$ and $\F$ are continuously included in the space of functions $\X \to \R^d$, $(h_n)$ converges pointwise to both $h$ and $f$, hence $h=f \in \H_K$, and the graph of $\iota: \H_K \to \F$ is closed.
Thus \citep[Cor.~4, Chap.~17]{treves67} implies it is continuous, since the product of metrizable (resp. complete) TVS is metrizable (resp. complete).

In particular $\C{s}{}(\R^d)$ is a complete metrizable space by \citep[Ex. 1 Chap.~10]{treves67}.

\subsection{Proof of \ncref{thm:shifted translation invariant RKHS contained}}
\label{app:shifted translation invariant RKHS contained}

The first result follows by the characterization of \citet{aronszajn50trk},
once we note that the product kernel
$(x,y) \mapsto k(x,y)k_2(x,y)$ contains the functions of the form $x \mapsto h(x) f(x)$ with $h \in \H_{k_2}$ and $f \in \H_{k}$, since it is the pullback under the diagonal map $x \mapsto (x,x)$ of the tensor product kernel $k \otimes k_2$, and the latter is the completion of the inner product space of functions $(x,y) \mapsto h(x)f(y)$.

The second and third result follow by
\citet[Prop.~3.1]{zhang2013inclusion} and the convolution theorem, which implies that
the (translation invariant) product kernel
$k(r)k_2(r) =  \hat \mu (r) \hat \nu(r) = \widehat{\mu \star \nu}(r)$.
Finally, for the final result, observe that \cref{bochner} implies $f$ is the Fourier transform of a non-negative finite measure $\nu$, which
satisfies for any $\gamma$ in the Schwartz space
$$\hat f \dd x [\gamma] \defn f \dd x[\hat \gamma] =  \hat{\nu}\dd x [\hat \gamma]=
\nu[\gamma \circ R]
\defn R_*\nu[ \gamma], $$
where $R_*$ is the pushforward,
and thus $R_*\nu = \hat f \dd x$ (i.e., $R_*\nu$ is the generalized Fourier transform of $f$), which implies $\nu = \hat f \circ R \dd x$.
 By
 \citet[Thm.~6.2]{wendland04scattered} $f$ is even.
 In fact $\hat f \circ R$ is also the generalized Fourier transform of $f$, and  $\hat f \circ R= \hat f$ almost everywhere.
 Indeed, on the one hand
 $ \int \hat f \gamma \dd x = \int \hat f \circ R \gamma \circ R R_* \dd x = \int \hat f \circ R \gamma \circ R \dd x$.
 On the other hand
 \begin{talign}
     \int \hat f \gamma \dd x = \int  f \hat \gamma \dd x  =
     \int  f \circ R \hat \gamma \dd x
     = \int  f \circ R \hat \gamma \circ R \circ R  \dd x
     = \int  f  \hat \gamma  \circ R  \dd x = \int  f  \widehat{ \gamma  \circ R} \dd x.
 \end{talign}
Since composition with $R$ is a bijection from the Schwartz space to itself, this shows that $\hat f \circ R$ is the generalized Fourier transform of $f$.
 The result then follows by the third result.

\section{Proof of \ncref{Stein kernel control tight convergence via bounded P separation}}
\phantomsection
\label{app: Stein kernel control tight convergence via bounded P separation}

Before proving the result, let us introduce the Banach space  $\BB{1}{\growth}(\R^d)$, a generalization of $\C{1}{0}(\R^d)$, which is easier to handle than the topological vector space $\C{1}{b,\theta}(\R^d)_\beta$ (the analogous generalization of $\C{1}{b}(\R^d)_\beta$).
 \begin{lemma}[Definition of $\BB{1}{\growth}(\R^d)$]
 \label{def:B^1_theta}
    Given a  continuous function $\growth: \R^d \to [c,\infty)$, for some $c>0$,  let  $\BB{1}{\growth}(\R^d)$ be the completion of $\cts^1_c(\R^d)$ with respect to
    \begin{align}\label{eq:bnorm}
        \| f\|_{\BB{1}{\growth}} \defn \sup \|\growth(x) f(x)\| +\sum_{|p| =1}\sup \| \partial^{p}_xf\|.
    \end{align}
    Then
    $\BB{1}{\growth}(\R^d) \cong \{ f \in \C{1}{}(\R^d): \growth f \in \C{}{0}(\R^d), \partial f \in \C{}{0}(\R^{d\times d}) \} $.
    \end{lemma}
\begin{proof}
   We first show that if $f \in \{ f : \C{1}{}(\R^d): \growth f \in \C{}{0}(\R^d), \partial f \in \C{}{0}(\R^{d\times d}) \} $, then $\exists c_n \in \C{1}{c}(\R^d)$
   such that $c_n \overset{\BB{1}{\growth}}{\to} f$.
   By definition $\forall \epsilon >0$ there exists compact subsets $S_1,S_2$ such that
   $\|\growth(x) f(x) \| < \epsilon$ for $x \in S_1^c$
   and
   $\|\partial f(x) \| < \epsilon$ for $x \in S_2^c$.
   Since $S_1 \cup S_2$ is compact, there exists a ball of radius $r$ such that  $ S\defn  S_1 \cup S_2 \subseteq \B_{r}$.
   Using Lemma 14 in \citet{gorham17measuring}
   we can find a function $ c_\epsilon \in \C{1}{c}(\R^d) $ with
   $$c_\epsilon:\R^d \to [0,1], \quad c_\epsilon|_{ \overline{\B_r}} =1, \quad
   c_\epsilon |_{\overline{\B_{r+2\delta}}^c} =0, \quad \| \partial c_\epsilon \| \leq \mathbb{I}[\overline{\B_{r+2\delta}}/\overline{\B_r}] $$
   for some $\delta>0$.
   In particular $\partial c_\epsilon =0$ and $c_\epsilon =1$ on $S \subseteq \overline{\B_r}$.
   Now let $f_\epsilon \defn f c_\epsilon \in \C{1}{c}(\R^d)$.
   Then on
   $S$ we have
   $\| \growth(x)f(x)-\growth(x)f(x) c_\epsilon(x) \|=0$
   and
   $$\| \partial f - \partial (f c_\epsilon) \|
   = \| \partial f - c_\epsilon \partial f - f \otimes \partial  c_\epsilon \|
   =
   \| \partial f - c_\epsilon \partial f\|=0.$$
   On $S^c$,
   we have
   $| \growth(x)f(x)-\growth(x)f(x) c_\epsilon(x) | \leq 2 | \growth(x)f(x)| \leq 2 \epsilon $
   and
   $$\| \partial f - \partial (f c_\epsilon) \|
   \leq \| \partial f \| + \| f \otimes \partial c_\epsilon \| +  \| c_\epsilon \partial f \| \leq 3 \epsilon. $$
   Thus $\C{1}{c}(\R^d)$ is dense in
   $\{ f : \C{1}{}: \growth f \in \C{}{0}, \partial f \in \C{}{0} \}$.

   On the other hand, suppose we have a Cauchy sequence $c_n \in \C{1}{c}(\R^d)$ for the norm \cref{eq:bnorm}.
   Then, since $\growth \geq c>0$,  $(c_n)_n$ is a fortriori a $\C{1}{0}(\R^d)$-Cauchy sequence, and thus $\|\cdot \|_{\C{1}{0}}$-converges   to a function $f \in \C{1}{0}(\R^d)$.
   Now we  show that $c_n$ also converges to $f$ in the norm defined in \cref{eq:bnorm}.
   Indeed $\forall \epsilon >0$ $\exists \ell$ such that $n,m\geq \ell$ implies
   $\|\growth(x)c_n(x)-\growth(x) c_m(x) \|\leq \epsilon$ for all $x$,
   and thus taking $m\to \infty$ gives
    $\|\growth(x)c_n(x)-\growth(x) f(x) \|\leq \epsilon$ for all $x$,
    i.e.,
    $\|\growth(c_n- f)\|_\infty \leq \epsilon$.
    An analogous argument shows  $\|\partial c_n - \partial f \|_\infty \to 0$, and thus
   $c_n \overset{\BB{1}{\growth}}{\to} f$.

   Finally, note that $\| \growth f - \growth c_n \|_\infty \to 0$ and $\growth c_n \in \C{}{c}(\R^d)$  imply $\growth f \in \C{}{0}(\R^d)$.
   Similarly
   $\| \partial f - \partial c_n \|_\infty \to 0$ implies
   $\partial f \in \C{}{0}(\R^{d\times d})$ since
   $\partial c_n \in \C{}{c}(\R^{d\times d})$.

\end{proof}

We will now first prove the non-tilted case, with $\tilt(x)=1$.
We will use the following result, proved in \cref{app: B1a universality from transformation}.

\begin{lemma}[Characteristicness of tilted bounded kernels]\label{cor: B1a universality from transformation}
    Using the notations of \cref{characteristic transformed RKHS}, let $\phi$ be the multiplication by $1/\growth$, where
    $\growth$ is a strictly positive $\C{1}{}$ function such that  $1/\growth, \partial (1/\growth) $ are bounded.
    If $K$ is universal to $\cts^1_0(\R^\ell)$ (resp. $\C{1}{b}(\R^\ell)_\beta$), then $K(x,y)/(\growth (x) \growth(y))$ is universal to $\BB{1}{\growth}(\R^\ell)$ (resp. $\C{1}{b,\growth}(\R^\ell)_\beta$).
\end{lemma}

Construct $\k_s$ and $\k_f$ satisfying the conditions
of \cref{thm:Schwarz-tilted sub-RKHS}.
Note $\kb_f$
 is characteristic to $(\C{1}{b,\growth})_\beta^*$, as can been seen by applying \cref{cor: B1a universality from transformation} to $\kb_s$ with $\growth(x) \defn 1/f(x)$, and recalling that the RKHS of $\kb_f(x,y)=f(x)\kb_s(x,y) f(y)$ is $f \H_{\kb_s}$.

 Summarizing, we have proven that $\H_{\kb_f} \subseteq \H_{\kb}$ and that $\kb_f$ is characteristic to $(\C{1}{b,\growth})_\beta^*$.
 Hence, $\H_{\tilde K} \subseteq \H_K$, where $\tilde K \defn \kb_f \Id $ is characteristic to
 $\C{1}{b,\growth}(\R^d)_\beta^*$ by \cref{thm:scalar Vs vector characteristicness}.
The Stein RKHS associated to $\tilde K$ consists of bounded functions and is characteristic to $\P \in \Pset$ by \cref{thm:characteristic to P using C^1_a}, because $D_\Q \in \C{1}{b,\growth}(\R^d)_\beta^* $ for any probability measure $\Q$.
Moreover, since $\H_{\ks}$ is a superset of a bounded $\P$-separating sub-RKHS,  the result then follows from \cref{thm:bounded separating is P-characteristic}. \\

Finally, consider a tilting function $\tilt$.
In the \cref{app:Schwarz-tilted sub-RKHS} 
we have constructed a Schwartz function $f$ that is strictly  positive and  s.t., $f$ and its partial derivatives have root exponential decay.
Moreover we have shown that
$\H_{\kb_f} \subseteq \H_{\kb} $,
where $\kb_f(x,y) \defn f(x)\kb_s(x,y)f(y)$ with $\kb_s$ a kernel obtained by ironing and scaling $\kb$, and shown that
$\kb_f$ is universal to $(\C{1}{b,\growth}(\R^d))_\beta$ (with $\growth(x) \defn 1/f(x)$).
Hence $\tilt \H_{\kb_f} \subseteq \tilt \H_{\kb} $.
Since $\tilt f$ and $\partial(\tilt f)$ are bounded, and $\kb_s$ is universal in $(\C{1}{b}(\R^d))_\beta$,
then by \cref{cor: B1a universality from transformation} $\tilt(x)\kb_f(x,y) \tilt(y)$ is universal to $(\C{1}{b,\frac{\growth}{\tilt}}(\R^d))_\beta$.

\subsection{Proof of \ncref{cor: B1a universality from transformation}}
\label{app: B1a universality from transformation}

Note $\phi : \cts^1_0(\R^\ell) \to \BB{1}{\growth}(\R^\ell)$ is continuous,
     indeed (here $\growth^{-1} \defn 1/\growth$)
     $$ \| \phi(f) \|_{B^1_{\growth}} \leq  \| f \|_{\infty} + \| \growth^{-1}\partial f +  \partial \growth^{-1} \otimes f \|_\infty  \leq  \| f \|_{\infty} + \| \growth^{-1} \|_{\infty} \| \partial f \|_\infty+\|  \partial \growth^{-1} \otimes f \|_\infty \leq  C\| f \|_{\cts^1_0}, $$
     for some $C>0$ (where we have used the boundedness of $| \growth^{-1}|$ and $\| \partial \growth^{-1} \|$), where $\otimes $ is the outer product.
     Similarly,  $\phi$ is continuous as a map
     $\C{1}{b}(\R^\ell)_\beta \to \C{1}{b,\growth}(\R^\ell)_\beta$,
     since  for any $\gamma \in \C{}{0}$
     $$  \| \gamma \growth  \phi(f)\|_\infty = \| \gamma f \|_\infty$$
     and
     $$ \| \gamma \partial_i \phi(f)\|_\infty \leq \| \gamma \growth^{-1} \partial_i f\|_\infty + \| \gamma f \partial_i \growth^{-1}\|_\infty \leq  \| \growth^{-1}  \|_\infty \| \gamma
     \partial_i f\|_\infty + \| \partial_i \growth^{-1} \|_\infty \| \gamma f \|_\infty. $$

     Moreover  $\phi( \cts^1_c(\R^\ell))=\cts^1_c(\R^\ell)$ since $\growth^{-1} \in \C{1}{}$ is strictly positive, and $\C{1}{c}(\R^\ell)$ is dense in $\C{1}{b,\growth}(\R^\ell)$, and in $\BB{1}{\growth}(\R^\ell)$ since the latter is its completion (and  metric spaces are dense in their completion).
The result then follows by \cref{characteristic transformed RKHS}.

\section{Proof of \ncref{tightness}}
\label{app:proof of enforcing tightness}

Fix any $\epsilon > 0$, and pick any function $h \in \F$ and compact set $C$ satisfying
\begin{talign}
h - \P h \geq \indic{ C^c} - \eps/2.
\end{talign}
Moreover, suppose $d_{\F}(\Q_n, \P) \to 0$.
For each $n$, we have (note $h$ is bounded below, so $h_+ \in \L^1(\Q)$ for all $\Q \in \Pset$)
\begin{align}
\Q_n(C^c)
  \le \eps/2 + \Q_n h - \P h ,
\end{align}
and, since $h \in \F$ and $d_{\F}(\Q_n, \P)\to 0$, we further have $|\Q_n h - \P h| \leq \eps/2$ for all $n$ larger than some $N$.
Hence, $\Q_n(C^c) \leq \eps$ for all $n$ sufficiently large.
Since $\eps >0$ was arbitrary, $(\Q_n)_{n\ge 1}$ is  tight.

Finally, if $\k$ enforces tightness and controls tight weak convergence, then
$\mmd_{\k}(\Q_n, \P) \to 0$ implies $(\Q_n)$ is tight, so $\Q_n \to \P$, i.e., $\k$ controls weak convergence.

\section{Proof of \ncref{coercive-tightness}}
\label{app:proof of enforcing tightness with coercivity}

Since $\P\in\embedpettis{k}$, $\P h$ is finite and hence $h-\P h$ is also coercive and bounded below.
Since $h-\P h$ is bounded below, there exists $C > 0$ such that $  (h-\P h)/C \geq -1 $. Moreover, for any $\epsilon >0$, writing
$h_\epsilon \defn h \epsilon /C \in \HK $,
then $h_\epsilon -\P h_\epsilon \geq - \epsilon$,
and since $h_\epsilon -\P h_\epsilon$ is coercive, there  exists a compact set $S$ for which
$\inf_{x \in S^c} h_\epsilon -\P h_\epsilon \geq 1- \epsilon $, and therefore $\HK$ $\P$-dominates indicators.

\section{From Separating Measures in \tpdf{$\H_{\ks}$}{H-kp} to Separating Schwartz Distributions in \tpdf{$\H_K$}{H-K}}
\label{app:From separating measures in Hks to separating Schwartz distribution in HK}

For a bounded RKHS we have the following result shown in \cref{app:proof characteristic to P using C^1_a}:

\begin{proposition}[Separating measures with bounded Stein RKHSes]\label{thm:characteristic to P using C^1_a}
Suppose
$\| \sp(x) \| \leq \growth(x)$, and
$\H_K \subseteq \C{1}{b,\growth}(\R^d)$.
Then
$\k_\P$ is $\P$-separating  iff
$K$ is characteristic to $0$ in
$\{ D_\Q : \Q \in \Pset \} \subseteq   (\C{1}{b,\growth}(\R^d)_\beta)^*$, that is
for any $\Q \in \Pset$, $D_\Q |_{\H_K} = 0 \implies \Q = \P$.
\end{proposition}

In order to prove this it will be convenient to first prove in \cref{app:proof of characteristic Stein kernel via B1a spaces} the  analogous but simpler result, \cref{thm:characteristic Stein kernel via B1a spaces},  which relies on the  Banach space defined in  \cref{def:B^1_theta}.

\begin{proposition}[Separating measures with \tpdf{$\C{}{0}$}{C0} Stein RKHSes]
\label{thm:characteristic Stein kernel via B1a spaces}
Suppose
$\| \sp(x) \| \leq \growth(x)$,
and
$\H_K \subseteq \BB{1}{\growth}(\R^d)$.
Then
$\k_\P$ is  $\P$-separating   iff
$K$  separates $0$ from $\{ D_\Q : \Q \in \Pset \} \subseteq \BB{1}{\growth}(\R^d)^*$,
i.e.,
for any $\Q \in \Pset$, $D_\Q|_{\H_K} = 0 \implies \Q = \P$.
\end{proposition}

 \subsection{Proof of \ncref{thm:characteristic Stein kernel via B1a spaces}}
     \label{app:proof of characteristic Stein kernel via B1a spaces}

     Proceeding as in \cref{remark:bounded to unbounded}, we can define a Banach space $\BB{0}{\growth}(\R^d)$ such that division by $\growth$ yields an isometric isomorphism $ \cts_0^0(\R^d) \cong \BB{0}{\growth}(\R^d)$.
     Note the continuous inclusion $ \BB{1}{\growth}(\R^d) \hookrightarrow \BB{0}{\growth}(\R^d)$.\footnote{When $\theta \geq 1$ we also have
          $\BB{1}{\growth}(\R^d) \hookrightarrow \cts^1_0(\R^d)$.}
          In particular  $\growth \Q$, is a continuous linear functional on $\BB{1}{\growth}(\R^d)$, and hence so is $\sp \Q$, since
          $$| \sp \Q(f)| \defn |  \sum_i \Q(\spi f_i)  | \leq  \sum_i \Q( | \spi f_i|) \leq \tilde C
          \sum_i \Q( \growth |  f_i|) \leq C  \sup \| \growth f \| \leq C \, \| f \|_{\BB{1}{\growth}}$$
          for some constants $\tilde C,C> 0$ (that arise from the equivalence of norms on $\R^d$).
          Moreover,  $ \partial_{i} \Q $, and hence  $D_Q $, acts continuously on $\BB{1}{\growth}(\R^d)$, since
          $$|\partial_i \Q(f) | = | \Q \partial_i f | \leq \| \partial_i f \|_\infty \leq \| f \|_{\BB{1}{\growth}(\R^d)}.$$

          Moreover $\H_K \subseteq \BB{1}{\growth}(\R^d)$ implies
          $\H_K \hookrightarrow \BB{1}{\growth}(\R^d)$.
          Indeed, recalling that $K_x^* \in \mathcal B(\H_K, \R^d)$ is the evaluation functional,
          $\| \growth(x) K_x^*h \| = \| \growth(x) h(x) \| \leq \| h \|_{\BB{0}{\growth}(\R^d)}$ for all $h \in \H_K$,
          so by the
          Banach--Steinhaus Theorem
          $  \| \growth(x) K_x^*\| \leq C $ for some $C <  \infty$.
          From this we find that
          $\sup_x \| \growth(x) h(x) \| = \sup_x \| \growth(x) K_x^* h \|  \leq  \sup_x  \|\growth(x) K_x^* \| \| h \|_{\H_K} \leq C \| h \|_{\H_K}$.
          Proceeding analogously with the derivative contribution, we can use $ |  \metric{\partial^p_2 K^{e_i}(.,x)}{h}_k | \leq  \|h\|_{\BB{1}{\growth}}$
          to show
          $ \| \metric{\partial^p_2 K^{e_i}(.,x)}{ \cdot}_k \| \leq  A^i$, for some $A^i < \infty$,
          and then
          $\sup_x \|\partial^p h(x) \|  \leq  B  \sup_x  \max_i | \partial^p h^i(x) |  \leq    d B \max_i A_i \| h \|_{\H_K}$,
          which
          yields the continuity of the inclusion.

     Now, note $\H_{\k_\P} \subseteq \cts_0$, so any probability measure $\Q$ embeds into the Stein RKHS by \cref{thm:measure embed iff RKHS integrable}.
     Moreover, since by assumption the embedding of $\P$ into $\H_{\ks}$ is the null function, from \cref{ksdAlt}
         $$\ksd{\Q} = \| \Q \|_{\H_{\k_\P}} = \| \Q \circ \langevin \|_{\H_K}. $$
        We thus want to show that
         ``$ \Q \circ \langevin|_{\H_K}=0$ implies
         $\Q \circ \langevin |_{\BB{1}{\growth}}=0$'' iff
         ``$\ksd{\Q}=0$ implies $\Q = \P$''.

         If $\ksd{\Q}=0$ implies $\Q = \P$, then $ \Q \circ \langevin|_{\H_K}=0$ implies $\Q =\P$, so we want to show 
         $\P \circ \langevin |_{\BB{1}{\growth}}=0$.
         For this we can use $\P \circ \langevin |_{\cts^1_c} =0$ by the divergence theorem and  the fact $\cts^1_c(\R^d)$ is dense in $\BB{1}{\growth}(\R^d)$
         since the latter is its completion (and  metric spaces are dense in their completion).

         Conversely, we have that
         $\ksd{\Q}=0$ implies $D_\Q |_{\BB{1}{\growth}}=0$, that is
         the distributional Stein equation
         \begin{talign}  
         \sum_i (\sp^i \Q - \partial_{x^i}\Q)e^i =0, \end{talign}
         where $(e^i)_{i=1}^{i=d}$ is the dual basis to the canonical basis of $\R^d$.
         Applying to this vectorial distributional PDE compactly supported smooth vector fields of the form $f=(0,\ldots, 0, l, 0, \ldots)$ with $l \in \C{\infty}{c}$, yields the system of (scalar) distributional PDEs
         $\partial_{x^i} \Q = \spi \Q$.
         In particular, solving for the function $q:\R^d \to \R$ the classical PDE
          $\partial_{x^i} q = \spi q$,
           implies  $q$ is the target probability density, $\q=\p$. We then look for solutions via the method of variation of constants.
           We write the  the form $\Q =\p T$. Subbing in and using $\partial_{x^i}(\p T) = \partial_{x^i} \p \, T + \p \partial_{x^i} T$ we obtain the equivalent distributional PDE
         $ \partial_{x^i} T =0$, which implies that $T$ is a translation-invariant measure, and hence proportional to the Lebesgue measure, $T=C \mathrm{d} x$, by \citet[Thm.~VI of Chap.~II]{schwartzTD}. Since $\Q$ is a probability measure we must have $C=1$, and thus $\Q = \P$.

\subsection{Proof of \ncref{thm:characteristic to P using C^1_a}}
         \label{app:proof characteristic to P using C^1_a}

          Since $\H_K \subseteq \C{1}{b,\growth}(\R^d) $, then $\H_{\ks} \subseteq \C{}{b}$, indeed $| \mathcal S_\p(h)(x)| \leq \| \sp(x) \| \| h(x) \|+  | \nabla_x \cdot h| \leq \growth(x) \| h(x) \|+| \nabla_x \cdot h|  $ and the latter is a bounded function of $x$.
         Hence any probability measure embeds into $\H_{\ks}$ by \cref{thm:measure embed iff RKHS integrable}, and we can proceed as above once we have shown that $D_\Q$ is continuous on $\C{1}{b,\growth}(\R^d)_\beta$.
         First note that, using the fact that the dual of a finite product of TVS is isomorphic to the finite product of their duals (see, e.g., \citet[p.~259]{treves67})
         \begin{talign}
         \DL{1}(\R^d) \defn (\C{1}{b}(\R^d))_\beta^* =
         ( \prod_{i=1}^d \C{1}{b}(\R)_\beta)^* =\oplus_{i=1}^d\DL{1} \subseteq \C{1}{b,\growth}(\R^d)_\beta^*,
         \end{talign}
         indeed, $\C{1}{b,\growth}(\R^d)_\beta \hookrightarrow \C{1}{b}(\R^d)_\beta$ since $\growth \geq c >0$, and $ \C{1}{b}(\R^d)_\beta^* = \DL{1}(\R^d)$ by \citet[Sec.~1]{conway1965strict}.
         Thus $\partial_{x^i}\Q \, e^i \in (\C{1}{b,\growth}(\R^d)_\beta^*$.
         \footnote{A more direct proof that establishes the continuity of $\partial_i \Q e^i$ on $\C{1}{b,\growth}(\R^d)_\beta$ reads as follows:
         if $f \in \C{1}{b,\growth}(\R^d)$, then
         $| \partial_i \Q e^i f| = | \Q \partial_i f^i | \leq  A \max_{j \in [n]} \| \gamma_j \partial_i f^i \|_\infty$ for some $\gamma_j \in \C{}{0}$, $n \in \mathbb{N} $, $A>0$, by continuity of $\Q$  on $(\C{}{b})_\beta$.
         Since $f \mapsto \| \gamma_j \partial f \|_\infty$ are semi-norms on $\C{1}{b,\growth}(\R^d)_\beta$ the result follows.}

         Moreover, we have $\C{1}{b,\growth}(\R^d)_\beta \hookrightarrow \C{}{b,\growth}(\R^d)_\beta \cong \C{}{b}(\R^d)_\beta$, where the latter  isomorphism of topological vector spaces is given by multiplication by $\growth$.
         Since $(\C{}{b})_\beta^*$ is the space of finite Radon measures,  this shows that $\growth \Q e^i \in  \C{}{b,\growth}(\R^d)_\beta^* \subseteq \C{1}{b,\growth}(\R^d)_\beta^*$.

 \section{Proof of \ncref{imq-tightness}}
 \label{app:imq-tightness}

 Our proof parallels that of \citet[Lem.~16]{gorham17measuring}.
Fix any $c>0$, $\gamma \in (0,2u-1)$, $a > c/2$, and $\alpha \in (1-u,\half (1-\gamma))$,  and consider the functions
\begin{align}
    g_j( x) = -x_j(a^2 + \twonorm{ x}^2)^{\alpha - 1}
    \qtext{for} 1 \leq j \leq d.
\end{align}
By
\citet[proof of Lem.~16]{gorham17measuring}, $g = (g_1, \dots, g_d) \in \H_K$ for $K = \kb\Id$.
Moreover, the Stein operator applied to $g$ takes the form
\begin{talign}
\langevin(g)(x)
    = {-\frac{\inner{\sp( x)}{x}}{(a^2 + \twonorm{ x}^2)^{1-\alpha}}}
    {-\frac{d}{(a^2 + \twonorm{ x}^2)^{1-\alpha}}
    + \frac{2(1-\alpha) \twonorm{ x}^2}{(a^2 + \twonorm{ x}^2)^{2-\alpha}}}.
\end{talign}
Since $\alpha < 1$, the final two terms in this expression are  uniformly bounded in $ x$.
Meanwhile, our generalized dissipativity assumption \cref{eq:generalized-dissipativity} implies that $-\inner{\sp( x)}{x} = \Omega(\twonorm{x}^{2u})$ as $\twonorm{x}\to\infty$, so $-\frac{\inner{\sp( x)}{x}}{(a^2 + \twonorm{ x}^2)^{1-\alpha}} = \Omega(\twonorm{x}^{2u-2+2\alpha}) = \omega(1)$  since $\alpha > 1 - u$.
Hence, $\langevin(g)$ is coercive.
In addition, the generalized disspativity condition \cref{eq:generalized-dissipativity} implies that $-\inner{\sp( x)}{x}$ is bounded below and hence that $\langevin(g)$ is bounded below.

Let $K = \kb \Id$.
Since $\sp$ is well defined and continuous on $\R^d$, the density $\p$ is strictly positive and continuously differentiable.
In addition, since $\P \in \Pset_{\sp}$, $K \in \C{(1,1)}{b}(\R^d)$, and $K \in \L^1(\P)$, \cref{thm:embeddability_conditions} and \cref{ksdDef} imply that $\p\H_K \subseteq \C{1}{}(\R^d)$, $\P\in\embedtozero$, $\KSD_{K,\P} =\mmd_{\ks}(\cdot,\P)$, and $\langevin(\H_K) = \H_{\ks}$.
Since $\langevin(g) \in \H_{\ks}$, $\Hks$  $\P$-dominates indicators by \cref{coercive-tightness} and enforces tightness by \cref{tightness}.

Finally, since $\kb \in \C{(1,1)}{b}$ is translation-invariant with a spectral density bounded away from zero in a neighborhood around the origin \citep[Thm.~8.15]{wendland04scattered}, we conclude that $\H_{\kb} \subseteq \C{1}{b}$ by \cref{thm:embeddability_conditions} and that $\ks$ controls $\P$ convergence by \cref{ksd-tightness}.

\section{Proof of \ncref{tilted-tightness}}
\label{app:proof tilted tightness}
Our aim is to identify a function in $\Hks$ that satisfies the indicator bounding property \cref{eq:indic-approx} for each $\epsilon > 0$.
To this end, for each $m \in \N$, define the compact set $C_m  = \{ x \in\R^d :\twonorm{ x} \leq m\}$, and fix any $m > 1$ for which
$-\inner{\sp( x)}{ x} $ is nonnegative on  $C_{m-1}^c$ and
\begin{talign}\label{eq:m-properties}
-\inner{\sp( x)}{ x} - r_0 \onenorm{\sp( x)} - 1 - r_0 - 2|\gamma|(1+\sqrt{d}/(c^2+m)) \geq (r_1/2)\twonorm{ x}^{2u}
\end{talign}
holds on $C_{m}^c$.
These properties hold for all $m$ sufficiently large (specifically, for all $m$ such that $\half r_1 m^{2u} \geq 1+r_0-r_2+ 2 | \gamma | (1+ \frac{\sqrt{d}}{c^2+m})$) due to generalized dissipativity \cref{eq:generalized-dissipativity} with $u > 0$.
Fix also any $\epsilon_m \in (0, s]$ satisfying
\begin{talign}\label{eq:epsm-properties}
\eps_m
	\sup_{ x \in C_m }
	\max\left( \onenorm{\sp( x)}  a(\twonorm{ x}),
	\frac{2|\gamma| \onenorm{ x}}{c^2 + \twonorm{ x}^2} a(\twonorm{ x}),
	\frac{ \onenorm{ x} }{c^2 + \twonorm{ x}^2} ,
	a(\twonorm{ x})
	\right) \leq \frac{a(m)}{m}.
\end{talign}
Consider the smoothed indicator function
\[
\tilde{f}_m( x) = \sigma(m - \twonorm{ x} )
\qtext{for}
\sigma(r) = 2\max(0,r)^2 \indic{r < .5} + (1-2\max(0,1-r)^2) \indic{r \geq .5}
\]
which satisfies
$\tilde{f}_m( x) = 1$ on $C_{m-1}$, $\tilde{f}_m( x) = 0$ on $C_{m}^c$,
\[
\indic{ x \in C_{m-1}} \leq \tilde{f}_m( x) \leq \indic{ x \in C_m}, \qtext{and }
-\indic{ x \in C_{m} \backslash C_{m-1}} \leq \dx \tilde{f}_m( x) \leq 0.
\]
Moreover, for each $i \in \{1,\dots, d\}$, $ x \mapsto x^i \tilde{f}_m( x) \in \C{1}{0}$.

Since $\HK \subseteq \C{1}{0}$, then
$\HK \hookrightarrow \C{1}{0}$,
so by \citet[Thm.~6]{simon18kde}
$\HK$ is dense in $\C{1}{0}$.
Hence, for each $i \in \{1,\dots, d\}$ there exists $\tilde{g}_{mi} \in \HK$ satisfying
\begin{talign}\label{eq:tildeg-guarantee}
  \sup_{ x\in\R^d} \max( |\tilde{g}_{mi}( x) - x^i\tilde{f}_m( x)|,  |\dx \tilde{g}_m( x) - \dx( x^i\tilde{f}_m( x))|) \leq \epsilon_m.
\end{talign}
Moreover, the function $w_i( x) = x^i$ belongs to $\H_{\tilde{\k}_i}$ for $\tilde{\k}_i( x, y) \defn x^i y^i$.
Since $\HK \subseteq \H_{\k+\tilde{\k}_i}$ and $\H_{\tilde{\k}_i} \subseteq \H_{\k+\tilde{\k}_i}$ (see for example \citet[Prop.~5]{carmeli2010vector}), the functions $\tilde{g}_{mi}, w_i, $ and  $g_{mi} = w_i - \tilde{g}_{mi} $ are all elements of $\H_{\k+\tilde{\k}_i}$.

Consider now the Stein function
\begin{talign}
h_m( x)
	&= \sum_{i=1}^d \frac{-\dx (p( x)  a(\twonorm{ x}) g_{mi}( x)  ))}{p( x)} \\
	&= \underbrace{-\inner{\sp( x)}{\gvec_m( x)} a(\twonorm{ x})}_{(i)} \underbrace{- \inner{\gvec_m( x)}{\grad a(\twonorm{ x})}}_{(ii)}  \underbrace{- a(\twonorm{ x}) \grad \cdot \gvec_m( x)}_{(iii)}, \label{eq:hm-expansion}
\end{talign}
where $\gvec_m$ is the vector valued function $(g_{mi})_{i=1}^d$.
By construction, $h_m \in \langevin \H_K$,
and thus in $\H_{\ks}$ by \cref{ksdDef}.
Therefore, the zero-mean embedding assumption $\P \in \embedtozero$ and  \cref{thm:embeddability_conditions} imply that $\P h_m =0$.
We will show that a rescaled version of $h_m$ satisfies the indicator bound property \cref{eq:indic-approx} for a choice of $\tilde{\epsilon}_m$ that decays to $0$ as $m \to\infty$.
We begin by lower-bounding each of the components in the expansion \cref{eq:hm-expansion}.

To lower-bound term (i), we first record several properties of $\gvec_m$.
First, our approximation guarantee  \cref{eq:tildeg-guarantee} implies
\begin{talign} \label{eq:g-guarantee}
\sup_{ x\in\R^d} |g_{mi}( x) -  x^i f_m( x)| \leq \eps_m \qtext{for each}  i \in \{1, \dots, d\},
\end{talign}
where $f_m \defn 1 - \tilde{f}_m$ satisfies
\begin{talign} \label{eq:fm-properties}
\indic{ x \in C_{m-1}^c} \geq f_m( x) \geq \indic{ x \in C_m^c}
\qtext{and}
\indic{ x \in C_{m} \backslash C_{m-1}} \geq \dx f_m( x) \geq 0.
\end{talign}
Since $a$ is nonnegative and $f_m( x) = 0$ on $C_{m-1}$,
\Holder's inequality,
the guarantee \cref{eq:g-guarantee},
the assumed nonnegativity of $-\inner{\sp( x)}{ x} $ on $C_{m-1}^c$,
generalized dissipativity, and our choice \cref{eq:epsm-properties} of $\epsilon_m$ implies that
\begin{talign}
-&\inner{\sp( x)}{\gvec_m( x)}  a(\twonorm{ x})
	= -\inner{\sp( x)}{ x} f_m( x) a(\twonorm{ x})
	- \inner{\sp( x)}{\gvec_m( x) -  x f_m( x)}  a(\twonorm{ x})
	\\
	&\geq -\inner{\sp( x)}{ x} f_m( x) a(\twonorm{ x})
	- \onenorm{\sp( x)}\infnorm{\gvec_m( x) -  x f_m( x)}  a(\twonorm{ x})
	\\
	&\geq -\inner{\sp( x)}{ x} f_m( x) a(\twonorm{ x}) - \onenorm{\sp( x)} a(\twonorm{ x}) \eps_m \\
	&\geq -\inner{\sp( x)}{ x} \indic{ x \in C_m^c} a(\twonorm{ x}) - \onenorm{\sp( x)} a(\twonorm{ x}) \eps_m \\
	&= (-\inner{\sp( x)}{ x}- \onenorm{\sp( x)}\eps_m) \indic{ x \in C_m^c} a(\twonorm{ x}) - \onenorm{\sp( x)} \indic{ x \in C_m} a(\twonorm{ x}) \eps_m \\
	&\geq (-\inner{\sp( x)}{ x}- \onenorm{\sp( x)}s) \indic{ x \in C_m^c} a(\twonorm{ x}) - a(m)/m.
\end{talign}

To lower bound (ii), we again employ \Holder's inequality, the approximation guarantee \cref{eq:g-guarantee}, and the  $\epsilon_m$ properties \cref{eq:epsm-properties} to find that
\begin{talign}
- &\inner{\gvec_m( x)}{\grad a(\twonorm{ x})}
	= 2\gamma \inner{\gvec_m( x)}{ x} / (c^2 + \twonorm{ x}^2)^{\gamma+1} \\
	&= 2\gamma \twonorm{ x}^2 f_m( x)/ (c^2 + \twonorm{ x}^2)^{\gamma+1}
	+
	2\gamma \inner{\gvec_m( x)- x f_m( x)}{ x} / (c^2 + \twonorm{ x}^2)^{\gamma+1} \\
	&\geq 2\gamma \twonorm{ x}^2 f_m( x)/ (c^2 + \twonorm{ x}^2)^{\gamma+1}
	-
	2\gamma \infnorm{\gvec_m( x)- x f_m( x)}\onenorm{ x} / (c^2 + \twonorm{ x}^2)^{\gamma+1} \\
	&\geq 2\gamma \twonorm{ x}^2 f_m( x) / (c^2 + \twonorm{ x}^2)^{\gamma+1} - 2|\gamma| \onenorm{ x} \eps_m/  (c^2 + \twonorm{ x}^2)^{\gamma+1} \\
	&\geq -2|\gamma| \twonorm{ x}^2 \indic{ x \in C_{m-1}^c} / (c^2 + \twonorm{ x}^2)^{\gamma+1} - 2|\gamma| \onenorm{ x} \eps_m/  (c^2 + \twonorm{ x}^2)^{\gamma+1} \\
	&= -2|\gamma|\frac{ \twonorm{ x}^2}{c^2 + \twonorm{ x}^2} a(\twonorm{ x}) (\indic{ x \in C_{m}^c} + \indic{ x \in C_{m}\backslash C_{m-1}}) \\
	&- \frac{2|\gamma| \onenorm{ x} \eps_m}{c^2 + \twonorm{ x}^2} a(\twonorm{ x})  (\indic{ x \in C_{m}^c} + \indic{ x \in C_{m}}) \\
	&\geq -2|\gamma|\bigg(\frac{ \twonorm{ x}^2+ \onenorm{ x} \eps_m}{c^2 + \twonorm{ x}^2}\bigg) a(\twonorm{ x}) \indic{ x \in C_{m}^c}
	- 2|\gamma|\max(a(m-1), a(m))
	- a(m)/m \\
	&\geq -2|\gamma|(1+\sqrt{d}/(c^2+m)) a(\twonorm{ x}) \indic{ x \in C_{m}^c}
	- 2|\gamma|\max(a(m-1), a(m))
	- a(m)/m.
\end{talign}

To lower bound (iii), we first note that the derivative approximation \cref{eq:tildeg-guarantee} implies
\begin{talign} \label{eq:grad_g-guarantee}
\sup_{ x\in\R^d} |\dx g_{mi}( x) -  \dx( x^i f_m( x))| \leq \eps_m \qtext{for each}  i \in \{1, \dots, d\}.
\end{talign}
Moreover, we have
\[
\dx( x^if_m( x)) = f_m( x) + x^i \dx f_m( x) \leq \indic{ x \in C_{m-1}^c} + |x^i| \indic{ x \in C_{m} \backslash C_{m-1}}
\]
by our $\dx f_m$ constraints \cref{eq:fm-properties}.
Therefore, the nonnegativity of $a$  and the  $\epsilon_m$ properties \cref{eq:epsm-properties} give the bound
\begin{talign}
-&a(\twonorm{ x}) \grad \cdot \gvec_m( x)
	\geq -a(\twonorm{ x})(\indic{ x \in C_{m-1}^c} + \onenorm{ x} \indic{ x \in C_{m} \backslash C_{m-1}} + \eps_m)\\
	&= -a(\twonorm{ x})(1+\eps_m)\indic{ x \in C_{m}^c} - a(\twonorm{ x})(1+\onenorm{ x}) \indic{ x \in C_{m} \backslash C_{m-1}} - a(\twonorm{ x})\eps_m\indic{ x \in C_{m}} \\
	&\geq -a(\twonorm{ x})(1+s)\indic{ x \in C_{m}^c} - \max(a(m-1),a(m)) (1+\sqrt{d}m) - a(m)/m.
\end{talign}

Our assumption $\gamma \leq u$ implies that $\twonorm{ x}^{2u} a(\twonorm{x}) \geq m^{2u} a(m)$ whenever $\twonorm{ x} \geq m$.
This fact combined with our collected results and the assumed growth \cref{eq:m-properties} induced by our choice of $m$ now imply that
\begin{talign}
h_m( x)
	&\geq  (-\inner{\sp( x)}{ x}- \onenorm{\sp( x)}r_0  - 1 - r_0 - 2|\gamma|(1+\sqrt{d}/(c^2+m))) \indic{ x \in C_m^c} a(\twonorm{ x}) \\
	 &- 3a(m)/m
	- (1+\sqrt{d}m + 2|\gamma|)\max(a(m-1), a(m)) \\
	&\geq  (r_1/2)\twonorm{ x}^{2u} \indic{ x \in C_m^c} a(\twonorm{ x}) - 3a(m)/m	- (1+\sqrt{d}m + 2|\gamma|)\max(a(m-1), a(m)) \\
	&\geq  (r_1/2)m^{2u} a(m) \indic{ x \in C_m^c} - 3a(m)/m	- (1+\sqrt{d}m + 2|\gamma|)\max(a(m-1), a(m)) .
\end{talign}
Hence, the rescaled Stein function $\tilde{h}_m = h_m / ((r_1/2)m^{2u} a(m))$, satisfies the indicator approximation property \cref{eq:indic-approx} for the compact set $C_m$ and the approximation factor
\[
\tilde{\eps}_m = 6/(r_1 m^{2u+1}) +  (1+\sqrt{d}m + 2|\gamma|)\max(a(m-1), a(m))/((r_1/2)m^{2u} a(m)).
\]
Since $u > 1/2$, $\tilde{\eps}_m$ vanishes as $m \to \infty$, and hence $\Hks$ $\P$-dominates indicators.
Thus by  \cref{tightness} the Stein kernel enforces tightness.\\

For (b), we use \cref{thm:Scaled tilted universal kernel controls tight convergence}:
\begin{lemma}[Universal KSDs tilted by score growth control tight convergence]
\label{thm:Scaled tilted universal kernel controls tight convergence}
Suppose that
 $\| \sp( x) \| \leq (c^2+\|  x \|^2)^{\gamma}$, where $c \neq 0$, $\gamma \geq 0$,
 and
 $\Kb$ is characteristic to $\DL{1}(\R^d)$.
 Then
 the Stein kernel induced by $(c^2+\|  x \|^2)^{-\gamma} \Kb( x,  y)  (c^2+\|  y \|^2)^{-\gamma}$
 is $\P$-separating and controls tight $\P$-convergence.
\end{lemma}
\begin{proof}
 The result follows by \cref{thm: tilted controls tightness}.
     Indeed the function $\growth(x) \defn (c+ \| x \|^2)^{\gamma}$ has  $\partial \nicefrac{1}{ \growth} (x) = -2\gamma  x (c+ \| x \|^2)^{-\gamma-1}$ satisfies the assumption of \cref{thm: tilted controls tightness}, so the result follows.
\end{proof}

This shows that we can easily construct bounded Stein kernels that control \emph{tight} weak convergence to $\P$ in $\Pset$ by simply tilting the base kernel through a function that bounds the score.
By \cref{thm:Scaled tilted universal kernel controls tight convergence}
the Stein kernel induced by the tilted base kernel $a(\| x \|) \k(x,y) a(\|y \|)$
controls tight weak convergence,
and thus so does the overall Stein kernel (which further controls weak convergence since it enforces tightness)
as it may be viewed as the sum of two Stein kernels.
Indeed, as proved in \cref{app:MMD control of subset RKHS},  we have the following general bound between MMDs when an RKHS contains another one:
\begin{lemma}[MMD controls subset MMDs]
\label{MMD control of subset RKHS}
    Suppose $\H_\k \subseteq \H_{\tilde k}$ and that $\P\in\embedpettis{\tilde k}$.
    Then $\exists c\geq 0$ such that for all $\Q \in \Pset$
      $$ \mmd_\k(\Q,\P) \leq c \, \mmd_{\tilde k}(\Q,\P). $$
      Hence,
      \begin{itemize}
          \item[(i)] If $\k$ is $\P$-separating then $\tilde k$ is $\P$-separating.
          \item[(ii)] If $\k$ controls (tight) weak $\P$-convergence, then $\tilde k$ controls (tight) weak $\P$-convergence.
      \end{itemize}
\end{lemma}

Finally, for (c),
first note that
$\H_{\ks} \subseteq \C{}{b}$.
Indeed for any $h \in \H_{\ks}$ we have
$ h(x) = \langevin(ag) = \metric{\sp(x)}{ag}+ \nabla \cdot (ag) =
\metric{\sp}{ag} + a \nabla \cdot g+ \metric{g}{\partial a}$,
for some vector-valued function $g$ with $ g_i \in \H_{\k +\tilde \k_i} $, so $h$ is continuous.
Moreover it is bounded since (i)  $| g_i(x) | \leq \| g_i \|_{\k +\tilde \k_i} ( \sqrt{\k(x,x)} +| x^i |)$,
implies
\begin{talign}
|\metric{g}{\partial a}| \leq  \| g_i \|_{\k +\tilde \k_i}\sum_i 2\gamma  \frac{|x^i|\sup_x\sqrt{\k(x,x)} +| x^i |^2}{(c^2+\| x\|^2)^{\gamma+1}}.
\end{talign}
(ii) $a \nabla \cdot g$ is bounded
since $\partial_i g \in \H_{\partial_i \partial_{i+d} \k +1} \subseteq \C{}{b}$.
(iii)
$\metric{\sp(x)}{ag}$ is bounded since
$$|\metric{\sp(x)}{a(x)g(x)}| \leq \| \sp(x)\| |a(x)| \| g(x)\| \leq
\|\sp(x)\| \| x \| |a(x)| \leq 1. $$

Finally, $\P \in \embedtozero$ since
$\H_{\ks} \subseteq \C{}{b} \subseteq \L^1(\P)$ by above,
and $\H_{K} \subseteq \L^1(\P)$ as
 for large enough $x$
$$ \| \sp \| \| x \| \geq - \metric{\sp(x)}{x} \geq r_1 \| x \|^{2u}-r_2 \geq A \| x \|^{2u}$$
for some $A>0$,
so $\| \sp \| \geq A \| x \|^{2u-1} $ for $x$ large enough,
so $\| x \|  a(x) \leq 1/\| \sp(x) \| \leq \frac{1}{A \| x \|^{2u-1}}$ for $x$ large enough,
which implies that $\H_K \subseteq \C{}{b}$.

\subsection{Proof of \ncref{MMD control of subset RKHS}}
\label{app:MMD control of subset RKHS}

By \citet[Prop.~2]{schwartz1964sous}, $\H_\k \subseteq \H_{\tilde k}$ implies that there exists $c \geq 0$ such that, for all $h \in \H_\k$,
    $\| h \|_{\H_{\tilde k}} \leq c \| h \|_{\H_\k}$,
    so $c^{-1}\B_{\k}  \subseteq \B_{\tilde k}$.
    Hence
    $$ \mmd_\k(\Q,\P) \defn \sup_{h \in \B_\k \, :  \,  h_+ \in \L^1(\Q) \ } \left | \Q h - \P h \right | \leq c \, \mmd_{\tilde k}(\Q,\P). $$
    Since
$\H_\k \subseteq \H_{\tilde k}$,
the $\P$-separation and tightness results are immediate.

\section{Proof of \ncref{thm:failure of convergence control}}
\phantomsection
\label{app:failure of convergence control}

We will use the following result based on a construction from 
\citet[Section 5]{simon2023metrizing}.

\begin{theorem}[Vanishing mean-zero kernels fail to control $\P$-convergence]\label{thm:c0-nonconvergence-V2}
Suppose that $\X$ is locally compact but not compact.
If $\H_{\k} \subseteq \C{}{0}$ and $\k$ maps  $\P \in \Pset$ to $0 \in \HK$, i.e., $\Phi_{\k}(\P)=0$, then $\k$ cannot control weak convergence to  $\P \in \Pset$.
\end{theorem}
\begin{proof}
Since $\P$ is a regular measure, we can find a compact set $C \subseteq \X$ for which $\P(C) \geq \nicefrac{1}{2}$.
By  \citet[Lemma 9 and 10]{simon2023metrizing}, we can find an open set $U$ and compact set $C'$ such that $C \subseteq U \subseteq C'$ and sequence of probability measures $(\Q_n)_n$ such that $\| \Q_n \|_{\k} \to 0$ and $\Q_n(C')=0$ for all $n$.
Then $\| \Q_n -\P \|_k=
\|  \Q_n  \|_k \to 0$, so
$ \Q_n$ converges to $\P$ in maximum mean discrepancy but not in weak convergence since
$ \P(U) \geq \P(C) >  \Q_n (U) =0 $.

\end{proof}

Now note that $\Phi_k(\P)=0$ implies that every function in $\H_k$ has vanishing $\P$-integral.
Given any RKHS $\H_\k \subseteq \L^1(\P)$, we can construct a new RKHS whose functions have zero expectation under $\P$ and has the same $\mmd$ between embeddable measures, as we now show by simply applying the projection operator $\projP(h)=h-\P h$.
Since
$$|h(x)-\P h| =| \ipdK{h}{\k_x} -\ipdK{\Phi_\k(\P)}{h}| \leq (\|\k_x\|_\k + \| \Phi_\k(\P)\|_\k) \|h \|_\k, $$
\citep[Prop.~2.4]{carmeli06vector} implies
$\projP(\H_\k)$ is a RKHS
with kernel (using the fact $\xi^*_\P( x)(h)= \projP(h)( x)$ where $\xi^*_\P( x) \defn \delta_x -\P$)
$$\k^\P( x, y) =  \ipdK{\emb_{\k}(\delta_{ x} - \P)}{\emb_{\k}(\delta_{ y} - \P)}.$$
Thus, the elements of $\H_{\k^\P}$ have the form $ h -\P h$ for some $h \in \H_\k$, and hence  $\P(\H_{\k^\P})=\{0\}$.

Importantly,
 $\k^\P$ and $\k$ generate the same MMD.
 First let us show this for embeddable measures: since for any finite measure $\mu$  with $\int \mu =0$ that embeds into $\H_\k$, we have using \cref{Continuous linear functional shifted by Feature Operator} and
 $\mu \circ \Pi_\P|_{\H_\k} = \mu|_{\H_\k}$  (from $\int \mu =0$)
that
$$   \| \mu \|_{\k^\P}= \| \mu \circ \Pi_\P \|_\k = \| \mu  \|_\k.$$
Hence for any two embeddable probability measures $\Q,\P$ we have
$\mmd_{\k^\P}(\Q,\P) = \mmd_{\k}(\Q,\P)$.
In general, for any $\Q \in \Pset$,
note that $h_+ \in \L^1(\Q)$ iff $(\projP(h))_+ \in \L^1(\Q)$.
Moreover $\B_{\k^\P} = \projP(\B_\k)$ by \cref{thm:feature operators preserve unit ball}.
Thus, writing $S_\k(\Q) \defn \{  h \in \B_\k :  h_+ \in \L^1(\Q) \}$, we have
$S_{\k^\P}(\Q) = \projP(S_\k(\Q))$
\begin{align}
     \mmd_{\k^\P}(\Q,\P) &=  \sup_{f \in S_{\k^\P}(\Q) \ } \left | \Q(f) - \P(f) \right |
=  \sup_{f   \in  \projP S_\k(\Q) \ } \left | \Q(f) - \P(f) \right |
\\
&=
\sup_{h   \in  S_\k(\Q) \ } \left | \Q(\projP h) - \P(\projP h) \right |
= \sup_{h   \in  S_\k(\Q) \ } \left | \Q h - \P h \right | =  \mmd_{\k}(\Q,\P).
\end{align}
Combining with \cref{thm:c0-nonconvergence-V2} we obtain the other advertised result.
\bibliography{steinKernels}

\begin{thebibliography}{68}
\providecommand{\natexlab}[1]{#1}
\providecommand{\url}[1]{\texttt{#1}}
\expandafter\ifx\csname urlstyle\endcsname\relax
  \providecommand{\doi}[1]{doi: #1}\else
  \providecommand{\doi}{doi: \begingroup \urlstyle{rm}\Url}\fi

\bibitem[Ambrosio et~al.(2005)Ambrosio, Gigli, and Savar\'{e}]{ambrosio05gradient}
Luigi Ambrosio, Nicola Gigli, and Giuseppe Savar\'{e}.
\newblock \emph{Gradient Flows - {I}n Metric Spaces and in the Space of Probability Measures}.
\newblock Birkh\"{a}user Verlag, Springer, 2005.

\bibitem[Anastasiou et~al.(2023)Anastasiou, Barp, Briol, Ebner, Gaunt, Ghaderinezhad, Gorham, Gretton, Ley, Liu, et~al.]{anastasiou2023stein}
Andreas Anastasiou, Alessandro Barp, Fran{\c{c}}ois-Xavier Briol, Bruno Ebner, Robert~E Gaunt, Fatemeh Ghaderinezhad, Jackson Gorham, Arthur Gretton, Christophe Ley, Qiang Liu, et~al.
\newblock Stein’s method meets computational statistics: A review of some recent developments.
\newblock \emph{Statistical Science}, 38\penalty0 (1):\penalty0 120--139, 2023.

\bibitem[Aronszajn(1950)]{aronszajn50trk}
Nachman Aronszajn.
\newblock Theory of reproducing kernels.
\newblock \emph{Transactions of the American Mathematical Society}, 1950.

\bibitem[Barp et~al.(2019)Barp, Briol, Duncan, Girolami, and Mackey]{barp2019minimum}
Alessandro Barp, Francois-Xavier Briol, Andrew Duncan, Mark Girolami, and Lester Mackey.
\newblock Minimum {S}tein discrepancy estimators.
\newblock In \emph{Advances in Neural Information Processing Systems}, pages 12964--12976, 2019.

\bibitem[Barp et~al.(2021)Barp, Takao, Betancourt, Arnaudon, and Girolami]{barp2021unifying}
Alessandro Barp, So~Takao, Michael Betancourt, Alexis Arnaudon, and Mark Girolami.
\newblock A unifying and canonical description of measure-preserving diffusions.
\newblock \emph{arXiv:2105.02845}, 2021.

\bibitem[Barp et~al.(2022{\natexlab{a}})Barp, Da~Costa, Fran{\c{c}}a, Friston, Girolami, Jordan, and Pavliotis]{barp2022geometric}
Alessandro Barp, Lancelot Da~Costa, Guilherme Fran{\c{c}}a, Karl Friston, Mark Girolami, Michael~I Jordan, and Grigorios~A Pavliotis.
\newblock Geometric methods for sampling, optimization, inference, and adaptive agents.
\newblock In \emph{Handbook of Statistics}, volume~46, pages 21--78. Elsevier, 2022{\natexlab{a}}.

\bibitem[Barp et~al.(2022{\natexlab{b}})Barp, Oates, Porcu, and Girolami]{barp2018riemannian}
Alessandro Barp, Chris~J Oates, Emilio Porcu, and Mark Girolami.
\newblock A {R}iemann--{S}tein kernel method.
\newblock \emph{Bernoulli}, 2022{\natexlab{b}}.

\bibitem[Berg et~al.(1984)Berg, Christensen, and Ressel]{berg84harmonic}
Christian Berg, Jens P.~R. Christensen, and Paul Ressel.
\newblock \emph{Harmonic Analysis on Semigroups Theory of Positive Definite and Related Functions}.
\newblock Springer, 1984.

\bibitem[Berlinet and Thomas-Agnan(2004)]{berlinet04reproducing}
Alain Berlinet and Christine Thomas-Agnan.
\newblock \emph{Reproducing Kernel {H}ilbert Spaces in Probability and Statistics}.
\newblock Springer, 2004.

\bibitem[Bochner(1932)]{bochner1932vorlesungen}
Salomon Bochner.
\newblock \emph{Vorlesungen {\"u}ber Fouriersche integrale}, volume~12.
\newblock Akademische Verlagsgesellschaft mbh, 1932.

\bibitem[Briol et~al.(2019)Briol, Barp, Duncan, and Girolami]{briol2019statistical}
Francois-Xavier Briol, Alessandro Barp, Andrew~B. Duncan, and Mark Girolami.
\newblock Statistical inference for generative models with maximum mean discrepancy.
\newblock \emph{arXiv:1906.05944}, 2019.

\bibitem[Buck(1958)]{buck1958bounded}
Creighton~R. Buck.
\newblock Bounded continuous functions on a locally compact space.
\newblock \emph{Michigan Mathematical Journal}, 1958.

\bibitem[Buescu et~al.(2004)Buescu, Paixao, Garcia, and Lourtie]{buescu2004positive}
Jorge Buescu, AC~Paixao, F~Garcia, and I~Lourtie.
\newblock Positive-definiteness, integral equations and fourier transforms.
\newblock \emph{The Journal of Integral Equations and Applications}, pages 33--52, 2004.

\bibitem[Carmeli et~al.(2006)Carmeli, De~Vito, and Toigo]{carmeli06vector}
Claudio Carmeli, Ernesto De~Vito, and Alessandro Toigo.
\newblock Vector valued reproducing kernel {H}ilbert spaces of integrable functions and {M}ercer theorem.
\newblock \emph{Analysis and Applications}, 2006.

\bibitem[Carmeli et~al.(2010)Carmeli, De~Vito, Toigo, and Umanit{\'a}]{carmeli2010vector}
Claudio Carmeli, Ernesto De~Vito, Alessandro Toigo, and Veronica Umanit{\'a}.
\newblock Vector valued reproducing kernel {H}ilbert spaces and universality.
\newblock \emph{Analysis and Applications}, 2010.

\bibitem[Chen et~al.(2018)Chen, Mackey, Gorham, Briol, and Oates]{chen18stein}
Wilson~Y. Chen, Lester Mackey, Jackson Gorham, Fran\c{c}ois-Xavier Briol, and Chris~J. Oates.
\newblock {S}tein points.
\newblock In \emph{ICML}, 2018.

\bibitem[Chen et~al.(2019)Chen, Barp, Briol, Gorham, Girolami, Mackey, Oates, et~al.]{chen2019stein}
Wilson~Ye Chen, Alessandro Barp, Fran{\c{c}}ois-Xavier Briol, Jackson Gorham, Mark Girolami, Lester Mackey, Chris Oates, et~al.
\newblock {S}tein point {M}arkov chain {M}onte {C}arlo.
\newblock \emph{arXiv:1905.03673}, 2019.

\bibitem[Ch{\'e}rief-Abdellatif and Alquier(2020)]{cherief2020mmd}
Badr-Eddine Ch{\'e}rief-Abdellatif and Pierre Alquier.
\newblock {MMD}-{B}ayes: {R}obust {B}ayesian estimation via maximum mean discrepancy.
\newblock In \emph{Symposium on Advances in Approximate {B}ayesian Inference}, 2020.

\bibitem[Chwialkowski et~al.(2016)Chwialkowski, Strathmann, and Gretton]{chwialkowski16kernel}
Kacper Chwialkowski, Heiko Strathmann, and Arthur Gretton.
\newblock A kernel test of goodness of fit.
\newblock In \emph{NeurIPS}, 2016.

\bibitem[Conway(1965)]{conway1965strict}
John~Bligh Conway.
\newblock \emph{The Strict Topology and Compactness in the Space of Measures}.
\newblock Louisiana State University and Agricultural \& Mechanical College, 1965.

\bibitem[Cooper(1960)]{cooper1960positive}
JLB Cooper.
\newblock Positive definite functions of a real variable.
\newblock \emph{Proceedings of the London Mathematical Society}, 3\penalty0 (1):\penalty0 53--66, 1960.

\bibitem[Dellaporta et~al.(2022)Dellaporta, Knoblauch, Damoulas, and Briol]{dellaporta2022robust}
Charita Dellaporta, Jeremias Knoblauch, Theodoros Damoulas, and Fran{\c{c}}ois-Xavier Briol.
\newblock Robust {B}ayesian inference for simulator-based models via the {MMD} posterior bootstrap.
\newblock In \emph{AISTATS}, 2022.

\bibitem[Dziugaite et~al.(2015)Dziugaite, Roy, and Ghahramani]{dziugaite15training}
Gintare~K. Dziugaite, Daniel~M. Roy, and Zoubin Ghahramani.
\newblock Training generative neural networks via maximum mean discrepancy optimization.
\newblock In \emph{UAI}, 2015.

\bibitem[Ethier and Kurtz(2009)]{ethier2009markov}
Stewart~N. Ethier and Thomas~G. Kurtz.
\newblock \emph{{M}arkov processes: {C}haracterization and convergence}, volume 282.
\newblock John Wiley \& Sons, 2009.

\bibitem[Fisher et~al.(2022)Fisher, Oates, et~al.]{fisher2022gradient}
Matthew~A Fisher, Chris Oates, et~al.
\newblock Gradient-free kernel {S}tein discrepancy.
\newblock \emph{arXiv:2207.02636}, 2022.

\bibitem[Futami et~al.(2019)Futami, Cui, Sato, and Sugiyama]{futami2019bayesian}
Futoshi Futami, Zhenghang Cui, Issei Sato, and Masashi Sugiyama.
\newblock {B}ayesian posterior approximation via greedy particle optimization.
\newblock In \emph{AAAI}, 2019.

\bibitem[Gorham and Mackey(2015)]{gorham15measuring}
Jackson Gorham and Lester Mackey.
\newblock Measuring sample quality with {S}tein's method.
\newblock In \emph{NeurIPS}, 2015.

\bibitem[Gorham and Mackey(2017)]{gorham17measuring}
Jackson Gorham and Lester Mackey.
\newblock Measuring sample quality with kernels.
\newblock In \emph{ICML}, 2017.

\bibitem[Gorham et~al.(2019)Gorham, Duncan, Vollmer, and Mackey]{gorham2019measuring}
Jackson Gorham, Andrew~B. Duncan, Sebastian~J. Vollmer, and Lester Mackey.
\newblock Measuring sample quality with diffusions.
\newblock \emph{The Annals of Applied Probability}, 2019.

\bibitem[Gorham et~al.(2020)Gorham, Raj, and Mackey]{gorham2020stochastic}
Jackson Gorham, Anant Raj, and Lester Mackey.
\newblock Stochastic {S}tein discrepancies.
\newblock \emph{Advances in Neural Information Processing Systems}, 2020.

\bibitem[Gretton et~al.(2012)Gretton, Borgwardt, Rasch, Sch\"olkopf, and Smola]{gretton12}
Arthur Gretton, Karsten~M. Borgwardt, Malte~J. Rasch, Bernhard Sch\"olkopf, and Alexander Smola.
\newblock A kernel two-sample test.
\newblock \emph{JMLR}, 2012.

\bibitem[Han and Liu(2018)]{han2018stein}
Jun Han and Qiang Liu.
\newblock {S}tein variational gradient descent without gradient.
\newblock In \emph{ICML}, 2018.

\bibitem[Hodgkinson et~al.(2020)Hodgkinson, Salomone, and Roosta]{Hodgkinson2020}
Liam Hodgkinson, Robert Salomone, and Fred Roosta.
\newblock The reproducing {S}tein kernel approach for post-hoc corrected sampling.
\newblock \emph{arXiv:2001.09266}, 2020.

\bibitem[Huggins and Mackey(2018)]{HugginsMa2018}
Jonathan Huggins and Lester Mackey.
\newblock Random feature {S}tein discrepancies.
\newblock In \emph{NeurIPS}, 2018.

\bibitem[Johnson(2018)]{johnson2015saddle}
Steven~G. Johnson.
\newblock Saddle-point integration of {C}$\infty$ ``bump'' functions.
\newblock \emph{arXiv:1508.04376}, 2018.

\bibitem[Kanagawa et~al.(2022)Kanagawa, Barp, Gretton, and Mackey]{kanagawa2022controlling}
Heishiro Kanagawa, Alessandro Barp, Arthur Gretton, and Lester Mackey.
\newblock Controlling moments with kernel stein discrepancies.
\newblock \emph{arXiv preprint arXiv:2211.05408}, 2022.

\bibitem[LeCam(1957)]{lecam1957convergence}
Lucien LeCam.
\newblock Convergence in distribution of stochastic processes.
\newblock \emph{University of California Publications in Statistics}, 1957.

\bibitem[Liu and Zhu(2018)]{liu2018riemannian}
Chang Liu and Jun Zhu.
\newblock {R}iemannian {S}tein variational gradient descent for {B}ayesian inference.
\newblock In \emph{AAAI}, 2018.

\bibitem[Liu(2017)]{liu2017stein}
Qiang Liu.
\newblock {S}tein variational gradient descent as gradient flow.
\newblock In \emph{NeurIPS}, 2017.

\bibitem[Liu and Lee(2017)]{liu2017black}
Qiang Liu and Jason Lee.
\newblock Black-box importance sampling.
\newblock In \emph{AISTATS}, 2017.

\bibitem[Liu and Wang(2016)]{liu2016stein}
Qiang Liu and Dilin Wang.
\newblock {S}tein variational gradient descent: {A} general purpose {B}ayesian inference algorithm.
\newblock In \emph{NeurIPS}, 2016.

\bibitem[Liu et~al.(2016)Liu, Lee, and Jordan]{liu16kernelized}
Qiang Liu, Jason Lee, and Michael Jordan.
\newblock A kernelized {S}tein discrepancy for goodness-of-fit tests.
\newblock In \emph{ICML}, 2016.

\bibitem[Matsubara et~al.(2021)Matsubara, Knoblauch, Briol, Oates, et~al.]{matsubara2021robust}
Takuo Matsubara, Jeremias Knoblauch, Fran{\c{c}}ois-Xavier Briol, Chris Oates, et~al.
\newblock Robust generalised {B}ayesian inference for intractable likelihoods.
\newblock \emph{arXiv:2104.07359}, 2021.

\bibitem[Matsubara et~al.(2022)Matsubara, Knoblauch, Briol, and Oates]{matsubara2022generalised}
Takuo Matsubara, Jeremias Knoblauch, Fran{\c{c}}ois-Xavier Briol, and Chris Oates.
\newblock Generalised {B}ayesian inference for discrete intractable likelihood.
\newblock \emph{arXiv:2206.08420}, 2022.

\bibitem[Micheli and Glaunes(2013)]{micheli2013matrix}
Mario Micheli and Joan~Alexis Glaunes.
\newblock Matrix-valued kernels for shape deformation analysis.
\newblock \emph{arXiv:1308.5739}, 2013.

\bibitem[M{\"u}ller(1997)]{muller1997integral}
Alfred M{\"u}ller.
\newblock Integral probability metrics and their generating classes of functions.
\newblock \emph{Advances in Applied Probability}, 1997.

\bibitem[Musia\l{}(2002)]{musial02pettis}
Kazimierz Musia\l{}.
\newblock \emph{Handbook of {Measure} {Theory} - Chapter 12 - {P}ettis Integral}.
\newblock Elsevier, 2002.

\bibitem[Oates et~al.(2014)Oates, Girolami, and Chopin]{oates2014control}
Chris~J. Oates, Mark Girolami, and Nicolas Chopin.
\newblock Control functionals for {M}onte {C}arlo integration.
\newblock \emph{arXiv:1410.2392}, 2014.

\bibitem[Oates et~al.(2019)Oates, Cockayne, Briol, and Girolami]{oates2019convergence}
Chris~J. Oates, Jon Cockayne, Fran{\c{c}}ois-Xavier Briol, and Mark Girolami.
\newblock Convergence rates for a class of estimators based on {S}tein's method.
\newblock \emph{Bernoulli}, 2019.

\bibitem[Paulsen and Raghupathi(2016)]{paulsen2016introduction}
Vern~I. Paulsen and Mrinal Raghupathi.
\newblock \emph{An Introduction to the Theory of Reproducing Kernel {H}ilbert Spaces}.
\newblock Cambridge University Press, 2016.

\bibitem[Phillips(2018)]{phillips2018positive}
Tomos Phillips.
\newblock \emph{On positive and conditionally negative definite functions with a singularity at zero, and their applications in potential theory}.
\newblock PhD thesis, Cardiff University, 2018.

\bibitem[Pigola and Setti(2014)]{pigola2014global}
Stefano Pigola and Alberto~G. Setti.
\newblock Global divergence theorems in nonlinear {PDEs} and geometry.
\newblock \emph{Ensaios Matem{\'a}ticos}, 2014.

\bibitem[Riabiz et~al.(2022)Riabiz, Chen, Cockayne, Swietach, Niederer, Mackey, and Oates]{riabiz2022optimal}
Marina Riabiz, Wilson~Ye Chen, Jon Cockayne, Pawel Swietach, Steven~A. Niederer, Lester Mackey, and Chris.~J. Oates.
\newblock Optimal thinning of {MCMC} output.
\newblock \emph{Journal of the Royal Statistical Society: Series B (Statistical Methodology)}, 2022.

\bibitem[Schwartz(1964)]{schwartz1964sous}
Laurent Schwartz.
\newblock Sous-espaces hilbertiens d'espaces vectoriels topologiques et noyaux associ{\'e}s (noyaux reproduisants).
\newblock \emph{Journal d'analyse math{\'e}matique}, 1964.

\bibitem[Schwartz(1978)]{schwartzTD}
Laurent Schwartz.
\newblock \emph{Th\'eorie des Distributions}.
\newblock Hermann, 1978.

\bibitem[Simon-Gabriel and Sch\"{o}lkopf(2018)]{simon18kde}
Carl-Johann Simon-Gabriel and Bernhard Sch\"{o}lkopf.
\newblock Kernel distribution embeddings: {U}niversal kernels, characteristic kernels and kernel metrics on distributions.
\newblock \emph{JMLR}, 2018.

\bibitem[Simon-Gabriel et~al.(2023)Simon-Gabriel, Barp, Sch{\"o}lkopf, and Mackey]{simon2023metrizing}
Carl-Johann Simon-Gabriel, Alessandro Barp, Bernhard Sch{\"o}lkopf, and Lester Mackey.
\newblock Metrizing weak convergence with maximum mean discrepancies.
\newblock \emph{Journal of Machine Learning Research}, 24\penalty0 (184):\penalty0 1--20, 2023.

\bibitem[Sriperumbudur(2016)]{sriperumbudur16optimal}
Bharath~K. Sriperumbudur.
\newblock On the optimal estimation of probability measures in weak and strong topologies.
\newblock \emph{Bernoulli}, 2016.

\bibitem[Sriperumbudur et~al.(2010)Sriperumbudur, Gretton, Fukumizu, Sch\"olkopf, and Lanckriet]{sriperumbudur10hilbert}
Bharath~K. Sriperumbudur, Arthur Gretton, Kenji Fukumizu, Bernhard Sch\"olkopf, and Gert R.~G. Lanckriet.
\newblock {H}ilbert space embeddings and metrics on probability measures.
\newblock \emph{JMLR}, 2010.

\bibitem[Sriperumbudur et~al.(2011)Sriperumbudur, Fukumizu, and Lanckriet]{sriperumbudur11}
Bharath~K. Sriperumbudur, Kenji Fukumizu, and Gert~RG Lanckriet.
\newblock Universality, characteristic kernels and {RKHS} embedding of measures.
\newblock \emph{JMLR}, 2011.

\bibitem[Steinwart and Christmann(2008)]{steinwart08support}
Ingo Steinwart and Andreas Christmann.
\newblock \emph{Support Vector Machines}.
\newblock Information Science and Statistics. Springer, 2008.

\bibitem[Stewart(1976)]{stewart1976positive}
James Stewart.
\newblock Positive definite functions and generalizations, an historical survey.
\newblock \emph{The Rocky Mountain Journal of Mathematics}, 6\penalty0 (3):\penalty0 409--434, 1976.

\bibitem[Sun et~al.(2023)Sun, Barp, and Briol]{sun2023vector}
Zhuo Sun, Alessandro Barp, and Fran{\c{c}}ois-Xavier Briol.
\newblock Vector-valued control variates.
\newblock In \emph{International Conference on Machine Learning}, pages 32819--32846. PMLR, 2023.

\bibitem[Treves(1967)]{treves67}
Fran\c{c}ois Treves.
\newblock \emph{Topological Vector Spaces, Distributions and Kernels}.
\newblock Academic Press, 1967.

\bibitem[Wendland(2004)]{wendland04scattered}
Holger Wendland.
\newblock \emph{Scattered Data Approximation}.
\newblock Cambridge University Press, 2004.

\bibitem[Wynne et~al.(2022)Wynne, Kasprzak, and Duncan]{wynne2022spectral}
George Wynne, Miko{\l}aj Kasprzak, and Andrew~B Duncan.
\newblock A spectral representation of kernel {S}tein discrepancy with application to goodness-of-fit tests for measures on infinite dimensional {H}ilbert spaces.
\newblock \emph{arXiv:2206.04552}, 2022.

\bibitem[Zhang and Zhao(2013)]{zhang2013inclusion}
Haizhang Zhang and Liang Zhao.
\newblock On the inclusion relation of reproducing kernel {H}ilbert spaces.
\newblock \emph{Analysis and Applications}, 2013.

\bibitem[Zhou(2008)]{zhou2008derivative}
Ding-Xuan Zhou.
\newblock Derivative reproducing properties for kernel methods in learning theory.
\newblock \emph{Journal of computational and Applied Mathematics}, 220\penalty0 (1-2):\penalty0 456--463, 2008.

\end{thebibliography}

\end{document}